%% file: main.tex
\documentclass{article}
\usepackage[final,main]{neurips_2026}

\usepackage[utf8]{inputenc} % allow utf-8 input
\usepackage[T1]{fontenc}    % use 8-bit T1 fonts
\usepackage{hyperref}       % hyperlinks
\usepackage{url}            % simple URL typesetting
\usepackage{booktabs}       % professional-quality tables
\usepackage{amsfonts}       % blackboard math symbols
\usepackage{nicefrac}       % compact symbols for 1/2, etc.
\usepackage{microtype}      % microtypography
\usepackage{xcolor}         % colors
\usepackage{wrapfig}
\usepackage{booktabs}   
\usepackage[table]{xcolor}  
\usepackage{graphicx}
\usepackage{subcaption}
\usepackage{enumitem}
\usepackage{amsmath}
\usepackage{amssymb}
\usepackage{mathtools}
\usepackage{cleveref}

\usepackage{algorithm}
\usepackage{algpseudocode}
\usepackage{float} % provides [H] float placement option
\usepackage{placeins} % provides \FloatBarrier to prevent floats/wraptables drifting too far
\definecolor{remdmbrown}{RGB}{180,96,24}
\newcommand{\remdm}[1]{{\color{remdmbrown}#1}}

\newtheorem{theorem}{Theorem}[section]
\newtheorem{lemma}{Lemma}[section]
\newtheorem{proof}{Proof}[section]

\newtheorem{proposition}{Proposition}[section]

\input{math_commands.tex}

\newcommand{\vheader}{\vspace*{-0.11cm}}
\usepackage[textsize=tiny]{todonotes}

\title{Unlocking Continuous Diffusion Language Models via Discrete Stochastic Localization}

\author{%
\normalfont
\textbf{Yunshu Wu}\textsuperscript{1} \quad
\textbf{Jiayi Cheng}\textsuperscript{2} \quad
\textbf{Longxuan Yu}\textsuperscript{1} \quad
\textbf{Partha Thakuria}\textsuperscript{1} \\
\textbf{Rob Brekelmans} \quad
\textbf{Evangelos E. Papalexakis}\textsuperscript{1} \quad
\textbf{Greg Ver Steeg}\textsuperscript{1} \\
\textsuperscript{1}University of California Riverside \quad
\textsuperscript{2}New York University \\
\texttt{\{ywu380, ylong030, pthakuria, epapalex, greg.versteeg\}@ucr.edu} \\
\texttt{jiayi.cheng@nyu.edu \quad rob.brekelmans@gmail.com}
}

\begin{document}

\maketitle

\begin{abstract} 
Continuous diffusion is a natural framework for non-autoregressive generation but has generally lagged behind masked discrete diffusion models (MDMs) on discrete sequence generation. 
We argue that continuous diffusion can effectively model discrete sequences through the joint design of noise parameterization and token representation geometry. 
We introduce \emph{Discrete Stochastic Localization} (DSL), a continuous-state framework that combines a stochastic localization channel with unit-sphere token embeddings, resulting in a Bayes-optimal denoiser that depends only on the noisy state, without additional timestep conditioning. 
A single trained denoising network then supports multiple sampling methods, from continuous denoising to masked diffusion, random-order autoregressive generation and hybrid continuous-then-discrete generation. 
Parallel generation in this continuous-state framework, therefore, need not imply synchronized denoising: some tokens can establish context earlier while others remain open to refinement.
Fine-tuning a pretrained MDLM checkpoint with DSL substantially improves generation quality, as measured by MAUVE, on OpenWebText across all step budgets from $T{=}128$ to $T{=}1024$.
The same checkpoint supports random-order autoregressive sampling, and a hybrid continuous-then-discrete sampler using as few as $T=48$ total steps, without distillation or further training. 
\end{abstract}

\vheader
\section{Introduction}
\label{sec:intro}
\vheader

Continuous diffusion has become a standard framework for high-dimensional generation, with strong results in image, video, and audio synthesis \citep{ddpm,ddim,ho2022video,huang2025self,kong2020diffwave}. 
For language, however, continuous diffusion over token embeddings has generally lagged behind masked discrete diffusion models (MDMs) \citep{austin2021structured,sahoo2024simple,shi2024simplified}. 
This gap is puzzling. 
Continuous diffusion exposes the model to intermediate states at finite signal-to-noise ratios (SNRs), between the fully uninformative and nearly clean limits, which should in principle provide richer supervision than endpoint-like masked corruptions. 
Why, then, has this extra continuity not been translated into stronger language generation?

We argue that continuous diffusion can effectively model discrete sequences through the joint design of noise parameterization and token representation geometry. 
We introduce \emph{Discrete Stochastic Localization} (DSL), a continuous-state framework that combines a stochastic localization channel with unit-sphere token embeddings. 
Together, these choices yield a Bayes-optimal denoiser that provably depends only on the noisy sequence, without additional timestep conditioning. 
This allows different token positions to follow different denoising schedules while sharing the same denoiser. 
% Together, these choices yield a provably SNR-invariant Bayes-optimal denoiser: 
% the clean-token posterior depends only on the noisy sequence, without additional timestep conditioning. 
% Each token position can therefore have its own SNR, while sharing the same denoiser. 
% This DSL construction provides a continuous representation of token information, connecting a fully uninformative state to increasingly informative observations of discrete tokens.
% At zero SNR, every token is represented by the same zero vector, providing a deterministic \texttt{[MASK]} state; as SNR increases, the clean-token posterior concentrates on the underlying token. 
% Unlike standard variance-preserving (VP) or variance-exploding (VE) diffusion in their usual coordinates, DSL thus gives the zero-information limit an explicit mask-like representation within a continuous state space. 

Parallel token generation in DSL need not imply synchronized denoising. 
Some tokens can establish context earlier, while others remain open to refinement.
DSL also continuously connects masked and clean token states. 
At zero SNR, the state is a zero vector, serving as a deterministic \texttt{[MASK]}.
As SNR increases, token identity is progressively resolved through intermediate states, and the clean-token posterior concentrates on a specific token.
This connection allows DSL to support both continuous denoising and discrete sampling within the same denoising framework. 
We train a single denoising network on intermediate finite-SNR states as well as states near the mask and clean limits.
The resulting network supports continuous denoising, mask/uniform diffusion sampling, random-order autoregressive generation, and hybrid continuous-then-discrete sampling.

We instantiate DSL by fine-tuning a pretrained MDLM checkpoint and evaluate the resulting model on OpenWebText using multiple sampling methods. 
With a ReMDM-family sampler, DSL substantially improves generation quality, measured by MAUVE, across sampling budgets from $T=128$ to $T=1024$. 
The same checkpoint also supports random-order autoregressive generation and a hybrid continuous-then-discrete sampler that produces coherent generations in as few as $T=48$ steps, without further training or distillation. 
Together, these results demonstrate that DSL is an effective continuous-state framework for language generation, offering improved generation quality and efficient few-step sampling with a single denoising network.
% Together, these results support the central claim: endpoint revealing, masked refinement, and continuous/hybrid denoising can be served by one posterior estimator rather than by separate sampler-specific models.

\begin{figure}[tb]
  \centering
  \begin{subfigure}[b]{0.43\linewidth}
    \includegraphics[width=\linewidth]{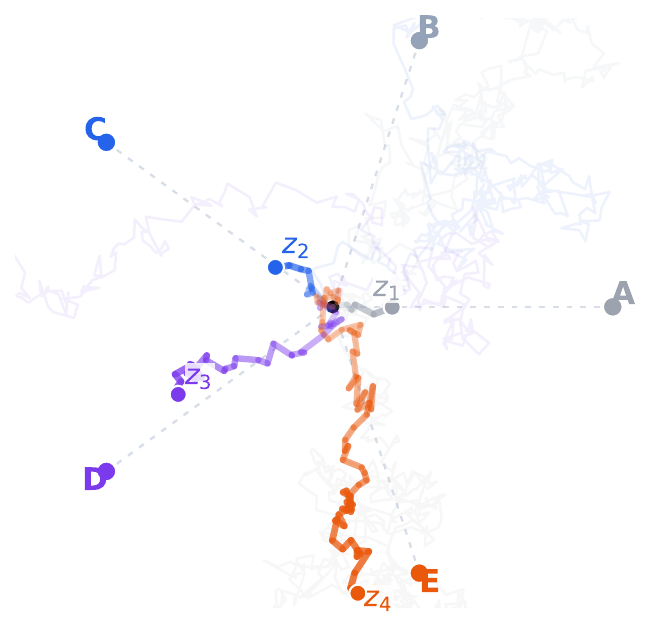}
    \caption{DSL dynamics: each $\vz_i(t)$ localizes toward a symbol anchor over time.}
    \label{fig:example}
  \end{subfigure}
  \hfill
  \begin{subfigure}[b]{0.43\linewidth}
    \includegraphics[width=\linewidth]{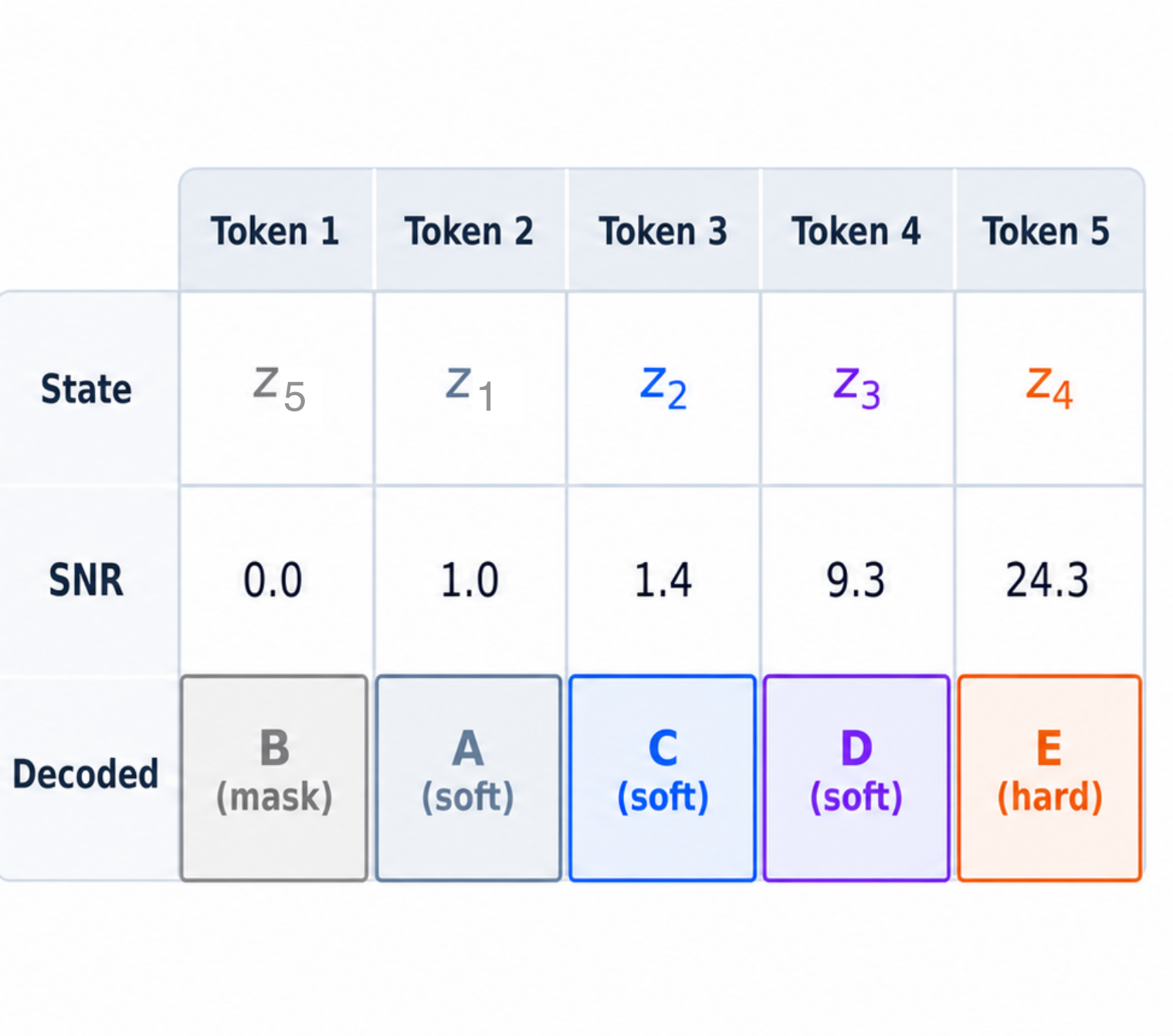}
    \caption{State summary: tokens span SNR $1.0$--$24.3$, from masked to hard decoding.}
    \label{fig:state_table}
  \end{subfigure}
  \caption{Discrete Stochastic Localization (DSL) dynamically ``localizes'' to sample discrete tokens. \emph{(Left)} DSL dynamics on a cyclic toy sequence (ABCDE, BCDEA, \ldots); details in Appendix~\ref{sec:mech_appendix}. \emph{(Right)} State summary at one snapshot: each token occupies a different SNR regime, transitioning from masked to soft to hard decoding as SNR increases.}
  \label{fig:dsl_main}
\end{figure}

\textbf{Contributions.}
\begin{itemize}[itemsep=1pt, topsep=2pt, leftmargin=*]
    \item \textbf{A continuous-state framework for discrete generation.} 
    We introduce DSL, a stochastic localization channel combined with unit-sphere token embeddings, and prove that its Bayes-optimal denoiser is time-invariant: it depends only on the noisy state and requires no timestep conditioning.
    This allows different token positions to follow different denoising schedules while sharing the same network.
    \item \textbf{A unified training recipe.}
    We train a single denoising network on both intermediate finite-SNR states and states near the masked and clean limits using a shared cross-entropy objective.
    A mixture-of-tokens converter maps these states to inputs for a standard Transformer/DiT backbone.
    \item \textbf{Improved generation quality and efficient sampling.}
    Fine-tuning a pretrained MDLM checkpoint with DSL improves OpenWebText MAUVE across sampling budgets from $T=128$ to $T=1024$.
    The same checkpoint supports random-order autoregressive generation and hybrid sampling in as few as $T=48$ steps, without further training or distillation.
    DSL also achieves competitive likelihood on Text8.
\end{itemize}

\vheader
\section{Discrete Stochastic Localization}
\label{sec:method}
\vheader

This section develops DSL in three steps. We first set up the simplest version of the framework---a single Gaussian observation channel applied to the whole sentence at one global signal-to-noise ratio (\S\ref{subsec:sentence_level}). We then introduce DSL's key design choice, placing clean token embeddings on the unit sphere, and show that under this geometry the Bayes-optimal denoiser depends only on the noisy token and not on the nominal SNR (\S\ref{subsec:time_invariant}). Because a single denoiser then suffices across all SNR regimes, we can let each token position evolve under its own SNR; this generalizes the channel to a continuous family of per-token SNR paths (\S\ref{subsec:arbitrary_paths}). 
% Finally, two exact likelihood views identify which states this path family asks training to cover (\S\ref{subsec:nll_views}).
Finally, two exact NLL views justify the state support used by the training
objective (\S\ref{subsec:nll_views}).

\vheader
\subsection{A sentence-level continuous starting point}
\label{subsec:sentence_level}
\vheader

We start from the simplest version of DSL: a Gaussian observation channel applied to the whole sentence at a single global signal-to-noise ratio. This setup is the natural continuous analogue of standard diffusion and lets us derive the MMSE drift form that the rest of the section will build on.

Let $\vs=(s_1,\dots,s_L)\in \gV^L$ be a sentence of length $L$, and let each token be embedded as $\vx_i=\enc(s_i)\in\R^d$. We write the full sentence embedding as $\vx=[\vx_1;\dots;\vx_L]\in\R^{Ld}$. The sentence-level localization channel is
\[
\vz_t = t\,\vx + \sqrt{t}\,\eps,
\qquad
\vx\sim P(\vx),\;\eps\sim\mathcal N(\vzero,\mI).
\]
In this parameterization, time $t$ equals the signal-to-noise ratio (SNR): for $t>0$, the squared signal coefficient divided by the noise variance is $t^2/t=t$.
Dividing by $t$ gives $\vz_t/t = \vx + \eps/\sqrt{t}$, so $\vz_t$ is equivalent in information to observing $\vx$ through additive Gaussian noise of variance $1/t$. 
Thus, increasing $t$ provides progressively more information about $\vx$, starting from the deterministic zero-information state $\vz_0=\vzero$.
This channel admits a simple SDE realization,
\begin{equation}
d\vz_t = \vx\,dt + d\rvw,
\qquad
\vx\sim P(\vx),
\label{eq:conditional}
\end{equation}
which conditions on the unknown clean sentence and so cannot be used directly for generation. 
%In \cref{app:conditional_unconditional}, we derive an 
We seek an 
equivalent unconditional SDE---equivalent in the sense that both dynamics induce the same marginal on $\vz_t$
---whose drift depends only on $\vz_t$:
\begin{equation}
d\vz_t = \xhat(\vz_t,t)\,dt + d\rvw,
\qquad
\xhat(\vz_t,t):=\E_{P_t(\vx\mid \vz_t)}[\vx].
\label{eq:unconditional}
\end{equation}
%See \cref{app:conditional_unconditional} for proof.
The drift that achieves this equivalence is the MMSE denoiser, the posterior mean of $\vx$ given $\vz_t$ (proof in Appendix~\ref{app:conditional_unconditional}). At this point DSL still resembles a standard continuous-denoising construction: one global SNR, one denoiser that nominally depends on $t$.

\vheader
\subsection{Unit-norm token geometry gives a time-invariant denoiser}
\label{subsec:time_invariant}
\vheader

The key design choice in DSL is to constrain every clean token embedding to lie on the unit sphere, $\norm{\vx_i}_2 = 1$ for all $i$. This is a structural constraint built into the embedding geometry, not a regularizer on the loss. As we now show, the benefit of this construction is to remove the nominal SNR $t$ from the Bayes-optimal denoiser.

Under the localization channel \eqref{eq:conditional}, 
$P_t(\vx\mid\vz_t) \;\propto\; P(\vx)\exp\!\Big(\vx\cdot\vz_t \;-\; \tfrac{t}{2}\norm{\vx}_2^2\Big)$ is the Bayes posterior factorization.
The term $\tfrac{t}{2}\norm{\vx}_2^2$ is the only place the nominal SNR enters. Under unit-norm token embeddings, $\norm{\vx}_2^2 = L$ for every length-$L$ sentence, so this term is sentence-independent and cancels after normalization. The posterior, and therefore the MMSE denoiser, depends only on $\vz_t$:
\begin{equation}
\hat{\vx}(\vz,t)
=
\E_{P_t(\vx\mid\vz)}[\vx]
=
\E_{P(\vx)}\!\left[\vx e^{\vx\cdot\vz}\right]\Big/\E_{P(\vx)}\!\left[e^{\vx\cdot\vz}\right]
=
\textcolor{red}{\xhat(\vz)}.
\label{eq:t_invariant_denoiser}
\end{equation}
The SNR-invariance in Eq.~\ref{eq:t_invariant_denoiser} requires both the localization parameterization, which makes $z$ a posterior natural coordinate with $a_\tau/\sigma_\tau^2=1$, and unit-sphere geometry, which makes $\|x_i\|^2$ vocabulary-constant. Either condition alone is insufficient; their composition produces the structural cancellation.

This SNR-invariance is the structural property that drives the rest of the paper. The Bayes denoiser does not need to know which $t$ generated $\vz_t$ to denoise it optimally; it only needs the noisy token. A single network can therefore in principle serve every SNR regime at once. We give the full derivation in Appendix~\ref{app:snr_invariant_denoiser}; geometric consequences---an entropy envelope on the induced token posterior $p(\vx_i\mid\vz_t)$ and Lipschitz smoothness---are deferred to Section~\ref{sec:analysis}.

This connection between spherical geometry and autonomous posterior-mean dynamics also appears in Sekimoto's independent work on classical spins~\cite{ken24PhysRevE.109.014106,ken25PhysRevE.111.044102}; see Section~\ref{sec:related_work} for the correspondence discussed in~\cite{sekimoto2026autonomous}.
% \textbf{Remark (necessity).}\quad
% The converse also holds. With per-token SNRs, $P_t(\vx\mid\vz_t)$ is independent of $t$ for all $\vz_t$ only if each $\norm{\vx_i}$ is constant on the support of $P$ (Appendix~\ref{app:snr_invariant_denoiser}).
% Sekimoto~\cite{sekimoto2026autonomous} reaches the same conclusion from the opposite side. 
% Requiring the drift of $d\vz = \vx(\vz)\,dt + d\rvw$ to be both a martingale and autonomous forces the tilted measure onto a sphere~\cite{ken24PhysRevE.109.014106,ken25PhysRevE.111.044102}.
% Eq.~\ref{eq:unconditional}, with the SNR-invariance drift of Eq.~\ref{eq:t_invariant_denoiser}, is this dynamics with a non-uniform background measure $P(\vx)$ on the token embeddings.

\vheader
\subsection{Arbitrary per-token SNR paths}
\label{subsec:arbitrary_paths}
\vheader

The SNR-invariance from \S\ref{subsec:time_invariant} lets us drop the assumption that all tokens share a single SNR. We can let each position $i$ evolve under its own SNR:
\[
\vz_i = \gamma_i\,\vx_i + \sqrt{\gamma_i}\,\eps_i,
\qquad \eps_i\sim\mathcal N(\vzero,\mI),
\]
% and we allow each $\gamma_i(t)$ to follow an arbitrary admissible path---continuous, non-decreasing, with $\gamma_i(0)=0$ and $\gamma_i(t)\to\infty$ as $t\to\infty$.
and let each $\gamma_i$ follow a piecewise-smooth path from $0$ to $\infty$.
\eqref{eq:probability} holds on any such path. Eqs.~\ref{eq:snr_cond}--\ref{eq:snr_uncond} describe the segments where $\gamma_i$ increases; Proposition~\ref{prop:reverse} describes the segments where it decreases.

For these per-token paths, the data-conditioned and unconditional dynamics are
\begin{align}
d\vz_i &= \vx_i\,\dot{\gamma}_i\,dt + \sqrt{\dot{\gamma}_i}\,d\rvw_i, \qquad \vx\sim P(\vx),
\label{eq:snr_cond}\\
d\vz_i &= \xhat_i(\vz)\,\dot{\gamma}_i\,dt + \sqrt{\dot{\gamma}_i}\,d\rvw_i.
\label{eq:snr_uncond}
\end{align}
We show in Appendix~\ref{app:conditional_unconditional} that these two SDEs induce the same marginal on $\vz_i$ along any admissible per-token path. Combined with the SNR-invariance of $\xhat_i$, this means a \emph{single} trained denoiser can be sampled along arbitrary per-token SNR schedules without changing the target distribution.

\textbf{Lowering SNR.}\quad
\eqref{eq:snr_cond} and \eqref{eq:snr_uncond} describe increasing SNR.
The same localization process also admits an exact reverse transition.
This supplies the SNR-lowering operation underlying remasking in the DSL framework, which we use to interpret the masked samplers in Section~\ref{subsec:masked_refinement}.

\begin{proposition}[Reverse localization]\label{prop:reverse}
Let $\vz_i$ and $\vz_i'$ denote the states at SNRs $\gamma_i$ and $\gamma_i'$, respectively, coupled by \eqref{eq:snr_cond}.
For $0\le\gamma_i'<\gamma_i$,
\[
\vz_i' \mid \vz_i,\vx
\sim
\mathcal N\!\left(
    \frac{\gamma_i'}{\gamma_i}\vz_i,\;
    \gamma_i'\left(1-\frac{\gamma_i'}{\gamma_i}\right)\mI
\right).
\]
This reverse kernel is independent of $\vx$. 
Setting $\gamma_i'=0$ recovers the mask state $\vzero$, while $0<\gamma_i'<\gamma_i$ gives a partial remask.
\end{proposition}

Generating toward higher SNR uses SNR-invariant denoiser $\xhat(\vz)$; lowering SNR requires no denoiser.
Together, these transitions allow a token to accumulate evidence, discard it through remasking, and be refined again within the same localization framework.
In SNR time, the reverse process is a Brownian bridge toward $\vzero$ (Appendix~\ref{app:reverse_sl}).
% \textbf{Lowering SNR.}\quad
% \eqref{eq:snr_cond}--\eqref{eq:snr_uncond} move $\gamma_i$ up. The channel can also move it down, and this direction has a closed form.

% \begin{proposition}[Reverse localization]\label{prop:reverse}
% Write $\vz_i(\gamma)=\gamma\vx_i+\rvb_i(\gamma)$, with $\rvb_i$ a Brownian motion. \yunshu{Revise notation, align to what we use in paper.}
% For $0\le a<b$,
% \[
% \vz_i(a)\mid\vz_i(b),\vx \sim
% \mathcal N\!\big(\tfrac{a}{b}\vz_i(b),\ \tfrac{a(b-a)}{b}\mI\big).
% \]
% The kernel does not depend on $\vx$. At $a=0$ it returns the zero-SNR state $\vzero$.
% \end{proposition}

% Raising SNR needs the SNR-invariance denoiser $\xhat(\vz)$; lowering SNR needs no learned quantity. In continuous time, a decreasing segment is a Brownian bridge to $\vzero$ (Appendix~\ref{app:reverse_sl}).

The path interpretation is then immediate. Driving one token from $\gamma{=}0$ to $\infty$ while holding others fixed corresponds to revealing that token in random-order autoregressive (ROAR) style; 
lowering selected positions to $\gamma{=}0$ (Proposition~\ref{prop:reverse}) and raising them again corresponds to masked diffusion with remasking;
jointly increasing all $\gamma_i$ corresponds to standard continuous diffusion. \textbf{Random-order AR generation, masked refinement, and continuous diffusion sampling are different paths through one shared family of per-token SNR configurations}---and, by the argument above, can all be served by one trained DSL model.

\vheader
\subsection{Exact NLL views that motivate the training support}
\label{subsec:nll_views}
\vheader

The DSL training objective is not chosen as a sampler heuristic. It is derived from two exact likelihood views of the same per-token SNR formalism: a continuous path-integral NLL over finite-SNR denoising states, and an endpoint ROAR NLL over mask/reveal states. 
These identities are the likelihood-level reason that the mixed-support loss in Section~\ref{subsec:mixed_snr} places training mass on both state families.

\textbf{Continuous path-integral view.} For any admissible per-token SNR path $C$, DSL admits the path-integral identity
\begin{equation}
-\log P(\vx)
\;=\;
\tfrac{1}{2}\int_C \mathbf{E}(\vx,\gamma)\cdot d\gamma,
\qquad
E_i(\vx,\gamma):=\E_{p_\gamma(\vz\mid\vx)}\!\big[\norm{\vx_i - \xhat_i(\vz)}_2^2\big]
\label{eq:probability}
\end{equation}
(derivation in Appendix~\ref{app:arbitrary_paths}). Under the Bayes-optimal denoiser, this error field is conservative, so the integral is path-independent: every admissible path gives the same $-\log P(\vx)$. Intermediate continuous-SNR states are therefore not heuristic interpolants---they are the integration variables of an exact likelihood.

\textbf{Endpoint (ROAR) view.} A second exact view comes from configurations in which some positions are already fully revealed ($\gamma{=}\infty$) and the rest remain fully masked ($\gamma{=}0$). In this regime the sentence likelihood reduces to an autoregressive-style decomposition, and averaging over the size of the revealed subset, the choice of revealed positions, and the next index to predict yields the random-order autoregressive estimator
\begin{equation}
- \tfrac{1}{L}\log P(\vs)
\;=\;
- \E_{(k,A,i)}\!\big[\log P(s_i\mid s_A)\big]
\label{eq:roar_estimator}
\end{equation}
(derivation in Appendix~\ref{app:roar_likelihood}). This is the subset-conditioned objective used by masked-diffusion language models~\citep{sahoo2024simple,llada}; here, it arises directly from the same DSL path formalism.

\textbf{Bridge to mixed-support training.}
\eqref{eq:probability} and~\eqref{eq:roar_estimator} are the two
likelihood sources of the practical objective in Section~\ref{subsec:mixed_snr}.
The former contributes intermediate finite-SNR states; the latter contributes endpoint mask/reveal states. Section~\ref{sec:training} then implements this NLL-derived support with a single mixed cross-entropy objective.

\vheader
\section{DSL Training and Architecture}
\label{sec:training}
\vheader

Section~\ref{sec:method} identifies two exact likelihood views and their corresponding SNR paths. Turning them into a trainable recipe involves two design decisions: a mixed-support SNR sampler that covers both families under a single loss (\S\ref{subsec:mixed_snr}), and a posterior-view converter that presents the corrupted state to the backbone neural network architecture in a form that matches the underlying geometry (\S\ref{subsec:converter}).

\begin{algorithm}[tb]
\caption{DSL training}
\label{alg:dsl_training}
{\small
\begin{algorithmic}[1]
\Require Sequence length $L$; mixing weight $\lambda \in [0,1]$; ROAR endpoint values $\gamma_{\min}, \gamma_{\max}$; log-normal parameters $(\mu, \sigma)$; backbone $f_\theta$ with converter; learning rate $\eta$.
\For{each training step}
  \State Sample sentence $\vs = (s_1,\dots,s_L) \sim P_{\mathrm{data}}$; embed $\vx_i = \enc(s_i)$.
  \State Draw $u \sim \mathrm{Unif}[0,1]$.
  \If{$u < 1 - \lambda$} \Comment{ROAR branch}
      \State Sample reveal-set size $k \sim \mathrm{Unif}\{0,\dots,L{-}1\}$, subset $A \subseteq [L]$ with $|A|{=}k$.
      \State Set $\gamma_i \gets \gamma_{\max}$ for $i \in A$, and $\gamma_i \gets \gamma_{\min}$ otherwise.
  \Else \Comment{Continuous path branch}
      \State Sample $\gamma_i \sim \mathrm{LogNormal}(\mu, \sigma^2)$ independently for each $i \in [L]$.
  \EndIf
  \State Sample $\eps_i \sim \mathcal N(\vzero, \mI)$; form noisy tokens $\vz_i \gets \gamma_i \vx_i + \sqrt{\gamma_i}\,\eps_i$.
  \State Compute converter outputs $\vm_i^{\mathrm{conv}}$ from each $\vz_i$ via \eqref{eq:converter}.
  \State Run backbone: $p_\theta(\cdot \mid \vz) \gets f_\theta(\vm_1^{\mathrm{conv}}, \dots, \vm_L^{\mathrm{conv}})$.
  \State Update $\theta \gets \theta - \eta\,\nabla_\theta\,\sum_{i=1}^L -\log p_\theta(s_i \mid \vz)$.
\EndFor
\end{algorithmic}
}
\end{algorithm}

\vheader
\subsection{Mixed-SNR training}
\label{subsec:mixed_snr}
\vheader

% The two likelihood views in \S\ref{subsec:nll_views} cover disjoint SNR paths. A faithful DSL objective must place training mass on both, written as a convex combination
To cover both state families in a single model, the DSL objective should place training mass on both, written as a convex combination
\begin{equation}
\Ls_{\mathrm{theory}}
\;=\;
\lambda\,\Ls_{\mathrm{cont\mbox{-}NLL}}
\;+\;
(1-\lambda)\,\Ls_{\mathrm{ROAR\mbox{-}NLL}}.
\label{eq:theory_loss}
\end{equation}

\textbf{One CE loss for both branches.} \quad
$\Ls_{\mathrm{ROAR\mbox{-}NLL}}$ is already a token-level categorical
log-likelihood, so cross-entropy on $P_\theta(s_i\mid\vz)$ is its exact form.
$\Ls_{\mathrm{cont\mbox{-}NLL}}$ comes from the path-integral identity
\eqref{eq:probability} and is naturally an embedding-space MSE on the Bayes
denoiser $\xhat_i(\vz)=\E[\vx_i\mid\vz]$. We parameterize this denoiser through
a categorical token posterior,
\begin{equation}
\xhat_{\theta,i}(\vz)
=
\sum_{v\in\gV} P_\theta(s_i=v\mid\vz)\,\vx_v .
\end{equation}
Minimizing token-level cross-entropy makes $P_\theta(\cdot\mid\vz)$ match the
true posterior $P(\cdot\mid\vz)$; at this optimum, the posterior mean above
recovers the same Bayes denoiser that minimizes the path-integral MSE. Thus,
for the continuous branch, CE is a posterior-matching surrogate for the exact
MSE NLL rather than a separate heuristic. It also avoids embedding-space
optimization pathologies known from prior continuous diffusion language
models~\citep{gulrajani2023likelihood}. Full details are in
Appendix~\ref{app:ce_surrogate}.

\textbf{Implementation as mixed SNR sampling.} \quad
Since both branches share the same CE loss, the $\lambda$-weighted mixture of losses equals, in expectation, a single CE loss under a $\lambda$-mixture of SNR distributions. We implement the latter directly: with probability $\lambda$ we draw token-wise SNRs from a continuous-path log-normal distribution, and with probability $1-\lambda$ from a ROAR-style endpoint distribution concentrated near the low- and high-SNR extremes. Both branches are trained under
\begin{equation}
\Ls_{\mathrm{CE}}
\;=\;
-\,\E_{(\vs,\vz)\sim q_{\mathrm{train}}}
\sum_{i=1}^{L}
\log p_\theta(s_i\mid \vz),
\label{eq:ce_loss}
\end{equation}
realizing \eqref{eq:theory_loss} as a single CE objective over a $\lambda$-mixed SNR support. 
We summarize the full training procedure in Algorithm~\ref{alg:dsl_training}; hyperparameters are listed in Appendix~\ref{app:mixed_snr_support}.

\vheader
\subsection{Posterior-view converter and backbone}
\label{subsec:converter}
\vheader

The DSL posterior lives in token-embedding geometry, while a Transformer backbone expects token-like inputs that interact stably under self-attention. We introduce a converter that maps each noisy token $\vz_i$ to a mixture-of-tokens representation,
\begin{equation}
q_i^{\mathrm{conv}}(v \mid \vz_i)
\;=\;
\softmax_\tau\!\big(\langle \vz_i, \vx_{v}\rangle + b_v\big),
\qquad
\vm_i^{\mathrm{conv}}
\;=\;
\sum_{v \in \gV \cup \{[\text{MASK}]\}} q_i^{\mathrm{conv}}(v \mid \vz_i)\,\vx_{v},
\label{eq:converter}
\end{equation}
Here $v$ indexes candidate token values, $\vx_{v}=\enc(v)$, and $\tau>0$ is a learnable
converter temperature hyperparameter.
We feed the sequence $(\vm_1^{\mathrm{conv}},\dots,\vm_L^{\mathrm{conv}})$ to a standard Transformer encoder / DiT-style denoiser. Intuitively, $\vm_i^{\mathrm{conv}}$ is a vocabulary-weighted combination of unit-norm token embeddings: when $\vz_i$ is very noisy the mixture is broad, and as the SNR grows it concentrates on a single token.

The converter has two properties required for the analysis in Section~\ref{sec:analysis}: $\norm{\vm_i^{\mathrm{conv}}}_2 \le 1$ for every $\vz_i$ regardless of SNR, and $\vm_i^{\mathrm{conv}}$ concentrates smoothly as SNR grows. Together these make the backbone see a stationary input geometry in which semantic preference is carried by direction and confidence grows by controlled radial motion. This mirrors the norm--direction posterior decomposition made
explicit later in \Cref{eq:posterior_decomposition}.

Finally, DSL uses no informative timestep input: for compatibility with pretrained masked-diffusion DiTs~\citep{sahoo2024simple}, we keep the timestep embedding module but feed $t=0$ throughout training and sampling, consistent with \eqref{eq:t_invariant_denoiser} since $\vz_i$ and $\vm_i^{\mathrm{conv}}$ already encode the corruption regime.

\vheader
\section{Posterior Geometry and Pathwise Error Control}
\label{sec:analysis}
\vheader

Sections~\ref{sec:method}--\ref{sec:training} define DSL as a posterior-state framework: exact NLL views identify the continuous-SNR and endpoint state families, and mixed-SNR training places support on both.
This section explains what DSL geometry buys us for finite models.
In the unit-sphere localization channel, the Bayes target is not a family of time-indexed denoisers $\{\xhat(\cdot,t)\}_t$, but a single posterior-mean map $\xhat(\vz)=\E[\vx\mid\vz]$.
Different inference paths therefore differ only in which realized states $\vz$ they visit, and all query the same denoising map.
We analyze this shared map below: its geometry separates token identity from confidence and yields bounded, smooth posterior motion, which helps control path-dependent approximation error.

\textbf{Norm--direction posterior geometry.} \quad
Fix a token position. Written as a function of the noisy token $\vz$, the induced token posterior is
\begin{equation}
p(v \mid \vz)
\;\propto\;
\pi_v \exp\!\big(\langle \vz, \vx_v\rangle\big)
\;=\;
\pi_v \exp\!\big(\norm{\vz}_2 \cos\theta_v\big),
\label{eq:posterior_decomposition}
\end{equation}
where $\vx_v := \enc(v)$ is the unit-norm embedding of token $v$, $\pi_v$ is its prior, and $\theta_v$ is the angle between $\vz$ and $\vx_v$. 
Thus direction controls token preference while norm controls posterior concentration; this is a property of the unit-sphere channel, not a calibration objective.
Figure~\ref{fig:posterior_geometry} verifies that the trained converter follows this decomposition: trajectories organize by direction, and increasing SNR primarily increases concentration.

\begin{figure}[t]
    \centering
    \begin{subfigure}[b]{0.45\linewidth}
        \centering
        \includegraphics[width=.95\linewidth]{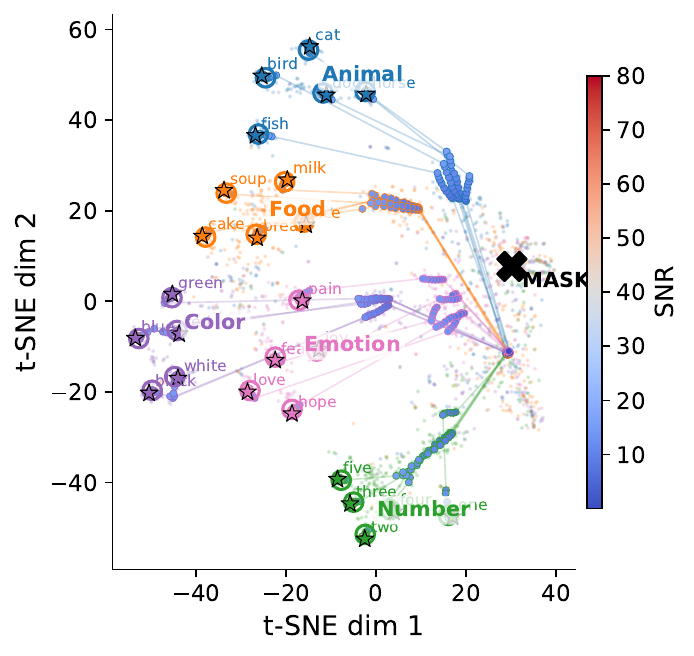}
        \caption{}
        \label{fig:tsne_traj}
    \end{subfigure}
    \hfill
    \begin{subfigure}[b]{0.45\linewidth}
        \centering
        \includegraphics[width=.95\linewidth]{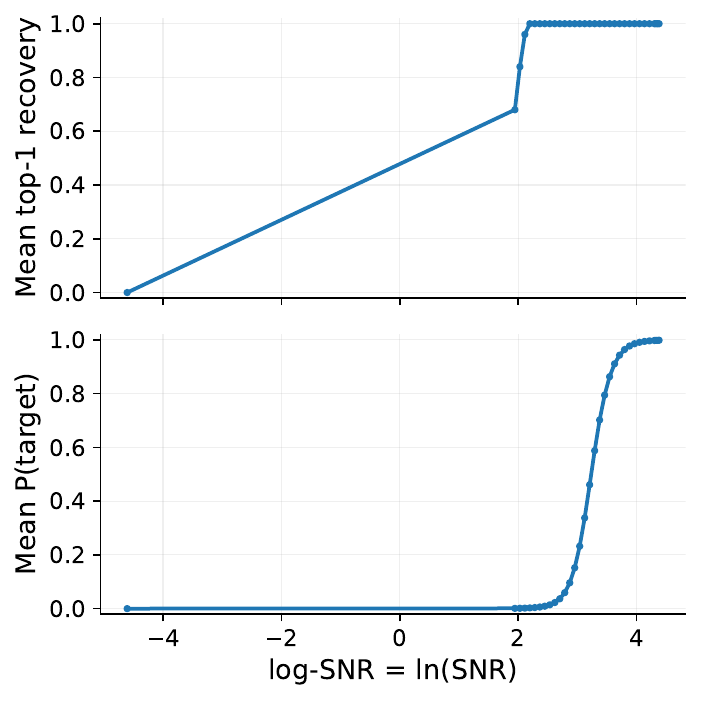}
        \caption{}
        \label{fig:recovery_curves}
    \end{subfigure}
    \caption{\textbf{The DSL posterior decomposes into direction and norm axes.}
    \textbf{(a)} t-SNE projection of converter outputs for 25 probe tokens spanning five semantic classes. Trajectories sweep $\vz_i = \gamma\,\ve_v$ from \texttt{[MASK]} across $\gamma \in [10^{-2},80]$ and are colored by SNR.
    \textbf{(b)} Mean top-1 token recovery and target-token probability as $\gamma$ increases.}
    \label{fig:posterior_geometry}
\end{figure}

\textbf{Bounded and Lipschitz posterior motion.} \quad
The converter maps each noisy state to a bounded mixture-of-tokens input,  with $\norm{\vm_i^{\mathrm{conv}}}_2\leq 1$.
At fixed converter temperature, this bounds the achievable logit range and gives an entropy envelope for the token posterior. 
Moreover, for each position $i$ and fixed $\vz_{-i}$, the Bayes posterior-mean map $\vm_i(\vz_i):=\E[\vx_i\mid\vz]$ is 1-Lipschitz:
\begin{equation}
\norm{\vm_i(\vz_i)-\vm_i(\vz_i')}_2 \le \norm{\vz_i-\vz_i'}_2 .
\label{eq:lipschitz}
\end{equation}
Thus a local perturbation at one position cannot be amplified by the Bayes denoising field at that position.
Full entropy bounds are in Appendix~\ref{subsec:geometry}, and the proof of ~\eqref{eq:lipschitz} is in Appendix~\ref{app:lipschitz}.

\begin{table*}[t]
\centering
\scriptsize
\vspace*{-.4cm}
\setlength{\tabcolsep}{5.0pt}
\renewcommand{\arraystretch}{1.08}
\caption{
OWT ($L=1024$) unconditional generation. 
We report MAUVE$\uparrow$ as the primary metric; the full table with MAUVE / GenPPL / SentEnt is deferred to Appendix~\ref{sec:additional_exp}. 
Bold marks the best MAUVE in each column among non-reference rows. 
Rows marked $\ddagger$ are taken from \citep{wang2025remdm}; rows marked $\S$ from \citep{rout2025anchored}. 
All DSL-FT rows use the same DSL-finetuned checkpoint. 
DSL-FT + ROAR (naive) and DSL-FT + MDLM are non-refinement decoder ablations: tokens are committed once and never remasked. 
The ReMDM-family rows add iterative remasking refinement.
}
\label{tab:owt_mauve}
\begin{tabular*}{\textwidth}{@{\extracolsep{\fill}}lcccc}
\hline
Method & $T=128$ & $T=256$ & $T=512$ & $T=1024$ \\
\hline
Data (reference) & 1.000 & 1.000 & 1.000 & 1.000 \\
AR (reference)   & --    & --    & --    & 0.760 \\
\hline
SEDD (absorb)          & 0.007 & 0.008 & 0.009 & 0.008 \\
MDLM$^{\ddagger}$      & 0.015 & 0.023 & 0.031 & 0.042 \\
MDLM+FB$^{\ddagger}$   & 0.064 & 0.086 & 0.103 & 0.133 \\
MDLM+DFM$^{\ddagger}$  & 0.041 & 0.098 & 0.168 & 0.254 \\
ReMDM$^{\ddagger}$     & 0.057 & 0.216 & 0.350 & 0.403 \\
PRISM$^{\ddagger}$     & 0.175 & 0.268 & 0.281 & 0.260 \\
PRISM-loop$^{\ddagger}$& 0.118 & 0.294 & 0.423 & 0.527 \\
ADLM$^{\S}$            & 0.140 & 0.349 & 0.573 & 0.699 \\
\hline
DSL-FT + ROAR (naive) & --    & --    & --    & 0.551 \\
DSL-FT + MDLM &
0.402 &
0.481 &
0.506 &
0.495 \\
DSL-FT + ReMDM (confidence-based) &
0.615 &
0.622 &
\textbf{0.707} &
0.610 \\
DSL-FT + ReMDM-loop (principled $\eta_{\mathrm{cap}}$) &
\textbf{0.639} &
\textbf{0.661} &
0.651 &
\textbf{0.722} \\
\hline
\end{tabular*}
\end{table*}

\vheader
\section{Experiments}
\label{sec:experiments}
\vheader

On OpenWebText (OWT), we evaluate the gains from DSL training with masked diffusion sampling, the additional benefits of ReMDM samplers~\citep{wang2025remdm}, and the composition of continuous and remasking samplers through hybrid sampling. 
All three evaluations use the same DSL-finetuned checkpoint, connecting the per-token SNR paths of Sec.~\ref{subsec:arbitrary_paths} to practical sampling procedures. 
We additionally evaluate likelihood on Text8 using a model trained from scratch.

% We evaluate how one DSL checkpoint supports mask diffusion sampling, remasking refinement \citep{wang2025remdm}, and hybrid continuous--discrete sampling on OpenWebText (OWT). 
% These experiments connect the per-token SNR paths of Sec.~\ref{subsec:arbitrary_paths} to practical sampling procedures.
% On OpenWebText, we evaluate the gains from DSL training under non-refinement decoding, the additional benefits of remasking, and the compatibility of continuous denoising with discrete refinement through hybrid sampling. All three evaluations use the same DSL-finetuned checkpoint.
% We additionally evaluate likelihood on Text8 using a model trained from scratch.

% We separate the evaluation into three stages. 
% First, we use simple non-refinement MDM decoders to isolate the effect of DSL training itself. 
% Second, we add ReMDM-family masked refinement \citep{wang2025remdm} to test whether stronger iterative samplers can further exploit the DSL-trained posterior. 
% Third, we evaluate a hybrid continuous-then-MDM sampler to show that the same checkpoint can compose continuous and discrete inference stages. 
% Unless otherwise stated, all DSL results use the same DSL-finetuned checkpoint.

% We evaluate how one DSL checkpoint supports endpoint decoding, masked refinement, and hybrid continuous--discrete sampling on OpenWebText (OWT). These experiments connect the per-token SNR paths of Section~\ref{subsec:arbitrary_paths} to practical decoding procedures. We additionally evaluate likelihood on Text8 using a model trained from scratch.

\vheader
\subsection{Setup}
\label{subsec:setup}
\vheader

All OWT DSL results use the same checkpoint fine-tuned from the official MDLM model. Following the MDLM/ReMDM protocol~\citep{sahoo2024simple,wang2025remdm}, we generate 5k sequences of length $L=1024$ with GPT-2 BPE tokenization and use nucleus sampling with top-$p=0.9$ where applicable. 
Our primary metric is MAUVE, computed with GPT-2 Large embeddings and $K=500$ clusters; GenPPL and sentence entropy are reported as diagnostics in Appendix~\ref{sec:additional_exp}. 
For DSL-FT + ReMDM-loop at $T=128$, bootstrap resampling of generated texts and variation of the clustering seed give MAUVE standard deviations of $0.020$ and $0.017$, respectively; these are not training-seed error bars.

For masked diffusion sampling in DSL, masked positions have $\vz_i=\vzero$ (SNR $0$), while sampled tokens $s_i$ are represented by $\vz_i=\gamma_{\max}\enc(s_i)$, with $\gamma_{\max}=100$ approximating the clean endpoint. 
All samplers use the same per-token posteriors $p_\theta(s_i\mid\vz)$ without timestep conditioning.

% We evaluate unconditional OpenWebText (OWT) generation under the MDLM/ReMDM protocol: GPT-2 BPE tokenization, sequence length $L=1024$, 5k samples, nucleus sampling with $\text{top-}p=0.9$ where applicable, and MAUVE computed with GPT-2 Large embeddings and $K=500$ clusters~\citep{sahoo2024simple,wang2025remdm}. 
% OWT is our checkpoint-compatibility setting: we fine-tune DSL from the official MDLM checkpoint, whereas Text8 uses \textit{from-scratch} training. 
% Full sampling protocols and diagnostic metrics are in Appendix~\ref{sec:additional_exp}.

% \textbf{Primary metric.} \quad
% Following ReMDM, we use MAUVE as the primary OWT generation metric and report GenPPL/SentEnt only as diagnostics~\citep{wang2025remdm}. 
% GenPPL can be reduced by narrowing the generation distribution and sacrificing entropy/diversity; MAUVE instead compares model and human text distributions via divergence frontiers~\citep{pillutla2021mauve}. 
% Full MAUVE / GenPPL / SentEnt results are in Appendix~\ref{sec:additional_exp}.

% \textbf{Uncertainty.}\quad
% At $T{=}128$ (DSL-FT + ReMDM-loop), bootstrap resampling of the 5k generated
% texts (50 replicates) gives MAUVE standard deviation $0.020$, and resampling
% MAUVE's clustering seed (10 seeds) gives $0.017$. %These describe one configuration; they are not training- or sampling-seed error bars.

\vheader
\subsection{Masked diffusion sampling without remasking}
\label{subsec:simple_mdm_ablation}
\vheader

% We implement MDLM-style sampling and ROAR using the same mask ($SNR=0$) and high-SNR token representations (almost clean). 
% Both start with all positions masked and commit each token once. 
% The MDLM sampler~\citep{sahoo2024simple} progressively reveals subsets of masked positions, samples their tokens from the DSL posterior, and replaces their zero states with the high-SNR representations above. 
% ROAR uses the same operation but reveals one position per step in a uniformly random order, requiring $T=L$ steps. 
% Neither sampler remasks committed tokens.

% Under the same vanilla MDLM sampling rule, replacing the original checkpoint with DSL-FT raises MAUVE from $0.015$--$0.042$ to $0.402$--$0.506$ across the four budgets in Table~\ref{tab:owt_mauve}. 
% Matched-step continued training of MDLM does not yield a comparable gain (Appendix~\ref{app:controls}), supporting the benefit of DSL adaptation beyond additional training alone. 
% ROAR reaches $0.551$ at $T=1024$, demonstrating that the same checkpoint also supports sequential random-order generation without refinement. 
% Full protocols are in Appendices~\ref{app:sampler_details} and~\ref{app:roar_setup}.

We first evaluate the same DSL-finetuned checkpoint using MDLM and random order AR (ROAR) sampling without remasking.
The vanilla MDLM sampler follows the official implementation of \citep{sahoo2024simple}, progressively converting masked positions into sampled tokens.
ROAR instead reveals one token at a time in a uniformly random order.
In both samplers, masked positions have $\vz_i=\vzero$ (SNR $0$). After sampling token $s_i$, we set $\vz_i=\gamma_{\max}\enc(s_i)$ (SNR $=\gamma_{\max}$).
Full protocols are given in Appendices~\ref{app:sampler_details} and~\ref{app:roar_setup}.
% Before evaluating iterative refinement, we first isolate the effect of DSL training itself. 
% We decode the same DSL-finetuned checkpoint with two simple MDM-style samplers that do not remask or refine committed tokens. 
% The first is the vanilla MDLM sampler from the official implementation of \citep{sahoo2024simple}: it monotonically converts masked positions into sampled tokens and never revisits them. 
% The second is ROAR, a random-order autoregressive sampler that reveals one token at a time using the DSL denoiser as an MDM posterior estimator. 
% Full ROAR details are given in Appendix~\ref{app:roar_setup}, and the shared OWT sampling protocol is specified in Appendix~\ref{app:sampler_details}.

Under the same vanilla MDLM sampler, replacing the original checkpoint with DSL-FT raises MAUVE from $0.015/0.023/0.031/0.042$ to $0.402/0.481/0.506/0.495$ at $T=128,256,512,1024$, respectively (Table~\ref{tab:owt_mauve}).
This shows that DSL improves generation quality across all evaluated budgets without changing the sampling rule or introducing remasking. 
Matched-step continued training of MDLM does not produce comparable gains, supporting the benefit of DSL adaptation beyond additional training alone (Appendix~\ref{app:controls}).
% Table~\ref{tab:owt_mauve} shows that DSL already performs strongly under these non-refinement decoders. 
% Replacing the MDLM checkpoint with the DSL-finetuned checkpoint under the same vanilla MDLM sampler raises MAUVE from $0.015/0.023/0.031/0.042$ to $0.402/0.481/0.506/0.495$ across $T=128,256,512,1024$. 
% This is the cleanest training-side comparison in the table: the sampler is the standard monotone MDM decoder, so the improvement comes from the DSL-trained posterior rather than from remasking or self-correction.

ROAR shows that the smae checkpoint also supports random-order autoregressive generation without retraining. 
It performs a single reveal pass with $T=L=1024$, committing each token exactly once, and performs no confidence-based ordering or remasking.
Despite this simple sampling procedure, DSL-FT + ROAR reaches MAUVE $0.551$, exceeding DSL-FT + MDLM at the same budget and several prior baselines. 
Together, these results demonstrate strong generation quality under both MDLM and ROAR sampling using one DSL checkpoint. 
Training-component ablations are reported in Appendix~\ref{app:controls}, Table~\ref{tab:controls}.

\begin{algorithm}[tb]
\caption{DSL sampling with ReMDM sampler.
Adapted from \cite[Algorithm~1]{wang2025remdm}; in DSL, the ReMDM mask state is represented by zero SNR.}
\label{alg:remdm_sampler_main}
\small
\begin{algorithmic}
    \State \textbf{Input:} DSL denoiser $\vx_\theta$, sampling steps $T$, noise schedule $\alpha_t$, clean endpoint $\mathrm{SNR}_{\max}$, \remdm{remask schedule $\sigma_t$}.
    \State \textbf{Initialize:} set every token embedding to the mask endpoint, i.e., $\operatorname{SNR}(\vz_T)=0$.
    \For{$i=T,T-1,\ldots,1$}
        \State $t=i/T,\quad s=(i-1)/T$; set $\alpha_t,\alpha_s$.
        \State \remdm{Set $\sigma_t\in[0,\sigma_t^{\max}]$, where $\sigma_t^{\max}=\min\{1,(1-\alpha_s)/\alpha_t\}$.}
        \State Predict clean-token posteriors $\widehat{\vx}=\vx_\theta(\vz_t)$.
        \State Form the ReMDM endpoint posterior
        $
        p_\theta(\vz_s\mid\vz_t)
        =
        q_{\remdm{\sigma}}^{\mathrm{DSL}}
        \bigl(\vz_s\mid\vz_t,\vx=\widehat{\vx}\bigr),
        $
        where ReMDM's mask atom is interpreted as $\operatorname{SNR}=0$ and its clean-token atom as $\operatorname{SNR}=\mathrm{SNR}_{\max}$.
        \State Sample $\vz_s\sim p_\theta(\vz_s\mid\vz_t)$.
        \State Realize sampled clean tokens at $\mathrm{SNR}_{\max}$; \remdm{realize remasked tokens by resetting their SNR to $0$.}
        \State Set $\vz_t\leftarrow\vz_s$.
    \EndFor
    \State \textbf{Output:} decoded tokens from the final clean-endpoint latents.
\end{algorithmic}
\end{algorithm}

\vheader
\subsection{Masked diffusion sampling with ReMDM}
\label{subsec:masked_refinement}
\vheader

We next evaluate the same DSL checkpoint with ReMDM samplers~\citep{wang2025remdm}, extending our evaluation to masked diffusion sampling with remasking.
Algorithm~\ref{alg:remdm_sampler_main} summarizes the sampling method. 
Starting from the zero-SNR mask state, the sampler repeatedly predicts clean-token posteriors and reveals masked positions by sampling tokens and assigning their high-SNR representations. 
Remasking resets previously generated positions to zero-SNR, allowing them to be sampled again conditioned on the remaining tokens.
We evaluate two schedules in \citep{wang2025remdm} -- confidence-based schedule ReMDM-conf and ReMDM-loop; 
the latter activates remasking within a loop window using the step-budgeted capped schedule $\eta_{\mathrm{cap}}(T)$.
Full settings are given in Appendix~\ref{app:sampler_details}.
% The ReMDM-loop row in Table~\ref{tab:owt_mauve} instantiates this template with the step-budgeted capped schedule $\eta_{\mathrm{cap}}(T)$; the confidence-based row changes only the remasking policy. 
% Vanilla MDLM, ROAR, and hybrid sampling details are deferred to Appendix~\ref{app:sampler_details}, Appendix~\ref{app:roar_setup}, and Appendix~\ref{app:hybrid_sampler}.
Remasking implements the decreasing-SNR direction of Sec.~\ref{subsec:arbitrary_paths}: resetting a position to $\vz_i=\vzero$ is the $\gamma_i'=0$ case of reverse localization (Proposition~\ref{prop:reverse}).
The same DSL denoiser then predicts from the updated state.
The reverse kernel specifies the SNR-lowering operation, while the remasking schedule determines when and which positions are revisited. 
% Remasking in Algorithm~\ref{alg:remdm_sampler_main} is the decreasing direction of \S\ref{subsec:arbitrary_paths}: resetting a token to SNR $0$ is Proposition~\ref{prop:reverse} with $a=0$, so the same denoiser applies after the reset without recalibration. The kernel is exact; the choice of which tokens to remask is not covered by it and follows ReMDM~\citep{wang2025remdm}.

Table~\ref{tab:owt_mauve} compares these samplers against published diffusion baselines and the preceding DSL results without remasking.
Both DSL remasking variants outperform ReMDM baseline across all four step budgets.
ReMDM-conf gives the best DSL result at $T=512$, with MAUVE $0.707$;
ReMDM-loop performs best at $T=128$, $256$, and $1024$, reaching $0.639$, $0.661$, and $0.722$, respectively.
At $T=1024$, the same checkpoint achieves $0.495$ with MDLM sampling and $0.551$ with ROAR, compared with $0.722$ with ReMDM-loop.
These results show that DSL achieves strong generation quality with ReMDM samplers, extending its performance acros different sampling strategies without retraining.

% Having isolated the training-side gain with simple decoders, we next evaluate whether iterative masked refinement can further exploit the same DSL checkpoint. 
% Table~\ref{tab:owt_mauve} compares DSL with ReMDM-family decoders against published diffusion baselines under the controlled OWT protocol. 
% Both DSL refinement rows substantially improve MAUVE across step budgets from $T=128$ to $T=1024$. 
% The confidence-based schedule performs best at $T=512$, reaching MAUVE $0.707$, while the step-budgeted ReMDM-loop schedule with principled $\eta_{\mathrm{cap}}(T)$ performs best at $T=128$, $T=256$, and $T=1024$, reaching $0.639$, $0.661$, and $0.722$ respectively.

% The comparison with the preceding ablation rows separates training and sampling effects. 
% DSL-FT + MDLM and DSL-FT + ROAR show that DSL training already improves non-refinement decoding. 
% The ReMDM-family rows then show that iterative remasking further unlocks the same checkpoint, improving from DSL-FT + MDLM's $0.495$ and ROAR's $0.551$ at $T=1024$ to $0.722$ with ReMDM-loop. 
% Thus DSL improves the denoiser first, and refinement policies further exploit the improved posterior.

\begin{wraptable}[12]{r}{0.52\linewidth}
\vspace{-3.2em}
\centering
\scriptsize
\setlength{\tabcolsep}{3.2pt}
\renewcommand{\arraystretch}{1.05}
\caption{
Few-step OWT generation with the DSL hybrid continuous-then-MDM sampler.
All rows use the same DSL-finetuned checkpoint. For each
$(T_{\rm cont},T_{\rm MDM})$, we report a balanced
$\sigma_{\rm switch}$ from a small handoff sweep. The full sweep with
MAUVE / GenPPL / SentEnt is in Appendix~
\ref{sec:additional_exp}.
%\ref{app:hybrid_sampler}.
}
\label{tab:hybrid}
\begin{tabular}{ccccc}
\hline
$T_{\rm cont}$ & $T_{\rm MDM}$ & NFE & $\sigma_{\rm switch}$ & MAUVE$\uparrow$ \\
\hline
16 & 16 & 32 & 0.49 & 0.501 \\
16 & 32 & 48 & 0.49 & 0.662 \\
16 & 48 & 64 & 0.49 & 0.702 \\
32 & 16 & 48 & 0.46 & 0.587 \\
32 & 32 & 64 & 0.46 & 0.589 \\
32 & 48 & 80 & 0.46 & 0.724 \\
\hline
\end{tabular}
% \vspace{-0.8em}
\end{wraptable}
\FloatBarrier

\vheader
\subsection{Hybrid continuous--discrete sampling}
\label{subsec:hybrid_sampling}
\vheader

% The hybrid sampler uses finite-SNR states before switching to endpoint refinement. First, an EDM-style solver~\citep{karras2022elucidating} updates continuous states $\vy$ down to a switching noise level $\sigma_{\mathrm{switch}}$, querying the DSL posterior mean at $\vz=\vy/\sigma^2$. At the switch, we sample a complete discrete draft from the token posteriors and encode it at high SNR. We then repeatedly reset uncertain positions to zero and resample them conditioned on the retained tokens, with a decreasing mask ratio. Both stages use the same checkpoint; Appendix~\ref{app:hybrid_sampler} gives the full algorithm.

% Table~\ref{tab:hybrid} reports the continuous and discrete stage budgets separately. With $(T_{\rm cont},T_{\rm disc})=(16,32)$, MAUVE reaches $0.662$; increasing the budgets to $(32,48)$ gives $0.724$, numerically comparable to ReMDM-loop at $T=1024$. The allocation matters: $(16,48)$ achieves $0.702$, whereas $(32,32)$ gives $0.589$. These results demonstrate that continuous denoising and discrete refinement can be composed within one DSL model without retraining.

We also evaluate a hybrid sampler that combines continuous sampling and masked diffusion sampling with remasking using the same DSL checkpoint. 
The hybrid sampler first runs a short continuous denoising stage, here we adopt an EDM-style sampler~\citep{karras2022elucidating} which denoises until $\sigma_{\mathrm{switch}}$. And then we continue with confidence-based remasking samplin: selected tokens are reset to SNR$=0$ and resampled. 
This tests whether one checkpoint can combine continuous and masked sampling within a single generation run.
Appendix~\ref{app:hybrid_sampler} gives the full procedure.
% , where each position carries soft finite-SNR token evidence, and then switches to an MDM-style endpoint stage that commits the remaining uncertainty into discrete tokens. 
% This experiment is designed as a path-compatibility test: can one DSL checkpoint support both continuous intermediate states and discrete endpoint decoding in a single generation run?

Table~\ref{tab:hybrid} reports strong generation quality with short continuous and masked sampling stages. 
With $(T_{\rm cont},T_{\rm disc})=(16,32)$, MAUVE reaches $0.662$; with $(32,48)$, it reaches $0.724$, comparable to ReMDM-loop at $T=1024$.
The budget allocation and switching point jointly affect quality: $(16,48)$ achieves $0.702$, whereas $(32,32)$ gives $0.589$ despite the same nominal total step count.
Thus, allocating more steps to continuous sampling does not necessarily improve the final result.
These configurations demonstrate the potential of combining samplers; systematic optimization of the switching point and compute allocation remains future work.
% At 48 total steps, the hybrid sampler reaches MAUVE $0.662$, already exceeding the non-refinement DSL-FT + MDLM result at $T=1024$. 
% With 80 total steps, it reaches $0.724$, comparable to the best ReMDM-family result in Table~\ref{tab:owt_mauve}. 
% The allocation between continuous and discrete stages matters: at the same nominal 64-step budget, $(T_{\mathrm{cont}},T_{\mathrm{MDM}})=(16,48)$ reaches $0.702$, while $(32,32)$ reaches $0.589$. 
% We therefore view hybrid sampling as a promising DSL-specific inference path, but not yet as an exhaustively optimized decoding policy.

\begin{wraptable}{r}{0.36\linewidth}
\vspace{-4.0em}
\scriptsize
    \centering
    \caption{Bits Per Character (BPC) on Text8 test set. Continuous diffusion and any-order entries are variational upper bounds; AR and flow entries are exact NLL.
    DSL is the only continuous-state entry without time conditioning, and is trained for 0.5M steps versus 1M for the diffusion baselines.}
    \label{tab:text8}
    \resizebox{\linewidth}{!}{%
        \begin{tabular}{lr}
            \toprule
            Method &  BPC ($\downarrow$) \\
            \midrule
            \textit{Continuous Diffusion} & \\
            Plaid (MD4 impl.) & $\le$ 1.48 \\
            DSL (Ours) & $\le$ 1.45 \\
            \midrule
            \textit{Discrete Diffusion} & \\
            Mult. Diffusion & $\le$ 1.72\\
            D3PM Uniform & $\le$ 1.61\\
            D3PM Absorb & $\le$ 1.45\\
            SEDD Absorb & $\le$ 1.41 \\
            MDLM & $\le$ 1.40 \\
            MD4 & $\le$ 1.37 \\
            \midrule
            \textit{Autoregressive} & \\
            Transformer AR & \textbf{1.23}\\
            IAF/SCF  & 1.88\\
            AR Argmax Flow & 1.39\\
            AR Discrete Flow & \textbf{1.23}\\
            \midrule
            \textit{Any-order Autoregressive} & \\
            ARDM & $\le$ 1.43 \\
            MAC & $\le$ 1.40\\
            \bottomrule
        \end{tabular}
    }
\vspace{-3.8em}
\end{wraptable}
\FloatBarrier

\vheader
\subsection{Likelihood-based evaluation on Text8}
\label{subsec:text8}
\vheader

The DSL framework admits two likelihood estimators following Section~\ref{subsec:nll_views}: the path-integral identity \eqref{eq:probability} and the ROAR estimator \eqref{eq:roar_estimator}. Table~\ref{tab:text8} reports BPC on Text8. DSL achieves $\le 1.45$ BPC, the first continuous-state NLL estimate to approach the masked-discrete-diffusion range (MDLM $\le 1.40$, MD4 $\le 1.37$). Continuous-state estimates are upper bounds via ELBO and so cannot match the exact NLL accessible to masked and AR models; the relevant comparison is to other continuous-state baselines, where DSL improves on Plaid by $\ge 0.03$ BPC. 
Setup details in Appendix~\ref{sec:additional_exp}.

% \vheader
% \section{Related Work}
% \vheader

% DSL connects two lines of work. Continuous diffusion LMs operate over noisy token embeddings and expose finite-SNR states, but typically rely on timestep/noise conditioning and have trailed masked discrete diffusion on language benchmarks~\citep{li2022diffusion,strudel2022self,dieleman2022continuous,gulrajani2023likelihood}. Masked diffusion and ReMDM decode through endpoint mask/reveal states, with remasking enabling revision of visible tokens~\citep{austin2021structured,sahoo2024simple,shi2024simplified,wang2025remdm}. Related hybrid methods connect discrete and Gaussian processes through loss-level mixtures or alignments~\citep{sahoo_duality,pynadath2025candi}. DSL differs in that its stochastic-localization SDE constrains the entire intermediate path: posterior-mean drift plus Brownian noise gives finite-SNR states a path-level identity from masked to clean tokens. This SDE structure underlies the path-integral identity under the Bayes-optimal denoiser, whereas loss-level hybrids do not impose an analogous conservative field or path-independence guarantee. DSL is therefore complementary: unit-sphere posterior coordinates make the Bayes denoiser SNR-invariant, so continuous denoising, ROAR revealing, and masked remasking become different per-token SNR paths through one time-agnostic model.
\vheader
\section{Related Work}\label{sec:related_work}
\vheader

\textbf{Continuous and masked diffusion language models.}\quad
Continuous diffusion LMs generate text by denoising token
embeddings, typically with timestep or noise-level
conditioning~\citep{li2022diffusion,strudel2022self,
dieleman2022continuous,gulrajani2023likelihood}.
Bayesian Flow Networks instead update distribution parameters
through Bayesian inference and condition their network on both
these parameters and time~\citep{graves2025bayesianflownetworks}.
Token geometry has also been explored: CDCD normalizes token
embeddings~\citep{dieleman2022continuous}, while RDLM defines
diffusion on a hypersphere~\citep{jo2025continuousdiffusionmodellanguage}.
Masked diffusion models generate tokens by
unmasking~\citep{austin2021structured,sahoo2024simple,shi2024simplified},
and ReMDM permits previously generated tokens to be
remasked~\citep{wang2025remdm}.
Hybrid approaches also connect discrete and Gaussian
diffusion~\citep{sahoo_duality,pynadath2025candi}.
DSL combines unit-norm clean token embeddings with stochastic
localization, while allowing noisy tokens to evolve in $\R^d$.
This combination makes the Bayes-optimal denoiser SNR-invariant
and allows one network to support continuous and masked diffusion
sampling (Appendix~\ref{app:continuous_vs_dsl}).
Our companion work scales DSL to an 8B masked diffusion LM DSL-LLaDA~\citep{yu-etal-2026-dsl}.

\textbf{Stochastic localization and classical spins.}\quad
DSL builds on stochastic
localization~\citep{eldan2013thin,montanari2023sampling}
and its use in sampling
algorithms~\citep{alaoui2024samplingsherringtonkirkpatrickgibbsmeasure,
grenioux2026stochasticlocalizationiterativeposterior}.
Sekimoto~\citep{ken24PhysRevE.109.014106,ken25PhysRevE.111.044102} studies autonomous Langevin drifts that are martingales along their own dynamics, obtaining, in the isotropic case, a drift expressed as the mean of a classical spin on the unit sphere.
DSL arrives at a corresponding structure from the other direction:
starting with \textbf{constant-norm token embeddings}, stochastic localization yields an autonomous posterior-mean drift that is also a martingale along the localization process.
The two constructions differ in their base measures: Sekimoto's isotropic formulation uses the uniform measure on the sphere, whereas DSL uses the generally non-uniform discrete data distribution $P(\vx)$ over token sequences.
Sekimoto's subsequent work~\citep{sekimoto2026autonomous} discusses the connection between these approaches.

\textbf{Diffusion and autoregressive generation.}\quad
Dieleman~\citep{dieleman2024spectral} interprets image diffusion as approximate autoregression over spatial frequencies.
This perspective motivated us to investigate how continuous denoising and token-by-token generation could be understood within a common framework for language.
DSL provides such a framework through per-token SNR paths: continuous denoising gradually accumulates evidence across positions, while autoregressive generation reveals tokens one at a time.
Under the Bayes-optimal denoiser, these paths admit an exact likelihood correspondence.
This connection also guides model design: DSL supports both sampling procedures with one SNR-invariant denoiser.

\vheader
\section{Conclusion}
\vheader

We introduced Discrete Stochastic Localization (DSL), a continuous-state framework for discrete sequence generation with unit-sphere token embeddings.
Combining this geometry with stochastic localization makes the Bayes-optimal clean-token posterior depend only on the noisy state, without additional SNR conditioning.
DSL implements this posterior with a time-invariance denoiser and a mixture-of-tokens interface for standard Transformer/DiT backbones.
A single trained denoiser supports continuous denoising, masked diffusion with remasking, random-order autoregression, and hybrid sampling, with strong OpenWebText generation results and Text8 likelihood bounds approaching those of masked diffusion models.
These results show how token geometry and the choice of noisy-state parameterization can bring continuous and discrete generation together within one model.
% We introduced Discrete Stochastic Localization (DSL), a continuous-state framework for discrete sequence generation with unit-sphere token embeddings. By making the clean-token posterior only depend on the noisy state rather than a nominal timestep, DSL yields a time-invariance denoiser and a mixture-of-tokens interface for standard Transformer/DiT backbones. A single DSL supports masked refinement, random-order AR, and hybrid continuous-then-MDM sampling, improving OpenWebText generation across paths and approaching the MDM likelihood range on Text8. These results suggest parameterizing continuous diffusion for language as posterior-state denoising over noisy tokens.

\textbf{Limitations.} \quad
Our evaluation is limited to OpenWebText generation and Text8 likelihood. The reported ReMDM-family and hybrid samplers are not the result of an exhaustive decoding-policy search; they are intended to test whether DSL training improves the underlying posterior and whether one checkpoint can be queried along multiple SNR paths. More optimized remasking schedules, hybrid handoff rules, larger models, longer contexts, and conditional generation remain important directions for future work.

\clearpage

\begin{ack}
Research was supported by National Science Foundation CAREER grant no.~IIS-2046086.
\end{ack}

\bibliographystyle{plain}
\bibliography{refs}

\appendix
\clearpage % flush any pending two-column floats before switching layout
\onecolumn

\section*{Appendix Overview}
\addcontentsline{toc}{section}{Appendix Overview}
\vspace{0.5em}

\noindent\textbf{\large
\hyperref[sec:derivations]{Appendix A: Mathematical Derivations}}
\begin{itemize}
    \setlength{\itemsep}{2pt}
    \item[\S\ref*{app:notation}]
    \hyperref[app:notation]{Notation and Summary Table}

    \item[\S\ref*{app:snr_invariant_denoiser}]
    \hyperref[app:snr_invariant_denoiser]{Optimal Denoiser is SNR-Invariant}

    \item[\S\ref*{app:arbitrary_paths}]
    \hyperref[app:arbitrary_paths]{Exact Likelihood over Arbitrary SNR Paths}

    \item[\S\ref*{app:conditional_unconditional}]
    \hyperref[app:conditional_unconditional]{Equivalence of Conditional and Unconditional Dynamics}

    \item[\S\ref*{app:reverse_sl}]
    \hyperref[app:reverse_sl]{Reverse Localization}

    \item[\S\ref*{app:ar_contours}]
    \hyperref[app:ar_contours]{Autoregressive Contours as a Special Case}

    \item[\S\ref*{app:roar_likelihood}]
    \hyperref[app:roar_likelihood]{Exact ROAR Likelihood Decomposition and Estimator}

    \item[\S\ref*{app:ce_surrogate}]
    \hyperref[app:ce_surrogate]{Cross-Entropy as a Surrogate for the Path-Integral MSE}

    \item[\S\ref*{subsec:posterior_matching}]
    \hyperref[subsec:posterior_matching]{Posterior-State Matching and Inference-Time Error Correction}

    \item[\S\ref*{subsec:geometry}]
    \hyperref[subsec:geometry]{Confidence Geometry: Norm--Direction Decomposition and Bounded Entropy}

    \item[\S\ref*{app:lipschitz}]
    \hyperref[app:lipschitz]{Lipschitz Continuity of the Induced Posterior-Mean Map}
\end{itemize}
% \noindent\textbf{\large \hyperref[sec:derivations]{Appendix A: Mathematical Derivations}}
% \begin{itemize}
%     \setlength{\itemsep}{2pt}
%     \item[\S A.1] \hyperref[app:notation]{Notation and Summary Table}
%     \item[\S A.2] \hyperref[app:snr_invariant_denoiser]{Optimal Denoiser is SNR-Invariant}
%     \item[\S A.3] \hyperref[app:arbitrary_paths]{Exact Likelihood over Arbitrary SNR Paths}
%     \item[\S A.4] \hyperref[app:conditional_unconditional]{Equivalence of Conditional and Unconditional Dynamics}
%     \item[\S A.5] \hyperref[app:ar_contours]{Autoregressive Contours as a Special Case}
%     \item[\S A.6] \hyperref[app:roar_likelihood]{Exact ROAR Likelihood Decomposition and Estimator}
%     \item[\S A.7] \hyperref[app:ce_surrogate]{Cross-Entropy as a Surrogate for the Path-Integral MSE}
%     \item[\S A.8] \hyperref[subsec:posterior_matching]{Posterior-State Matching and Inference-Time Error Correction}
%     \item[\S A.9] \hyperref[subsec:geometry]{Confidence Geometry: Norm--Direction Decomposition and Bounded Entropy}
%     \item[\S A.10] \hyperref[app:lipschitz]{Lipschitz Continuity of the Induced Posterior-Mean Map}
% \end{itemize}

\vspace{0.8em}
\noindent\textbf{\large \hyperref[app:continuous_vs_dsl]{Appendix B: Continuous States Are Not Enough}}
\begin{itemize}
    \setlength{\itemsep}{2pt}
    \item[] \hyperref[app:continuous_vs_dsl]{Posterior Coordinates vs.\ Timestep-Conditioned Denoising}
\end{itemize}

\vspace{0.8em}
\noindent\textbf{\large \hyperref[sec:implementation_details]{Appendix C: Training and Sampling Details}}
\begin{itemize}
    \setlength{\itemsep}{2pt}
    \item[\S C.1] \hyperref[app:training_settings]{OWT Finetuning Settings}
    \item[\S C.2] \hyperref[app:mixed_snr_support]{Mixed SNR Support Used in Training}
    \item[\S C.3] \hyperref[app:masked_refinement_alg]{Masked Refinement with a DSL Checkpoint}
    \item[\S C.4] \hyperref[app:sampler_details]{Sampling Settings on OWT}
    \item[\S C.5] \hyperref[app:roar_setup]{ROAR Sampler Details}
    \item[\S C.6] \hyperref[app:hybrid_sampler]{Hybrid Sampler Details}
\end{itemize}

\vspace{0.8em}
\noindent\textbf{\large \hyperref[sec:mech_appendix]{Appendix D: Mechanistic Analysis}}
\begin{itemize}
    \setlength{\itemsep}{2pt}
    \item[\S D.1] \hyperref[sec:cyclic_dataset]{Cyclic Toy Setup and Correction Example}
    \item[\S D.2] \hyperref[app:targeted_remasking]{Why Targeted Remasking Matters}
    \item[\S D.3] \hyperref[app:self_generated_drafts]{Robustness to Self-Generated Intermediate Drafts}
    \item[\S D.4] \hyperref[app:diagnostics]{Sampling Diagnostics and Over-Refinement}
\end{itemize}

\vspace{0.8em}
\noindent\textbf{\large \hyperref[sec:calibration_appendix]{Appendix E: Calibration and Endpoint-Smoothing Ablation}}
\begin{itemize}
    \setlength{\itemsep}{2pt}
    \item[\S E.1] \hyperref[app:smooth_roar_why]{Atomic vs.\ Smoothed ROAR Endpoints}
    \item[\S E.2] \hyperref[app:ece_reliability]{Near-Clean Calibration}
    \item[\S E.3] \hyperref[app:speed_quality_smoothing]{Downstream Step--Quality Tradeoff}
\end{itemize}

\vspace{0.8em}
\noindent\textbf{\large
\hyperref[sec:additional_exp]{Appendix F: Additional Experiments}}
\begin{itemize}
    \setlength{\itemsep}{2pt}
    \item[\S\ref*{app:additional_owt_exp}]
    \hyperref[app:additional_owt_exp]{OpenWebText Generation}

    \item[\S\ref*{app:controls}]
    \hyperref[app:controls]{Matched-Compute and Component Controls}

    \item[\S\ref*{app:additional_text8_exp}]
    \hyperref[app:additional_text8_exp]{Text8 Likelihood Evaluation}
\end{itemize}
% \noindent\textbf{\large \hyperref[sec:additional_exp]{Appendix F: Additional Experiments}}
% \begin{itemize}
%     \setlength{\itemsep}{2pt}
%     \item[\S F.1] \hyperref[app:additional_owt_exp]{OpenWebText Generation}
%     \item[\S F.2] \hyperref[app:additional_text8_exp]{Text8 Likelihood Evaluation}
% \end{itemize}

\vspace{0.8em}
\noindent\textbf{\large \hyperref[sec:limitations_impacts]{Appendix G: Limitations and Broader Impacts}}

\vspace{0.8em}
\noindent\textbf{\large \hyperref[sec:assets]{Appendix H: Assets and Licenses}}

\vspace{0.8em}
\noindent\textbf{\large \hyperref[sec:gen_text_examples]{Appendix I: Additional Qualitative Samples}}
\begin{itemize}
    \setlength{\itemsep}{2pt}
    \item[\S I.1] \hyperref[app:text8_samples]{Text8 Samples}
    \item[\S I.2] \hyperref[app:owt_samples]{OpenWebText Samples}
\end{itemize}

\section{Mathematical Derivations}
\label{sec:derivations}

This appendix collects the derivations underlying the DSL method in Section~\ref{sec:method}.
The main text only uses three consequences of these results:
(i) the denoiser is realized-state invariant under unit-sphere token geometry,
(ii) exact likelihood views motivate both continuous and ROAR-style support in training,
and (iii) the induced posterior-mean map is smooth.
Here we provide the full proofs and auxiliary technical remarks.

\subsection{Notation and Summary Table}
\label{app:notation}

% Keep Table~\ref{tab:summary} as the first item in the appendix.
\begin{table*}[h]
\centering
\begin{tabular}{c p{10cm}}
\hline
\textbf{Symbol} & \textbf{Description} \\[6pt]
\hline
$i = 1,\ldots,L$ & \underline{I}ndex for sequences of length $L$ \\[6pt]
$t \in [0, \infty)$ & Index for continuous \underline{t}ime dynamics, equal to SNR\\[6pt]
$\snr_i \in [0, \infty)$ & Per token SNR, defined on a contour or as function of $t$\\[6pt]
$s_i \in \mathcal{V}$ & $i$-th \underline{s}ymbol/token
from the vocabulary $\mathcal{V}$ \\[6pt]
$\vx_i = \enc(s_i) \in \mathbb{R}^d$ & \underline{E}mbed token as vector on unit hyper-sphere surface ($\norm{\vx}=1$) \\[6pt]
$\vz_i \in \mathbb{R}^d$ & Noisy embedding of the $i$-th token at noise level $t$ \\[6pt]
$d\vz = \vx~dt + dW$ & Conditional SDE for dynamics \\[6pt]
$\vz = t~\vx + \sqrt{t}~\eps$ & Sample conditional marginal \\[6pt]
$\vz \sim \mathcal N(t ~ \vx, t)$ & Alternate form to sample conditional marginal \\[6pt]
$d\vz = \xhat(\vz)~ dt + dW$ & Unconditional SDE with equivalent dynamics to conditional SDE \\[6pt]
$\xhat(\vz, t) = \mathbb E[\vx | \vz(t)] $ & Optimal denoiser would generally depend on time or noise level \\[6pt]
$\xhat(\vz) = \mathbb E_{\vx} [ \vx e^{\vx \cdot \vz} ] / \mathbb E_{\vx} [e^{\vx\cdot \vz} ] $ & Key result: optimal denoiser doesn't depend on time for spherical embeddings \\[6pt]
$\xhat(\vz) = \nabla_{\vz} \LSE_{\vx \in \mathcal{D}} (\vz\cdot \vx) $ & For empirical distribution, relates to gradient of cumulant generating function, or modern Hopfield energy\cite{hopfield} \\[6pt]
$\begin{gathered}
-\log P(\vs) = -\log P(\vx) \\
 = \half \int_0^\infty dt ~ \mathbb E_{\vz(t)|\vx}[\norm{\vx - \xhat(\vz)}^2] 
\end{gathered}
$ & Probability relates to MMSE, for any one-to-one embedding~\cite{guo,kong2023informationtheoretic,wu2024contrastivediffusionloss} \\[6pt]
\hline
\end{tabular}
\caption{Summary of notation and key relations.}\label{tab:summary}
\end{table*}

\subsection{Optimal Denoiser is SNR-Invariant}
\label{app:snr_invariant_denoiser}

% This subsection should remain almost unchanged.
% It supports the main-text theorem that the posterior depends on the realized state
% rather than the nominal corruption label.
We now derive the optimal denoiser for the noise channel with per token SNR described in the main text.
The denoiser is as follows, where we first re-write with Bayes rule, then expand the Gaussian noise channel.  
\benn
\xhat(\vz, \vsnr) &\equiv \mathbb E_{p_\vsnr(\vx|\vz)} [\vx]  = \frac{\sum_\vx  p_\vsnr(\vz | \vx) P(\vx) \vx}{p_\vsnr(\vz)} = \frac{\sum_\vx  p_\vsnr(\vz | \vx) P(\vx) \vx}{\sum_{\bar x} p_\vsnr(\vz| \bar \vx) P(\bar \vx) }\\
&= \frac{\sum_\vx  \exp(-\half \norm{\vz - \vsnr \vx}^2/\vsnr) /\mathcal Z(\vsnr) P(\vx) \vx}{\sum_{\bar x} \exp(-\half \norm{\vz - \vsnr \bar \vx}^2/ \vsnr) /\mathcal Z(\vsnr) P(\bar \vx)  } \\
&= \frac{\sum_\vx  \exp(\vz \cdot \vx)   P(\vx) \vx}{\sum_{\bar x} \exp(\vz \cdot \bar \vx) P(\bar \vx)  } \\
\xhat(\vz, \vsnr) &= \xhat(\vz) = \mathbb E_{\vx} [ \vx e^{\vx \cdot \vz} ] / \mathbb E_{\vx} [e^{\vx\cdot \vz} ] 
\eenn
The division and multiplication between $\vsnr$ and $\vz$ should be understood to be applied component-wise, per token.  
We canceled out the $\norm{\vz}$ term, and due to the normed embedding, $\exp(-\half \vsnr \norm{\vx}^2)$ terms become constant and cancel out, leaving us with no dependence on $\vsnr$ at all. 
The distribution $P(\vx) e^{\vz \cdot \vx} / \mathcal Z(\vz)$ is sometimes referred to as an exponentially tilted distribution of $P(\vx)$. This form also can be written as $\nabla_z \log \mathbb E_{\vx} [e^{\vx\cdot \vz} ]$, which is a gradient of the cumulant generating function. 

The results above also hold for the simpler case where $\snr_i(t) = t$, leading to a denoiser which is invariant to $t$.
Recent work has also empirically noted that even for traditional diffusion models, noise conditioning is not always necessary or useful\cite{no_noise_conditioning}.

\paragraph{Constant norm is also necessary.}
For any two support points $\vx,\vx'$ of $P$,
\[
\log\frac{P_{\vsnr}(\vx\mid\vz)}{P_{\vsnr}(\vx'\mid\vz)}
= \log\frac{P(\vx)}{P(\vx')} + \langle \vz,\vx-\vx'\rangle
  - \tfrac12\sum_{i=1}^{L}\snr_i\big(\norm{\vx_i}^2-\norm{\vx'_i}^2\big).
\]
The first two terms do not involve $\vsnr$. The posterior is therefore
$\vsnr$-invariant for all $\vz$ and all $\vsnr\in\R_{+}^{L}$ if and only if
$\norm{\vx_i}=\norm{\vx'_i}$ for every position $i$ and every pair of support
points. Under a single global SNR $\snr_i\equiv t$ the weaker condition
$\norm{\vx}^2=\text{const}$ suffices. Unit-norm token embeddings satisfy both.

\subsection{Exact Likelihood over Arbitrary SNR Paths}
\label{app:arbitrary_paths}

% Paste your current subsection:
% "Probability with Arbitrary SNR Paths"
%
% This supports Section 2.4.1 in the main text:
% exact NLL from continuous SNR contours.

% STATE theorem from Verdu, on grad of KL for noisy channel
We make use of the following theorem, re-stated with our setting and notation, from \cite{palomar2005gradient}, to prove \Cref{eq:probability}.

\begin{theorem}[\citep{palomar2005gradient}, Thm.~5]
\label{thm:palomar-divergence}
Consider the signal model
\[
\vz_i = \snr_i \vx_i + \sqrt{\snr_i} \eps_i,
\]
where \( \vx \) is arbitrarily distributed with finite second-order moments, and \( \eps_i \sim \mathcal{N}(0, \mathbf{I}) \) is independent Gaussian noise.

Then the gradient of the Kullback–Leibler divergence between the conditional output distribution \( p_\vsnr(\vz|\vx) \) and the unconditional output distribution \( p_\vsnr(\vz) \) is
\[
\nabla_{\vsnr} D\big(p_\vsnr(\vz|\vx) \,\|\, p_\vsnr(\vz)\big) 
= \half \mathbf E(\vx, \vsnr),
\]
where \( E_i(\vx, \vsnr) \equiv \mathbb{E}_{p_\vsnr(\vz | \vx)}[\norm{\vx_i - \hat{\vx}_i(\vz)}^2] \), and 
\( \hat{\vx}(\vz) = \mathbb{E}_{p_\vsnr(\vx|\vz)}[\vx] \) 
is the MMSE estimator.
\end{theorem}

% FUNDAMENTAL THEOREM OF CALC for Line integral
We use this theorem in conjunction with the fundamental theorem of calculus for line integrals to obtain the following.
\begin{align}\label{eq:ftc}
    \int_C d\vsnr \cdot \nabla_\vsnr D\big(p_\vsnr(\vz|\vx) \,\|\, p_\vsnr(\vz)\big) = \half \int_C d\vsnr \cdot \mathbf E(\vx, \vsnr)  = \left . D\big(p_\vsnr(\vz|\vx) \,\|\, p_\vsnr(\vz)\big) \right |_{\vsnr_0}^{\vsnr_1}
\end{align}
Here, $\vsnr_0, \vsnr_1$ define the endpoints of the contour, which must be piecewise smooth. 
Next, let's consider some of the limits.
\begin{align}
 \lim_{\vsnr \rightarrow 0} D\big(p_\vsnr(\vz|\vx) \,\|\, p_\vsnr(\vz)\big) &= 0\\
\lim_{\vsnr \rightarrow \infty} D\big(p_\vsnr(\vz|\vx) \,\|\, p_\vsnr(\vz)\big) &= -\log P(\vx) 
\end{align}
The second limit can be observed as follows.
\begin{align}
    D\big(p_\vsnr(\vz|\vx) \,\|\, p_\vsnr(\vz)\big) &= \mathbb E_{p_\vsnr(\vz|\vx)} \left [\log \frac{p(\vz | \vx)}{p(\vz)} \right ] \nonumber \\
    &= \mathbb E_{p_\vsnr(\vz|\vx)} \left [\log \frac{P(\vx|\vz)}{P(\vx)} \right ] \nonumber\\
    &= \mathbb E_{p_\vsnr(\vz|\vx)} \left [\log P(\vx|\vz) \right ] - \log P(\vx) \label{eq:kl}
\end{align}
Because $P(\vx)$ is discrete, $\lim_{\vsnr \to \infty} P(\mathbf{x} \mid \mathbf{z}) = 1 ~~ \text{almost surely}$.

Combining \Cref{eq:kl} and \Cref{eq:ftc}, we obtain the following. 
\begin{align}\label{eq:end}
    -\log P(\vx) = \half \int_{C_\vsnr} d\bar \vsnr \cdot \mathbf E(\vx, \bar\vsnr) - \mathbb E_{p_\vsnr(\vz|\vx)} \left [\log P(\vx|\vz) \right ]
\end{align}
Where the curve $C_\vsnr$ is piecewise smooth, begins at $\vsnr=0$ and ends at $\vsnr$. 
We can use this result to obtain a path that starts at zero and ends at $t$, and \Cref{eq:probability} when we let the end of the contour go to infinity. \qed

\subsection{Equivalence of Conditional and Unconditional Dynamics}
\label{app:conditional_unconditional}

% Paste your current subsection:
% "Equivalence of conditional and unconditional dynamics"
Consider the marginals generated conditioned on the data with the following noise channel. 
$$p_\vsnr(\vz) \equiv \sum_{\vx} p_\vsnr(\vz|\vx) P(\vx)$$
Where $p_\vsnr(\vz|\vx)$ is a Gaussian noise channel, which could be defined as follows,
$$\vz_i = \snr_i ~ \vx_i + \sqrt{\snr_i} ~\eps_i, \qquad \eps\sim\mathcal N(0,I).$$
This is equivalent to the conditional SDE dynamics, 
\begin{align} \label{eq:cond_appendix}
  d\vz_i &= \vx_i \dot \snr_i~dt + \sqrt{\dot \snr_i} dW_i, ~~ \vx \sim P(\vx),
\end{align} 
as can be seen from direct integration. 
We can use the Fokker-Planck equation to understand the marginal dynamics of this process. 
Consider a particular contour, $\vsnr(t)$, as defined in the text, which goes from zero to infinity. At $\vsnr(0) = 0$, we have $$p_{\vsnr(0)}(\vz) = \delta(\vz).$$
The Fokker-Planck equation gives the following for the conditional dynamics.
\[
\frac{\partial}{\partial t} p_{\vsnr(t)}(\vz \mid \vx) = 
- \sum_i \nabla_{\vz_i} \cdot \left( \dot \snr_i ~\vx \, p_{\vsnr}(\vz \mid \vx) \right) 
+  \sum_i \tfrac12\dot \snr_i \Delta_{\vz_i} \, p_{\vsnr}(\vz \mid \vx)
\]
By marginalizing over $P(\vx)$, get the following. 
\begin{align}
\frac{\partial}{\partial t} p_{\vsnr(t)}(\vz) &= \sum_\vx P(\vx) \frac{\partial}{\partial t} p_{\vsnr(t)}(\vz \mid \vx)\\ 
&= - \sum_i \nabla_{\vz_i} \cdot \left( \dot \snr_i ~ \sum_\vx P(\vx) \vx_i \, p_{\vsnr}(\vz \mid \vx) \right) 
+  \sum_i \tfrac12\dot \snr_i \Delta_{\vz_i} \, p_{\vsnr}(\vz) \\
&= - \sum_i \nabla_{\vz_i} \cdot \left( \dot \snr_i ~ \xhat_i(\vz) \, p_{\vsnr}(\vz) \right) 
+  \sum_i \tfrac12\dot \snr_i \Delta_{\vz_i} \, p_{\vsnr}(\vz) \label{eq:marginalize}
\end{align}
In the last line, we used the definition of the optimal denoiser in terms of the posterior mean. 
The diffusion term is identical in the conditional and marginal equations and is untouched by marginalization over $P(\vx)$; the equivalence rests on matching the drift terms.

The goal is to design an unconditional dynamics that obeys the same evolution of the marginal. 
The unconditional dynamics are defined as follows. 
\begin{align} \label{eq:uncond_appendix}
  d\vz_i &= \xhat_i(\vz) \dot \snr_i~dt + \sqrt{\dot \snr_i} dW_i,
\end{align} 
Applying the Fokker-Planck equation to this expression directly matches \Cref{eq:marginalize}. \qed

\subsection{Reverse Localization}
\label{app:reverse_sl}
% \yunshu{Revise this part.}

The main text uses only the kernel. Here we also record its continuous-time form and one consequence for resampling.

\paragraph{Proof of the kernel.}
Let $\vz_i=\gamma_i\vx_i+\rvb_i(\gamma_i)$ and
$\vz_i'=\gamma_i'\vx_i+\rvb_i(\gamma_i')$, where $\rvb_i$ is
a standard Brownian motion and $0\le\gamma_i'<\gamma_i$.
The Brownian bridge identity gives
\[
\rvb_i(\gamma_i')\mid\rvb_i(\gamma_i)
\sim \mathcal N\!\left(
    \frac{\gamma_i'}{\gamma_i}\rvb_i(\gamma_i),
    \gamma_i'\left(1-\frac{\gamma_i'}{\gamma_i}\right)\mI
\right).
\]
Substituting $\rvb_i(\gamma_i)=\vz_i-\gamma_i\vx_i$, the
conditional mean becomes
\[
\gamma_i'\vx_i+
\frac{\gamma_i'}{\gamma_i}(\vz_i-\gamma_i\vx_i)
=\frac{\gamma_i'}{\gamma_i}\vz_i.
\]
Thus,
\[
\vz_i'\mid\vz_i,\vx
\sim \mathcal N\!\left(
    \frac{\gamma_i'}{\gamma_i}\vz_i,
    \gamma_i'\left(1-\frac{\gamma_i'}{\gamma_i}\right)\mI
\right),
\]
which is independent of $\vx$.
The resulting marginal satisfies
\[
\E[\vz_i'\mid\vx]=\gamma_i'\vx_i,
\qquad
\operatorname{Cov}(\vz_i'\mid\vx)
=
\left[
\left(\frac{\gamma_i'}{\gamma_i}\right)^2\gamma_i
+\gamma_i'\left(1-\frac{\gamma_i'}{\gamma_i}\right)
\right]\mI
=\gamma_i'\mI.
\]
The kernel therefore maps the conditional channel at $\gamma_i$
to that at $\gamma_i'$. Since it is independent of $\vx$,
the same mapping holds after marginalizing over $P(\vx)$.
Finally, taking $\gamma_i\to\infty$ with
$\vz_i/\gamma_i\to\vx_i$ and $\gamma_i'$ fixed gives
$\mathcal N(\gamma_i'\vx_i,\gamma_i'\mI)$.

\paragraph{Continuous-time form.}
Along a segment with $\dot\gamma_i<0$, the bridge satisfies
\[
d\vz_i=\frac{\dot\gamma_i}{\gamma_i}\,\vz_i\,du+\sqrt{-\dot\gamma_i}\,d\rvw_i .
\]
The same drift follows from time reversal of \eqref{eq:snr_uncond}. The marginal
score is $\nabla_{\vz_i}\log p_{\vsnr}(\vz)=\xhat_i(\vz)-\vz_i/\gamma_i$, so the
reverse drift is $\xhat_i(\vz)-\nabla_{\vz_i}\log p_{\vsnr}(\vz)=\vz_i/\gamma_i$
per unit of $d\gamma_i$. The denoiser cancels. The conditional SDE
\eqref{eq:snr_cond} gives the same drift, since
$\nabla_{\vz_i}\log p_{\vsnr}(\vz\mid\vx)=\vx_i-\vz_i/\gamma_i$.
The drift contracts toward the origin at the time-varying rate $1/\gamma_i$, so
this is a Brownian bridge to $\vzero$, not an Ornstein--Uhlenbeck process.

\paragraph{Distribution-preserving resampling.}
Let $S\subseteq\{1,\ldots,L\}$ be a fixed set of token positions.
Choose $0\le\gamma_i'\le\gamma_i$ for $i\in S$, and set
$\gamma_i'=\gamma_i$ otherwise.
Starting from $(\vx,\vz)\sim P(\vx)\,p_{\vsnr}(\vz\mid\vx)$,
apply the reverse kernel independently at positions in $S$,
leaving the other positions unchanged.
The kernel calculation above gives
$(\vx,\vz')\sim P(\vx)\,p_{\vsnr'}(\vz'\mid\vx)$.
Drawing $\vx'\sim P_{\vsnr'}(\cdot\mid\vz')$ therefore preserves
this joint distribution, so $\vx'\sim P$.
This is a data-augmentation Gibbs step.

When positions outside $S$ are at the clean endpoint ($\gamma=\inf$),
resampling reduces to drawing from
$P(\vx_S\mid\vx_{-S},\vz_S')$.
Setting $\gamma_i'=0$ for all $i\in S$ gives the block Gibbs
update $P(\vx_S\mid\vx_{-S})$.
For $|S|=1$, this is a single categorical draw from the exact
token posterior. For $|S|>1$, independently drawing tokens from
their marginal posteriors generally does not reproduce the joint
conditional distribution and introduces parallel-decoding
factorization error.

\paragraph{Non-monotone paths and the likelihood identity.}
$D\big(p_{\vsnr}(\vz\mid\vx)\,\|\,p_{\vsnr}(\vz)\big)$ is a function of
$\vsnr$, so \eqref{eq:ftc} holds for any piecewise-smooth contour, including
decreasing segments. Monotonicity is needed only to write the dynamics in the
form \eqref{eq:snr_cond}--\eqref{eq:snr_uncond}, where $\sqrt{\dot\gamma_i}$
must be real. Path-independence is a property of the Bayes-optimal denoiser and
does not transfer to a trained network.

\subsection{Autoregressive Contours as a Special Case}
\label{app:ar_contours}

The continuous SNR path formulation also recovers the standard autoregressive factorization as a special contour choice.
In particular, when the contour reveals one token at a time while treating previously revealed tokens as effectively clean and future tokens as maximally corrupted, the path-integral likelihood reduces to the standard chain rule.
For completeness, we isolate this argument here, since the main text only needs the high-level consequence.

% Move the AR chain-rule derivation here from your current
% "Optimal denoiser is SNR invariant" subsection.
% Specifically, move the derivation around Eq. ar_prob into this subsection.
\textbf{Conditional denoisers} \quad
Replacing $P(\vx)$ above with $P(\vx|\vy)$ gives the conditional denoiser, and we still get a result invariant to SNR.
$$
\xhat(\vz, y, \vsnr) = \xhat(\vz, y) = \mathbb E_{P(\vx|y)} [ \vx e^{\vx \cdot \vz} ] / \mathbb E_{P(\vx|y)} [e^{\vx\cdot \vz} ] 
$$

We used the following result to show the connection between probability estimates using certain contours and autoregressive models. 
\begin{align*}
-\log P(\vx) &= \frac12 \int_{C_{AR}} \mathbf E(\vx, \vsnr) \cdot d\vsnr = \frac12 \sum_{i=1}^L \int_{C_i} \mathbf E(\vx, \vsnr)\cdot d\vsnr \nonumber \\
&= \sum_{i=1}^L \frac12 \int_{0}^\infty d\gamma_i E_i(\vx, (\infty, \dotsc, \snr_i, 0, \dotsc)) = -\sum_{i=1}^L \log P(\vx_i|\vx_{<i})
\end{align*}
To show this, we need to demonstrate the last equality explicitly.
% \begin{align*}
% - \log P(\vx_i|\vx_{<i}) &= \frac12 \int_{0}^\infty d\gamma_i E_i(\vx, (\infty, \dotsc, \snr_i, 0, \dotsc)) \\
% &= \frac12 \int_{0}^\infty d\gamma_i 
% ~ \mathbb E_{p_\vsnr(\vz | \vx)}[ \norm{\vx_i - \xhat_i(\vz)}^2 ], ~~ \vsnr = (\infty, \dotsc, \snr_i, 0, \dotsc) \\
% &= \frac12 \int_{0}^\infty d\gamma_i 
% ~ \mathbb E_{p_\snr_i(\vz_i | \vx_i)}[ \norm{\vx_i - \xhat_i(\vz_{<i}, \vz_i, 0, \dotsc)}^2 ], 
% \end{align*}
\begin{align}\label{eq:ar_prob}
- \log P(\vx_i \mid \vx_{<i})
&= \frac{1}{2} \int_0^\infty d\gamma_i ~ E_i\left(\vx, \vsnr = (\infty, \dotsc, \gamma_i, 0, \dotsc)\right) 
\end{align}

First, we rewrite the left hand side. 
\begin{align*}
\frac{1}{2} \int_0^\infty d\gamma_i ~ E_i\left(\vx, \vsnr = (\infty, \dotsc, \gamma_i, 0, \dotsc)\right) &= \frac{1}{2} \int_0^\infty d\gamma_i ~ \mathbb{E}_{p_{\vsnr}(\vz \mid \vx)} \left[ \left\| \vx_i - \xhat_i(\vz) \right\|^2 \right] \\
&= \frac{1}{2} \int_0^\infty d\gamma_i ~ \mathbb{E}_{p_{\snr_i}(\vz_i \mid \vx_i)} \left[ \left\| \vx_i - \xhat_i( \vx_{<i},\vz_i, 0, \dotsc) \right\|^2 \right]
\end{align*}
In the last line, we note that for tokens at SNR 0, there is no information, so the optimal denoiser will be conditioned on a constant, zero. For tokens with infinite SNR, $\vz_{<i}$ is equivalent in distribution to $\vx_{<i}$, which we replace in the argument, with a slight abuse of notation. 

Next, we see what happens when we try to write $ P(\vx_i|\vx_{<i})$, in terms of the optimal \emph{conditional denoiser} for denoising only $\vz_i$, using \Cref{eq:probability}, where the contours simplify to 1-d integrals as we only have one SNR, $\snr_i$. 
In that case we know we can write the probability in terms of the conditional denoiser, which we will call $\tilde \vx(\vz_i, \vx_{<i})$ to distinguish from $\xhat$. 
$$
- \log P(\vx_i|\vx_{<i}) = \frac12 \int_{0}^\infty d\gamma_i 
~ \mathbb E_{p_{\snr_i}(\vz_i | \vx_i)}[ \norm{\vx_i - \tilde \vx_i(\vz_i, \vx_{<i})}^2 ]
$$
We can see that this matches the expression above, except for the denoisers. These denoisers are conditioned on the same information, and have to be MMSE denoisers for the same denoising problem, and therefore they must be equivalent, and equality of \Cref{eq:ar_prob} must hold.
\textbf{Conditional denoisers} \quad
Replacing $P(\vx)$ above with $P(\vx|\vy)$ gives the conditional denoiser, and we still get a result invariant to SNR.
$$
\xhat(\vz, y, \vsnr) = \xhat(\vz, y) = \mathbb E_{P(\vx|y)} [ \vx e^{\vx \cdot \vz} ] / \mathbb E_{P(\vx|y)} [e^{\vx\cdot \vz} ] 
$$

\subsection{Exact ROAR Likelihood Decomposition and Estimator}
\label{app:roar_likelihood}

This subsection derives the exact random-order conditional likelihood that motivates the ROAR-style endpoint branch in DSL training.
The purpose is not to claim that the main experiments directly optimize this objective.
Rather, it provides a second exact likelihood source, complementary to the continuous path-integral likelihood, and justifies including masking/reveal-like endpoint states in the training support.

% Paste your current subsection:
% "Random Order AutoRegressive (ROAR) Log Likelihood"
%
% Keep the permutation-average derivation, set-based regrouping,
% and Monte Carlo estimator.
% Keep the ROAR NLL estimator algorithm if space allows.
% Delete the unfinished placeholder algorithm "Mask prediction with DSL".
In DSL, we point out that masked diffusion modeling is equivalent to paths in SNR space where we completely denoise one variable (going from zero SNR to infinite SNR), conditioned on zero noise (infinite SNR) for the already ``unmasked'' variables and infinite noise (zero SNR) for the ``masked'' variables. If we choose our SNR paths to be the ROAR paths, then we just need to model conditional probabilities of the next variable, given previously unmasked variables, similar to masked diffusion models. 
However, this idea doesn't exactly match up with practical implementations of masked diffusion models like \citep{llada} because they write their objective (Eq. 3 in \citep{llada} for example) in terms of probability of an unmasked set, conditioned on a masked set. 
Below, we derive a similar expression starting from the perspective that we are trying to model ROAR SNR paths. Unlike the result in \citep{llada}, our derivation is path-based and makes no explicit mention of time.  
%
% \footnote{A paper GV reviewed called ``Information-Theoretic Discrete Diffusion'' had a similar result that they called ``Time-Free Likelihood via I-MDCE''. We should be sure to cite it when it comes out. The bound  derived above is like an importance weighted version of their bound.} 
%\citep{jeon2025information}

Consider a multivariate probability distribution, $p(x_1, \ldots, x_n)$. We can always decompose this distribution autoregressively as $p(x_1, \ldots, x_n) = \prod_{i=1}^n p(x_i | x_{1:i})$. We could in principle use any ordering of the variables for this decomposition. Let $\pi$ be a permutation of $\{1,\dots,n\}$, indexed with vector notation. For $n=3$ we could have $\pi_{1:3} = (3, 1, 2)$, for example. For a particular permutation, $\pi$, we could write the AR decomposition as follows. 
$$p(x_1, \ldots, x_n) = \prod_{i=1}^n p(x_{\pi_i} | x_{\pi_{1:{i-1}}})$$
There are $n!$ permutations, so we can take the average over the set of all permutations, $\mathcal S_n$, and switch to log probabilities for convenience. 
\be
\log p(x_1, \ldots, x_n) = \frac{1}{n!}\sum_{\pi \in \mathcal S_n} \sum_{i=1}^n \log p(x_{\pi_i} | x_{\pi_{1:{i-1}}})
\ee 
Now, note that this expression has many equivalent terms that can be grouped together. We are summing over a total of $ n 2^n$ log probability terms, but terms like $\log p(x_3| x_1, x_2) = \log p(x_3|x_2, x_1)$, i.e., the order of the conditioned information doesn't matter. This suggests that we should re-write this sum in terms of sets. Let $[n]=\{1,\dots,n\}$ and, for $A\subseteq[n]$, write $x_A=\{x_i: i\in A\}$. We will have terms of the form $\log p(x_i | x_A)$ (where $i$ is not in the subset A) and we need to count how many such terms appear in the original sum. The number of permutations where we have the $|A|$ indices in $A$ followed by $i$, followed by the remaining $(n-|A|-1)$ indices in arbitrary order is $|A|!(n-|A|-1)!$.
\be
\log p(x_1, \ldots, x_n) &= \sum_{A\subseteq[n]} \frac{|A|!(n-|A|-1)!}{n!} \sum_{i \in [n] \setminus A} \log p(x_{i} | x_{A}) \\
&= \sum_{k=0}^{n-1} \frac{n}{n} \sum_{A\subseteq[n]: |A|=k} \frac{k!(n-k-1)!}{n!} \frac{n-k}{n-k} \sum_{i \in [n] \setminus A} \log p(x_{i} | x_{A}) \\
&= n \sum_{k=0}^{n-1} \frac{1}{n} \sum_{A\subseteq[n]: |A|=k} \frac{k!(n-k)!}{n!} \sum_{i \in [n] \setminus A} \frac{1}{n-k} \log p(x_{i} | x_{A})
\ee
In the second line, we rewrite the sum over sets to a sum over subsets of size $k$, then a sum of $k$. We also introduce some factors of 1 which are re-arranged in the third line. The goal is to re-write this expression in terms of expectations to get a convenient Monte Carlo estimator. 
\be
\frac1n \log p(x_1, \ldots, x_n)  = \underbrace{\sum_{k=0}^{n-1} \frac{1}{n} }_{\text{Uniform over k}} \underbrace{\sum_{A\subseteq[n]: |A|=k} \frac{1}{\dbinom{n}{k}}}_{\text{Uniform over sets $|A|=k$}} \underbrace{\sum_{i \in [n] \setminus A} \frac{1}{n-k}}_{\text{Uniform over $i \in [n] \setminus A$}} \log p(x_{i} | x_{A})
\ee
At this point, we can recognize that each sum corresponds to a sample over a uniform distribution. The number of indices $i$ in the set $[n] \setminus A$ is $n-k$, the number of subsets of $[n]$ of size $k$ is $\tbinom{n}{k}$, the number of subset sizes is $n$. 
\be
\frac1n \log p(x_1, \ldots, x_n)  &= \mathbb E_{P(k, A, i)} [\log p(x_{i} | x_{A})] \\
P(k) &= \mathrm{Unif}\{0,1,\dots,n-1\} \\ 
P(A|k) &= \mathrm{Unif}\{A \subseteq[n]: |A|=k\} \\ 
P(i|A,k) &= \mathrm{Unif}\{[n] \setminus A\}
\ee
This gives us a concrete procedure for obtaining an unbiased estimator of the log likelihood that only requires evaluating probabilities of marginals, conditioned on sets of variables. This is exactly what our denoiser is trained to do. For discrete random variables, the inner expectation is just the mean cross entropy per variable of unmasked variables conditioned on masked ones, as seen in the LLADA estimator. 
We construct a Monte Carlo estimate by sampling over subsets of $A$. Training with this objective is simple as each Monte Carlo sample is a simple cross entropy. 
We summarize the procedure in Alg.~\ref{alg:nll_roar}. 

\begin{figure}[htbp]
\centering
\begin{minipage}{0.48\linewidth}
  \begin{algorithm}[H]
    \caption{ROAR NLL Estimator}\label{alg:nll_roar}
    {\small 
    \begin{algorithmic}[1]
      \Require Sample, x. Model giving marginal probabilities, $p(x_i|x_A)$ for any subset of unmasked tokens, $A$.
  \State Sample k and set A of size k uniformly
  \State Average cross entropy per variable of masked tokens conditioned on unmasked token set, A
  \State \Return NLL estimate (bits per token) 
    \end{algorithmic}
    }
  \end{algorithm}
\end{minipage}
\vspace{-0.5em}
%\caption{Comparison of Algorithm A and B}
\label{fig:algorithm_nll}
\end{figure}

\subsection{Cross-entropy as a surrogate for the path-integral MSE}
\label{app:ce_surrogate}

This appendix justifies the choice of token-level cross-entropy on the continuous-SNR branch of DSL training. The path-integral identity \eqref{eq:probability} expresses $-\log P(\vx)$ as a contour integral of the squared embedding error
\[
E_i(\vx, \gamma) \;=\; \E_{p_{\gamma}(\vz\mid\vx)}\!\big[\norm{\vx_i - \xhat_i(\vz)}_2^2\big],
\]
where $\xhat_i = \E[\vx_i \mid \vz]$ is the Bayes-optimal denoiser. A direct continuous-NLL training scheme would parameterize a denoiser $\xhat_{i,\theta}$ in embedding space and minimize this MSE. We argue that DSL should instead minimize a token-level cross-entropy objective that has the same Bayes-optimal solution but avoids three concrete pathologies of embedding-space MSE.

\paragraph{Same Bayes-optimal solution.}
Under the unit-sphere constraint, the Bayes denoiser admits the closed form
\[
\xhat(\vz) \;=\; \E_{P(\vx)}\!\big[\vx\,e^{\vx\cdot\vz}\big]\,/\,\E_{P(\vx)}\!\big[e^{\vx\cdot\vz}\big]
\;=\; \sum_{v\in\gV} p(\vx{=}v\mid\vz)\,\ve_v,
\]
i.e., a vocabulary-weighted average of unit-norm embeddings under the induced token posterior $p(\vx{=}v\mid\vz)$. Embedding-space MSE is minimized when $\xhat_\theta(\vz)$ matches this posterior mean. Token-level cross-entropy on $p_\theta(\vx{=}v\mid\vz)$ is minimized when $p_\theta(\cdot\mid\vz) = p(\cdot\mid\vz)$. Both objectives therefore identify the same Bayes-optimal solution: at the global minimum, the CE-trained model recovers $p(\cdot\mid\vz)$, and reading off its posterior mean recovers $\xhat(\vz)$. We use cross-entropy because, away from the global minimum, embedding-space MSE has structurally worse optimization geometry, as we explain next.

\paragraph{Embedding collapse under jointly learned embeddings.}
\cite{gulrajani2023likelihood} observe that direct $\ell_2$ regression in embedding space is ill-posed when the embedding matrix $W_{\mathrm{emb}}$ is jointly trained with the denoiser: setting $W_{\mathrm{emb}}\to 0$ together with $\xhat_\theta\to 0$ achieves zero loss for any noisy input $\vz$, yielding a degenerate solution. Prior continuous diffusion language models address this only with handcrafted embeddings or auxiliary regularizers~\citep{li2022diffusion,dieleman2022continuous}. DSL constrains every clean embedding to lie on the unit sphere ($\norm{\vx_i}_2 = 1$), which already rules out exact collapse, but the corresponding low-rank near-collapsed configurations remain reachable by gradient descent on embedding-space MSE. Cross-entropy on a softmax over fixed unit-norm embeddings does not exhibit this pathology: the cross-entropy minimum requires the predicted distribution to actually concentrate on the correct token, not merely shrink the predicted embedding.

\paragraph{Low-SNR memorization burden.}
CDCD~\cite{dieleman2022continuous} first trained a continuous language diffusion model with cross-entropy on a softmax whose expectation over (normalized) embeddings gives the denoiser (score interpolation).
Plaid \cite{gulrajani2023likelihood} also note that a denoiser trained directly with embedding MSE must, at low noise, output the clean embedding vectors to high precision—effectively memorizing $W_{\mathrm{emb}}$ inside the network parameters. They mitigate this by reparameterizing the denoiser as a softmax over token logits, then averaging embeddings weighted by that softmax, an architectural change motivated by exactly this observation. The DSL converter \eqref{eq:converter} adopts the same architectural shape, and CE training is the natural objective for the resulting categorical head: the network's job is to assign probability mass over $\gV$, not to reconstruct embedding coordinates.

\paragraph{Mean-seeking shortcuts on multi-modal posteriors.}
At intermediate SNR, the token posterior $p(\cdot\mid\vz)$ is generically multi-modal: several tokens have non-trivial probability. The MMSE estimator $\xhat = \sum_v p(v\mid\vz)\ve_v$ is the posterior mean—an average of the candidate embeddings—which under unit-sphere geometry is itself non-unit-norm and lies in the interior of the embedding simplex. A finite-capacity denoiser trained on embedding MSE can approach this mean by outputting a generic interior vector that does not commit to any particular token, lowering raw MSE without producing useful posterior structure. The cross-entropy objective penalizes such non-committal outputs because they correspond to entropic predicted distributions; the optimizer is pushed toward distributions that match the actual posterior shape.

\medskip

In summary, cross-entropy and embedding MSE share the same Bayes-optimal solution under \eqref{eq:probability}, but cross-entropy is preferred for three operational reasons rooted in prior work~\citep{gulrajani2023likelihood}: it eliminates joint-embedding collapse, eliminates the low-SNR memorization burden, and avoids mean-seeking shortcuts on multi-modal posteriors. This is what makes cross-entropy a practical surrogate for the path-integral MSE on the continuous-SNR branch of DSL training.

\subsection{Posterior-State Matching and Inference-Time Error Correction}
\label{subsec:posterior_matching}

A useful way to understand why DSL can open a \emph{correction window} is through a minimal two-bit manifold example.
Consider sequences $(x_1,x_2)\in\{0,1\}^2$ whose data distribution is supported only on
\begin{equation}
\gM=\{(0,0),(1,1)\},
\qquad
P\bigl((0,0)\bigr)=P\bigl((1,1)\bigr)=\half.
\label{eq:toy_manifold}
\end{equation}
Thus, inconsistent states such as $(0,1)$ or $(1,0)$ are off-manifold: they may arise as intermediate draft errors, but they are not valid solutions.

Embed the two symbols on the unit sphere by
\begin{equation}
\enc(0)=-1,
\qquad
\enc(1)=+1,
\label{eq:toy_embedding}
\end{equation}
so the two valid clean sequence embeddings are
\begin{equation}
\vx^{(-)}:=(-1,-1),
\qquad
\vx^{(+)}:=(+1,+1).
\label{eq:toy_clean_states}
\end{equation}
Under the DSL localization channel, a corrupted state is
\begin{equation}
\vz = \snr \vx + \rootsnr \,\eps,
\qquad
\eps \sim \gN(0,\mI_2),
\qquad
\vx \in \{\vx^{(-)},\vx^{(+)}\}.
\label{eq:toy_channel}
\end{equation}

Because the prior is uniform and the clean embeddings are unit norm, the posterior over the two valid hypotheses is
\begin{align}
P\bigl(\vx=\vx^{(+)} \mid \vz\bigr)
&=
\frac{\exp(\vz^\top \vx^{(+)})}
{\exp(\vz^\top \vx^{(+)})+\exp(\vz^\top \vx^{(-)})}
=
\frac{\exp(z_1+z_2)}
{\exp(z_1+z_2)+\exp(-(z_1+z_2))},
\label{eq:toy_posterior_pos}
\\
P\bigl(\vx=\vx^{(-)} \mid \vz\bigr)
&=
\frac{\exp(\vz^\top \vx^{(-)})}
{\exp(\vz^\top \vx^{(+)})+\exp(\vz^\top \vx^{(-)})}
=
\frac{\exp(-(z_1+z_2))}
{\exp(z_1+z_2)+\exp(-(z_1+z_2))}.
\label{eq:toy_posterior_neg}
\end{align}
Equivalently,
\begin{equation}
P\bigl(\vx=\vx^{(+)} \mid \vz\bigr)=\sigmoid\!\bigl(2(z_1+z_2)\bigr),
\qquad
P\bigl(\vx=\vx^{(-)} \mid \vz\bigr)=1-\sigmoid\!\bigl(2(z_1+z_2)\bigr).
\label{eq:toy_posterior_sigmoid}
\end{equation}
Hence the decision boundary is not determined token-by-token, but by the \emph{joint} statistic
\begin{equation}
z_1+z_2=0.
\label{eq:toy_boundary}
\end{equation}

The posterior-mean denoiser is
\begin{align}
\xhat(\vz)
&=
\E[\vx \mid \vz]
=
P\bigl(\vx=\vx^{(+)} \mid \vz\bigr)\vx^{(+)}
+
P\bigl(\vx=\vx^{(-)} \mid \vz\bigr)\vx^{(-)}
\notag\\
&=
\Bigl(
P\bigl(\vx=\vx^{(+)} \mid \vz\bigr)
-
P\bigl(\vx=\vx^{(-)} \mid \vz\bigr)
\Bigr)(1,1)
\notag\\
&=
\tanh(z_1+z_2)\,(1,1).
\label{eq:toy_denoiser}
\end{align}
\Eqref{eq:toy_denoiser} is the key point: the denoising field points toward the correct manifold as long as the corrupted state remains on the correct side of the joint decision boundary.

To make this concrete, suppose the ground-truth sequence is $(1,1)$, but an intermediate draft contains a local inconsistency resembling $(1,0)$.
In continuous state space, this corresponds to a region such as
\begin{equation}
\gR_{11}^{\mathrm{corr}}
:=
\{\vz\in\sR^2:\ z_1>0,\ z_2<0,\ z_1+z_2>0\}.
\label{eq:toy_revisable_region}
\end{equation}
Inside $\gR_{11}^{\mathrm{corr}}$, the second coordinate looks locally wrong, but the joint posterior still favors $\vx^{(+)}=(1,1)$ because $z_1+z_2>0$.
Therefore, the denoising update in \eqref{eq:toy_denoiser} continues to point toward $(1,1)$ rather than amplifying the local error.
This is the \emph{correction window}: an off-manifold draft can still be revised if its induced posterior state remains inside the attraction region of the correct solution.

This example also clarifies what DSL does \emph{not} claim.
DSL does not guarantee correction for every erroneous draft.
If the state crosses the wrong side of the boundary, i.e.
\begin{equation}
z_1+z_2<0,
\label{eq:toy_wrong_basin}
\end{equation}
then the posterior collapses toward $\vx^{(-)}=(0,0)$ and the denoising field points to the wrong manifold.
Thus, posterior-state matching opens a correction window, but does not guarantee that inference remains inside it.

This toy case explains why mixed-corruption DSL training can help.
Training does not need to reproduce the exact symbolic error $(1,0)$ as a raw state.
Instead, it is sufficient that the corruption distribution exposes the model to continuous states whose induced posterior geometry overlaps with revisable regions such as \eqref{eq:toy_revisable_region}.
In that case, some imperfect self-generated drafts at inference time are less out-of-distribution, and the model is more likely to produce a denoising field that still points back toward a valid solution manifold.

By contrast, standard autoregressive decoding and plain absorbing masked diffusion do not naturally realize this mechanism.
For autoregressive decoding,
\begin{equation}
\hat{x}_1=\argmax_{x\in\{0,1\}} p_\theta(x_1=x),
\qquad
\hat{x}_2=\argmax_{x\in\{0,1\}} p_\theta(x_2=x \mid \hat{x}_1),
\label{eq:toy_ar}
\end{equation}
so once $\hat{x}_1$ is emitted, later steps condition on a hard prefix and do not revise it.
Likewise, in plain absorbing masked diffusion, visible tokens are effectively frozen once unmasked, so a visible-but-wrong token is not naturally moved back toward the correct manifold unless an explicit remasking mechanism is introduced.
DSL is therefore complementary to remasking-based samplers: remasking reopens a visible token for revision, while DSL increases the chance that the reopened state still lies in a posterior region where the denoising field points back toward the correct solution.

The lesson of this example is that correction is not explained by ``many possible contexts'' alone.
What matters is whether these possibilities are organized into a posterior geometry whose mass still overlaps the correct solution and whose denoising field still points toward it.
DSL helps precisely by broadening the posterior states seen during training, thereby making some refinement-time erroneous drafts less out-of-distribution and more revisable.

\subsection{Confidence Geometry: Norm--Direction Decomposition and Bounded Entropy}
\label{subsec:geometry}

With the posterior written as a function of the realized state alone, its geometry becomes transparent.
Throughout this subsection we use the unit-sphere token geometry of DSL: each clean token embedding
satisfies $\|\ve_v\|_2 = 1$ for $v \in \gV$.

\begin{proposition}[Norm--direction decomposition]
\label{prop:norm_direction}
For any nonzero state $\vz$ and any unit-norm token embedding $\ve_i$, the score of token $i$
induced by $\vz$ satisfies
\begin{equation}
s_i(\vz)
:=
\langle \vz,\ve_i\rangle
=
\norm{\vz}_2 \cos \theta_i,
\label{eq:norm_direction}
\end{equation}
where $\theta_i$ is the angle between $\vz$ and $\ve_i$.
At $\vz=0$, all scores are zero and the posterior reduces to the base measure.
\end{proposition}

\begin{proof}
For $\vz \neq 0$, the cosine identity gives
\[
\cos\theta_i
=
\frac{\langle \vz,\ve_i\rangle}{\|\vz\|_2\|\ve_i\|_2}.
\]
Since $\|\ve_i\|_2=1$, this implies
$\langle \vz,\ve_i\rangle=\|\vz\|_2\cos\theta_i$.
\end{proof}

Proposition~\ref{prop:norm_direction} separates two roles that are entangled in unconstrained logits:
the angular term determines \emph{which token} is preferred, whereas the radial term determines
\emph{how concentrated} the posterior is around that preference.
This yields the discrete hyperspherical posterior
\begin{equation}
p(x=i \mid \vz)
=
\frac{\pi_i \exp\!\big(\norm{\vz}_2 \cos \theta_i\big)}
{\sum_{j \in \gV} \pi_j \exp\!\big(\norm{\vz}_2 \cos \theta_j\big)}.
\label{eq:hyperspherical_posterior}
\end{equation}

\paragraph{Entropy envelope at bounded radius.}
For any finite radius cap $\beta < \infty$, write $\vz=\beta \vu$ with $\|\vu\|_2 \le 1$.
The corresponding bounded-radius posterior family is
\begin{equation}
q_\beta(v \mid \vu)
=
\frac{\pi_v \exp\!\bigl(\beta\langle\vu,\ve_v\rangle\bigr)}
{\sum_{j\in\gV}\pi_j \exp\!\bigl(\beta\langle\vu,\ve_j\rangle\bigr)},
\qquad
\norm{\vu}_2 \le 1,\;\norm{\ve_v}_2 = 1.
\label{eq:bounded_posterior}
\end{equation}
Since $\langle \vu,\ve_v\rangle \in [-1,1]$, the logits can only tilt away from the base
measure $\pi$ within a finite band controlled by $\beta$. Hence the achievable posterior entropy
over this bounded-radius family is confined to the interval
\begin{equation}
H_{\min}(\beta,\pi)
\;\le\;
H\!\left(q_\beta(\cdot \mid \vu)\right)
\;\le\;
H_{\max}(\beta,\pi),
\label{eq:entropy_envelope}
\end{equation}
where
\begin{equation}
H_{\min}(\beta,\pi)
:=
\min_{\|\vu\|_2\le 1} H(q_\beta(\cdot\mid\vu)),
\qquad
H_{\max}(\beta,\pi)
:=
\max_{\|\vu\|_2\le 1} H(q_\beta(\cdot\mid\vu)).
\end{equation}
Thus, within a finite-SNR or bounded-converter regime, $\beta$ controls how sharp the posterior may
become, while $\pi$ sets the baseline non-uniformity of token probabilities. Without such a finite
radius cap, the Bayes posterior can become arbitrarily sharp as $\|\vz\|_2$ grows.

\paragraph{Smoothness as a supporting property.}
The same hyperspherical geometry also induces a smooth Bayes posterior-mean map
\begin{equation}
\vm(\vz):=\E[\ve_x\mid \vz].
\label{eq:posterior_mean_map}
\end{equation}
When the clean-token distribution is supported on unit-norm embeddings, this Bayes posterior-mean
map is $1$-Lipschitz:
\begin{equation}
\|\vm(\vz_1)-\vm(\vz_2)\|_2
\le
\|\vz_1-\vz_2\|_2 .
\end{equation}
We treat this as a supporting geometric property: it suggests that DSL refinement is well-behaved by
construction, but in practice the decisive mechanism is whether confidence remains informative on
refinement-time draft errors.

\subsection{Lipschitz Continuity of the Induced Posterior-Mean Map}
\label{app:lipschitz}

The main text treats smoothness as a supporting geometric property rather than a core mechanism.
We therefore defer the full proof here.
The result shows that the posterior-mean refinement map induced by DSL is $1$-Lipschitz under unit-sphere token geometry.

% Paste your current subsection:
% "Lipschitz Continuity of the Drift in Unconditional SDE"
%
% You may want to rename "drift" to "induced posterior-mean map"
% so it matches the main text wording.
\begin{lemma}
Let the induced posterior-mean map (MMSE denoiser) of the unconditional SDE be defined by
\[
    \xhat(\vz)  = \mathbb E_{P(\vx)} [ \vx e^{\vx \cdot \vz} ] / \mathbb E_{P(\vx)} [e^{\vx\cdot \vz} ] 
    \qquad \vz\in\R^d,
\]
where the data distribution \(p(\vx)\) is supported on the unit sphere \(\lVert \vx\rVert=1\).
Then \(\hat{\vx}(\cdot)\) is \(1\)-Lipschitz; that is,
\[
    \forall \vz_1,\vz_2\in\R^d:
    \quad
    \lVert\hat{\vx}(\vz_1)-\hat{\vx}(\vz_2)\rVert
    \;\le\;
    \lVert \vz_1-\vz_2\rVert .
\]
\end{lemma}

Define the log–partition function
\[
    \phi(\vz)\;=\;
    \log \mathbb E_{p(\vx)}\!\bigl[e^{\vx\cdot \vz}\bigr],
\]
so that
\(
    \hat{\vx}(\vz)=\nabla_{\vz}\phi(\vz)
\).

Because \(e^{\vx\cdot \vz}\) is convex in \(\vz\) and the expectation preserves convexity, \(\phi\) is convex.  Its Hessian equals the covariance of \(\vx\) under the Gibbs posterior
\(
    p(\vx\mid \vz)\propto p(\vx)e^{\vx\cdot \vz}
\):
\[
    H(\vz)\;=\;\nabla_{\vz}^2\phi(\vz)\;=\;
    \Cov_{p(\vx\mid \vz)}\!\bigl(\vx\bigr).
\]
For any unit vector \( v\in\R^d\),
\[
    v^\top H(\vz)\,v
    \;=\;
    \mathbb E_{p(\vx\mid \vz)}\!\bigl[(v\cdot \vx)^2\bigr]
    -\Bigl(\mathbb E_{p(\vx\mid \vz)}[v\cdot \vx]\Bigr)^{\!2}
    \;\le\;
    \mathbb E_{p(\vx\mid \vz)}\!\bigl[(v\cdot \vx)^2\bigr]
    \;\le\;
    \lVert v\rVert^2
    \;=\;1,
\]
where the last inequality uses \(\lVert \vx\rVert=1\), therefore, spectral norm satisfies \(\lVert H(\vz)\rVert\le1\).

Since \(\nabla_{\vz}^2\phi\) is bounded by \(1\) everywhere, the gradient \(\nabla_{\vz}\phi=\hat{\vx}\) is \(1\)-Lipschitz:
\[
    \lVert\hat{\vx}(\vz_1)-\hat{\vx}(\vz_2)\rVert
    \;\le\;
    \int_0^1\!
        \bigl\lVert H\bigl(\vz_2+t(\vz_1-\vz_2)\bigr)\bigr\rVert
        \,\lVert \vz_1-\vz_2\rVert\,dt
    \;\le\;
    \lVert \vz_1-\vz_2\rVert .
\]
This proves non-expansiveness of the Bayes posterior-mean map. Empirical estimates with K < 1 are consistent with this bound, but are stronger than what the lemma guarantees.
\qed

\section{Continuous States Are Not Enough: Posterior Coordinates vs.\ Timestep-Conditioned Denoising}
\label{app:continuous_vs_dsl}

Prior continuous diffusion LMs already use finite-SNR states between noise and clean
embeddings~\citep{li2022diffusion,dieleman2022continuous,gulrajani2023likelihood,strudel2022self}.
The difference in DSL is not continuity itself. It is how the noisy token is
parameterized and what the denoiser needs to read it.

\paragraph{Generic Gaussian corruption.}
Consider
\begin{equation}
\vy_\tau = a_\tau\vx + \sigma_\tau\eps,\qquad \eps\sim\gN(\vzero,\mI),
\label{eq:generic_continuous_channel}
\end{equation}
with nominal time $\tau$. Let $\vx_v=\enc(v)$ and let $\pi_v$ be the context-dependent
prior of token $v$. The Bayes posterior is
\begin{equation}
P(\vx=\vx_v\mid\vy_\tau,\tau)\;\propto\;
\pi_v\exp\!\Big(\tfrac{a_\tau}{\sigma_\tau^2}\langle\vy_\tau,\vx_v\rangle
-\tfrac{a_\tau^2}{2\sigma_\tau^2}\norm{\vx_v}^2\Big).
\label{eq:generic_continuous_posterior}
\end{equation}
Time enters in two places: the norm term and the scale $a_\tau/\sigma_\tau^2$.
Normalizing the embeddings removes the first, but not the second. The natural
coordinate is
\begin{equation}
\tilde\vz_\tau=\tfrac{a_\tau}{\sigma_\tau^2}\,\vy_\tau,
\label{eq:generic_natural_parameter}
\end{equation}
so a denoiser that reads $\vy_\tau$ must also be told $\tau$:
\begin{equation}
f_\theta(\vy_\tau,\tau).
\label{eq:timestep_conditioned_denoiser}
\end{equation}
This is a representational point, not a probabilistic one: the pair
$(\vy_\tau,\tau)$ carries the full posterior.

\paragraph{DSL uses the natural coordinate.}
DSL sets
\begin{equation}
\vz_i=\snr_i\vx_i+\sqrt{\snr_i}\,\eps_i,\qquad
\frac{a_{\snr_i}}{\sigma^2_{\snr_i}}=\frac{\snr_i}{\snr_i}=1,
\label{eq:dsl_channel_appendix_compare}
\end{equation}
so the noisy token is already the natural coordinate. With $\norm{\vx_v}=1$ the norm
term also cancels:
\begin{equation}
P(\vx_i=\vx_v\mid\vz,\vsnr)=P(\vx_i=\vx_v\mid\vz)\;\propto\;\pi_v\exp\langle\vz_i,\vx_v\rangle,
\label{eq:dsl_posterior_appendix_compare}
\end{equation}
\begin{equation}
\xhat(\vz,\vsnr)=\E_{P(\vx)}[\vx e^{\vx\cdot\vz}]\big/\E_{P(\vx)}[e^{\vx\cdot\vz}]=\xhat(\vz).
\label{eq:dsl_snr_invariant_compare}
\end{equation}
DSL therefore learns one map,
\begin{equation}
f_\theta(\vz)\approx P(\vx_i\mid\vz),
\label{eq:dsl_time_agnostic_map}
\end{equation}
and different SNRs are different regions of the same posterior space.
The corruption level is still visible: $\norm{\vz_i}$ plays the role that time plays
elsewhere. It is not a missing input, because $\vz$ is a sufficient statistic
for $\vx$.

\paragraph{Relation to specific models.}
CDCD normalizes embeddings and trains with cross-entropy~\citep{dieleman2022continuous};
under VE noise its posterior depends on $\vy/\sigma^2$, so $\sigma$ remains an input.
Bayesian Flow Networks update a categorical belief from Gaussian evidence and feed
the network that belief together with $t$~\citep{graves2025bayesianflownetworks}.
RDLM places the noisy state on the hypersphere and builds diffusion bridges there,
with a time-conditioned network~\citep{jo2025continuousdiffusionmodellanguage}. DSL keeps the noisy token
in $\R^d$ and places only the clean embeddings on the sphere.

\paragraph{Noise-level-free denoising.}
\citet{no_noise_conditioning} observe that many image diffusion models work without
noise conditioning. In DSL this holds exactly: the Bayes denoiser has no SNR argument
(Eq.~\eqref{eq:dsl_snr_invariant_compare}), and constant-norm embeddings are required
for this (Appendix~\ref{app:snr_invariant_denoiser}).

\paragraph{The converter does not discard information.}
The converter output in \eqref{eq:converter} is the gradient of a log-sum-exp in
$\vz_i$. When the embeddings span $\R^d$, this function is strictly convex, so the
map $\vz_i\mapsto\vm_i^{\mathrm{conv}}$ is injective. The converter is therefore a
reparameterization of the noisy token, not a normalization trick. It presents the
posterior state to the Transformer as a bounded mixture of token embeddings
($\norm{\vm_i^{\mathrm{conv}}}\le 1$), in the same geometry as ordinary token inputs.

\paragraph{Why not add a per-token time input?}
A global-time denoiser cannot represent states in which positions have different
SNRs. A denoiser with a per-token $\snr_i$ input can, but it needs position-wise
conditioning in place of the sequence-level adaLN of DiT. Pretrained masked-diffusion
weights then no longer map onto the backbone. DSL removes the need for this input,
keeps the timestep module fixed at $t=0$, and stays compatible with the MDLM checkpoint.
% TODO: controlled per-token gamma-input ablation on Text8 (promised to bt6L / AC).

\paragraph{Comparison summary.}
Table~\ref{tab:continuous_vs_dsl} summarizes the differences.

\begin{table}[t]
\centering
\small
\setlength{\tabcolsep}{4pt}
\renewcommand{\arraystretch}{1.15}
\caption{\textbf{Prior continuous diffusion LMs vs.\ DSL.} Both use finite-SNR states;
they differ in how those states are parameterized and read by the denoiser.}
\label{tab:continuous_vs_dsl}
\begin{tabular}{p{0.20\linewidth} p{0.37\linewidth} p{0.37\linewidth}}
\toprule
& \textbf{Prior continuous diffusion LMs} & \textbf{DSL} \\
\midrule
Posterior depends on & The pair $(\vy_\tau,\tau)$. & The noisy token $\vz$ alone. \\
Denoiser & $f_\theta(\vy_\tau,\tau)$; noise levels are separate, time-indexed tasks. & $f_\theta(\vz)$; noise levels are regions of one posterior space. \\
Token geometry & Often unnormalized; CDCD normalizes, but the scale $a_\tau/\sigma_\tau^2$ still requires $\tau$. & Unit sphere and $a/\sigma^2=1$ together remove the SNR exactly. \\
Confidence & Learned; read relative to $\tau$. & $\norm{\vz_i}$ sets concentration; direction sets the token (Eq.~\eqref{eq:posterior_decomposition}). \\
Training support & One global noise schedule. & Mixed: finite-SNR states and endpoint mask/reveal states. \\
Network input & Noisy embedding plus time. & Bounded mixture of tokens (converter). \\
Sampling paths & One global schedule. & Per-token SNR paths; ROAR, masked diffusion and remasking are special cases. \\
\bottomrule
\end{tabular}
\end{table}

\paragraph{What this comparison does not claim.}
We do not claim that prior continuous diffusion LMs lack posterior meaning or are
numerically unstable. Any Gaussian corruption defines a posterior $P(\vx\mid\vy_\tau,\tau)$.
The claim is narrower: DSL chooses a coordinate in which the noisy token alone determines
the posterior, and builds the training support and the network interface around it.

\section{Training and Sampling Details}
\label{sec:implementation_details}

This appendix documents the practical choices used in the main experiments.
The main text focuses on the method-level design: mixed-support CE posterior matching and a posterior-view converter.
Here we give the concrete training and decoding settings needed for reproducibility.

\subsection{OWT Finetuning Settings}
\label{app:training_settings}

% Paste your current subsection:
% "Training settings (OWT finetuning)"
%
% Keep:
% - Data and sequence length
% - Backbone and parameterization
% - Optimization and batching

We exactly follow MDLM's \citep{sahoo2024simple} training configurations.

\paragraph{Data and sequence length.}
We fine-tune on OpenWebText using the \texttt{openwebtext-split} dataset.
All training samples are truncated/padded to a fixed sequence length $L=1024$ tokens.

\paragraph{Backbone and parameterization.}
We use the DiT backbone in official MDLM GitHub repository \citep{sahoo2024simple} with token embedding dimension 64.
Training is conducted in full precision (FP32).

\paragraph{Optimization and batching.}
We train for a maximum of 100{,}000 optimizer steps with no learning-rate warmup (\texttt{num\_warmup\_steps}=0).
Everything else in training setting is the same as MDLM training setting. 

\paragraph{Compute resources.}
All experiments were run on NVIDIA GPUs. Text8 experiments were run on 2 H100s, and sampling/evaluation took approximately 1 hour. OpenWebText training and sampling  evaluation used 2 H100s, with 5000 generated samples per configuration.
The hybrid-sampler sweeps required approximately 2 hours in total, including
preliminary hyperparameter searches. No additional large-scale pretraining was performed
beyond fine-tuning/evaluating the checkpoints described above.

\subsection{Mixed SNR Support Used in Training}
\label{app:mixed_snr_support}

This subsection specifies the practical mixed-support training distribution used in DSL.
A subset of token positions is sampled from a ROAR-style endpoint branch, while the remainder is sampled from a continuous log-normal SNR branch.
This construction is the practical instantiation of the two exact likelihood views in the main text.

% Keep and slightly emphasize from your current OWT finetuning details:
% - p_ROAR = 0.1
% - gamma_max = 50
% - continuous branch gamma ~ LogNormal(mu=1.65, sigma=0.9)
% - gamma denotes SNR, distinct from decoding progress tau
%
% You may move the log-normal figure here.
We denote by $\gamma$ the signal-to-noise ratio (SNR) used in the training corruption process.
We use a \emph{mixed} SNR path: with probability
\begin{equation}
p_{\mathrm{ROAR}} \;=\; \frac{1}{k} \;=\; 0.1,
\end{equation}
we draw $\gamma$ from the ROAR path sampler, where mask is represented by $\gamma=0$, and clean token is approximated by $\gamma_{\text{max}}=100$; 
with probability $1-p_{\mathrm{ROAR}}$ we draw $\gamma$ from a lognormal sampler.
Concretely, for the lognormal continuous SNR branch we sample
\begin{equation}
\gamma \sim \mathrm{LogNormal}(\mu,\sigma^2),\qquad \mu=1.65,\;\;\sigma=0.9,
\end{equation}
and there is no $\gamma_{\max}$ cut-off for this lognormal continuous training schedule.
Throughout training, \textbf{$\gamma$ denotes SNR and is distinct from the decoding progress variable $\tau$ used in sampling.}

We summarize DSL training in Algorithm~\ref{alg:dsl_training}. Each step samples a sentence from the data, draws a per-token SNR vector $\gamma$ from one of two branches (ROAR with probability $1{-}\lambda$, continuous-path with probability $\lambda$), produces noisy embeddings $\vz$, runs the converter and backbone to obtain token posteriors $p_\theta(\cdot\mid\vz)$, and updates parameters via cross-entropy against the clean tokens.

\subsection{Masked Refinement with a DSL Checkpoint}
\label{app:masked_refinement_alg}

The first decoding family used in the paper pairs a trained DSL checkpoint with a ReMDM-style masked-refinement schedule.
Its role is to test the primary empirical claim of the paper: whether mixed-support DSL training improves refinement robustness and step-efficiency under a fixed masked decoder family.
Algorithm~\ref{alg:remdm_sampler} gives the high-level procedure, while all schedule hyperparameters are listed in Section~\ref{app:sampler_details}.

\begin{algorithm}[tb]
\caption{DSL sampling with ReMDM sampler.
Adapted from \cite[Algorithm~1]{wang2025remdm}; in DSL, the ReMDM mask state is represented by zero SNR.}
\label{alg:remdm_sampler}
\small
\begin{algorithmic}
    \State \textbf{Input:} DSL denoiser $\vx_\theta$, sampling steps $T$, noise schedule $\alpha_t$, clean endpoint $\mathrm{SNR}_{\max}$, \remdm{remask schedule $\sigma_t$}.
    \State \textbf{Initialize:} set every token embedding to the mask endpoint, i.e., $\operatorname{SNR}(\vz_T)=0$.
    \For{$i=T,T-1,\ldots,1$}
        \State $t=i/T,\quad s=(i-1)/T$; set $\alpha_t,\alpha_s$.
        \State \remdm{Set $\sigma_t\in[0,\sigma_t^{\max}]$, where $\sigma_t^{\max}=\min\{1,(1-\alpha_s)/\alpha_t\}$.}
        \State Predict clean-token posteriors $\widehat{\vx}=\vx_\theta(\vz_t)$.
        \State Form the ReMDM endpoint posterior
        $
        p_\theta(\vz_s\mid\vz_t)
        =
        q_{\remdm{\sigma}}^{\mathrm{DSL}}
        \bigl(\vz_s\mid\vz_t,\vx=\widehat{\vx}\bigr),
        $
        where ReMDM's mask atom is interpreted as $\operatorname{SNR}=0$ and its clean-token atom as $\operatorname{SNR}=\mathrm{SNR}_{\max}$.
        \State Sample $\vz_s\sim p_\theta(\vz_s\mid\vz_t)$.
        \State Realize sampled clean tokens at $\mathrm{SNR}_{\max}$; \remdm{realize remasked tokens by resetting their SNR to $0$.}
        \State Set $\vz_t\leftarrow\vz_s$.
    \EndFor
    \State \textbf{Output:} decoded tokens from the final clean-endpoint latents.
\end{algorithmic}
\end{algorithm}

\subsection{Sampling Settings on OWT}
\label{app:sampler_details}

This section fully specifies the decoding schedules used in \Cref{sec:experiments}
and defines the rewrite-count statistics and diagnostics reported throughout the paper.

\paragraph{Common evaluation protocol (all rows in Table~1).}
Unless otherwise stated, all methods are evaluated with the same tokenizer (GPT-2 BPE),
sequence length $L=1024$, nucleus sampling (top-$p=0.9$), sample count (5k),
and MAUVE configuration (GPT-2 Large embeddings with $K=500$ buckets) on the same held-out OWT split.
GenPPL is computed as perplexity under the same GPT-2 Large autoregressive oracle LM.
Sentence entropy is computed as the average Shannon entropy of the token-ID histogram within each generated sequence
(prior to text decoding).
This matches the MDLM/ReMDM evaluation protocol and ensures baseline parity.

\paragraph{Time parameterization (matches logged $t$).}
We index a $T$-step decode by a \emph{0-based} step counter $k=0,\dots,T-1$ and use normalized time
\begin{equation}
t_k \triangleq 1 - \frac{k}{T} \in \left[\frac{1}{T},\,1\right].
\label{eq:time_param_app}
\end{equation}
Thus decoding progresses from $t_0=1$ (fully masked/noisy) to $t_{T-1}=1/T$ (near-clean).
This convention matches the sampler logs (e.g., for $T=128$, the final logged $t\approx 0.0078 \approx 1/128$).

\paragraph{Mask ratio schedule and reveal fraction.}
Let $M_k \subseteq [L]$ denote the masked positions at step $k$ (i.e., positions whose visible token is \texttt{[MASK]}).
We log the \emph{realized} masked ratio
\begin{equation}
r(t_k) \triangleq \frac{|M_k|}{L}.
\end{equation}
We additionally report the \emph{reveal fraction} (newly unmasked mass) per step,
\begin{equation}
\Delta r(t_k) \triangleq r(t_{k}) - r(t_{k+1}) \quad (k<T-1),
\end{equation}
which indicates how aggressively the schedule transitions from global drafting (large $\Delta r$ early)
to local refinement (small $\Delta r$ late).

\paragraph{ReMDM-loop and the loop window.}
Our main step-budgeted protocol uses the ReMDM-loop variant (for all $T$) with the official loop defaults:
\begin{equation}
t_{\mathrm{on}}=0.55,\qquad t_{\mathrm{off}}=0.05,\qquad \alpha_{\mathrm{loop}}=0.9,\qquad \texttt{refresh\_unmasked}=\texttt{true}.
\end{equation}
Conceptually, decoding consists of (i) a standard MDLM-style reveal phase,
(ii) a \emph{loop phase} over the time window $[t_{\mathrm{on}}, t_{\mathrm{off}})$ with $t_{\mathrm{on}} > t_{\mathrm{off}}$,
and (iii) a final reveal phase that anneals to the near-clean endpoint.
During the loop phase, the noise level is held fixed at $\alpha_{\mathrm{loop}}$ so that repeated self-correction
operates under a stable context quality.

\paragraph{Capped remasking intensity (ReMDM-cap inside the loop).}
When remasking is active (i.e., within the loop window), ReMDM remasks a subset of currently unmasked tokens
to avoid early lock-in. We use ReMDM's capped schedule:
\begin{equation}
q(t) \;=\; \min\!\left(1,\; \frac{\eta(t)}{1-r(t)}\right),
\qquad
\eta(t) \;=\; \eta_{\mathrm{cap}} \frac{\alpha(t)}{1-\alpha(t)},
\label{eq:remask_cap_schedule_app}
\end{equation}
where $q(t)$ is the per-token remask probability among unmasked positions and $\alpha(t)$ is ReMDM's noise schedule.
Outside the loop window, we set $q(t)=0$ (no remasking).

\paragraph{Step-budget-aware choice of $\eta_{\mathrm{cap}}(T)$ (as used in our logs).}
For the step-budgeted schedule used in \Cref{tab:owt_mauve} (DSL-FT + ReMDM-loop),
we set
\begin{equation}
\eta_{\mathrm{cap}}(128)=0.010,\quad
\eta_{\mathrm{cap}}(256)=0.008,\quad
\eta_{\mathrm{cap}}(512)=0.007,\quad
\eta_{\mathrm{cap}}(1024)=0.004,
\label{eq:eta_cap_values_app}
\end{equation}
and keep all other ReMDM-loop hyperparameters identical to the official setting.
This is a minimal, reproducible adaptation that controls rewrite counts while front-loading corrections.

\begin{table}[t]
\centering
\small
\begin{tabular}{@{}lcccc@{}}
\toprule
Step budget $T$ & 128 & 256 & 512 & 1024 \\
\midrule
$\eta_{\mathrm{cap}}(T)$ & 0.010 & 0.008 & 0.008 & 0.002 \\
top-$p$ nucleus & \multicolumn{4}{c}{0.9} \\
Tokenizer / $L$ & \multicolumn{4}{c}{GPT-2 BPE / 1024} \\
MAUVE embed / $K$ & \multicolumn{4}{c}{GPT-2 Large / 500} \\
GenPPL oracle & \multicolumn{4}{c}{GPT-2 Large} \\
SNR schedule & \multicolumn{4}{c}{linear, $\texttt{snr\_min}=10^{-3}$, $\texttt{snr\_max}=100$, $\rho=7$, $\gamma=4$} \\
$\sigma$ schedule & \multicolumn{4}{c}{Karras, $\sigma_{\min}=0.01$, $\sigma_{\max}=1.0$, $\rho=7$} \\
Loop window & \multicolumn{4}{c}{$[t_{\mathrm{on}},t_{\mathrm{off}})=[0.55,0.05)$} \\
$\alpha_{\mathrm{loop}}$ & \multicolumn{4}{c}{0.9} \\
\bottomrule
\end{tabular}
\vspace{2pt}
\caption{Decoding hyperparameters for OWT experiments under our step-budgeted ReMDM-loop protocol. All settings are identical across $T$ except $\eta_{\mathrm{cap}}(T)$.}
\label{tab:owt_sampler_hparams}
\end{table}

\paragraph{Rewrite-count statistics.}
Let $x^{(k)} \in \mathcal{V}^L$ be the sampled sequence after step $k$ (with \texttt{[MASK]} treated as a valid symbol during decoding).
We define the per-position rewrite count
\begin{equation}
R_i \triangleq \sum_{k=1}^{T-1} \mathbf{1}\!\left[x^{(k)}_i \neq x^{(k-1)}_i\right],
\qquad i \in [L],
\label{eq:rewrite_count_def}
\end{equation}
and report the mean rewrite count $\frac{1}{L}\sum_{i=1}^L R_i$ (equivalently ``Rewrites-per-token'' in our logs),
averaged over the 5k generated samples.

\begin{table}[t]
\centering
\small
\begin{tabular}{@{}lccc@{}}
\toprule
$T$ & $\eta_{\mathrm{cap}}(T)$ & mean rewrites-per-token $\left(\mathbb{E}[R_i]\right)$ \\
\midrule
128  & 0.010 & 0.579834 \\
256  & 0.008 & 0.924500 \\
512  & 0.008 & 1.610600 \\
1024 & 0.002 & 1.848940 \\
\bottomrule
\end{tabular}
\vspace{2pt}
\caption{Rewrite statistics for the step-budgeted ReMDM-loop protocol (computed from \eqref{eq:rewrite_count_def}, averaged over 5k samples).}
\label{tab:rewrite_stats}
\end{table}

\paragraph{Logged posterior sharpness checkpoints (sanity check for \Cref{app:diagnostics}).}
To connect the sampler to the diagnostics in \S\ref{app:diagnostics}, we also log
(i) the average per-token maximum probability (``mean\_max\_prob'') and
(ii) the average top-$p$ nucleus size (``nucleus size'') at selected steps.
A representative subset of checkpoints is shown below; the strong monotone sharpening (early diffuse $\rightarrow$ late near-deterministic)
is consistent across step budgets.

\begin{table}[t]
\centering
\small
\begin{tabular}{@{}cccccc@{}}
\toprule
$T$ & step $k$ & $t_k$ & mean\_max\_prob $\uparrow$ & nucleus size (top-$p{=}0.9$) $\downarrow$ \\
\midrule
128  & 0   & 1.0000 & 0.0398 & 10268.59 \\
128  & 56  & 0.5625 & 0.9391 & 12.17 \\
128  & 71  & 0.4453 & 0.9585 & 8.72 \\
128  & 127 & 0.0078 & 0.9928 & 1.74 \\
\midrule
256  & 0   & 1.0000 & 0.0398 & 10268.59 \\
256  & 113 & 0.5586 & 0.9455 & 8.72 \\
256  & 142 & 0.4453 & 0.9603 & 7.06 \\
256  & 255 & 0.0039 & 0.9952 & 1.33 \\
\midrule
512  & 0   & 1.0000 & 0.0398 & 10268.59 \\
512  & 227 & 0.5566 & 0.9466 & 9.85 \\
512  & 284 & 0.4453 & 0.9632 & 6.21 \\
512  & 511 & 0.0020 & 0.9973 & 1.04 \\
\midrule
1024 & 0    & 1.0000 & 0.0398 & 10268.59 \\
1024 & 455  & 0.5557 & 0.9482 & 8.75 \\
1024 & 568  & 0.4453 & 0.9640 & 4.83 \\
1024 & 1023 & 0.0010 & 0.9980 & 1.17 \\
\bottomrule
\end{tabular}
\vspace{2pt}
\caption{Logged posterior sharpness checkpoints from the sampler logs (``ReMDM logits stats'' and ``ReMDM nucleus size'').}
\label{tab:sharpness_checkpoints}
\end{table}

% \paragraph{Per-step schedule plots for each $T$.}
% The schedules behind \Cref{tab:rewrite_stats,tab:sharpness_checkpoints}, we recommend including per-step traces of $r(t_k)$, $\Delta r(t_k)$, remask intensity, and rewrite rate for each $T \in \{128,256,512,1024\}$.

\subsection{ROAR Sampler Details}
\label{app:roar_setup}

The random-order autoregressive (ROAR) sampler instantiates the ROAR endpoint configurations of \eqref{eq:roar_estimator} as a single-pass inference path on the same DSL-finetuned checkpoint. The sampler reveals one token at a time in a uniformly random order; each position is denoised exactly once, with no remasking, revisits, or self-correction. It requires no retraining or sampler-specific finetuning---it is an alternative decoding rule applied to the same checkpoint used for iterative refinement. Algorithm~\ref{alg:roar_sampler} gives the procedure.

\paragraph{State representation.}
At each step, the input $\vz \in \R^{L \times d}$ is composed position-by-position according to the current reveal status. Unrevealed positions $i \notin A$ are set to $\vz_i = \vzero$, corresponding to $\gamma_i = 0$ in the DSL channel. Revealed positions $i \in A$ are set to $\vz_i = \gamma_{\max}\,\enc(s_i)$ with $\gamma_{\max} = 100$, providing the denoiser with the same near-clean signal it sees at training time on the ROAR endpoint branch (\S\ref{subsec:mixed_snr}, Appendix~\ref{app:mixed_snr_support}). The sampler thus operates on the exact endpoint support the model was explicitly trained to cover; revealed positions act as conditioning context, unrevealed positions act as targets for the next prediction.

\paragraph{Reveal order.}
For each generation batch, we draw a single permutation $\pi$ of $[L]$ uniformly at random and process positions in the order $\pi_1, \pi_2, \dots, \pi_L$. The same permutation is shared across all sequences in the batch but resampled across batches. We do not use confidence-based or content-aware reveal orders, preserving the unbiased random-order structure of the ROAR estimator \eqref{eq:roar_estimator}. The sampler also exposes a \texttt{causal=True} option that fixes $\pi$ to the identity ($\pi_k = k$), recovering strict left-to-right autoregressive decoding; our reported results use the random-order setting.

\paragraph{Per-step prediction.}
At step $k$, the sampler queries the converter and backbone on the current partial state $\vz$ to obtain token posteriors $p_\theta(\cdot \mid \vz)$ at every position, reads the marginal at position $\pi_k$, and samples $s_{\pi_k}$ from this marginal using nucleus sampling. Position $\pi_k$ is then added to $A$, and $\vz_{\pi_k}$ is updated to $\gamma_{\max}\,\enc(s_{\pi_k})$ for use in subsequent steps. The remaining $L-1$ positions are recomputed in the next forward pass under the updated $\vz$, but each position $i$ is itself only \emph{committed} (sampled) once, namely at step $k$ such that $\pi_k = i$.

\paragraph{Hyperparameters.}
For OWT we use $L = 1024$, $\gamma_{\max} = 100$, $\text{top-}p = 0.9$, and run $T = L = 1024$ reveal steps per sample.

\begin{algorithm}[H]
\caption{DSL ROAR sampler (single pass, random-order)}
\label{alg:roar_sampler}
{\small
\begin{algorithmic}[1]
\Require Sequence length $L$; reveal-time SNR $\gamma_{\max}$; nucleus parameter $\text{top-}p$; \texttt{causal} flag (default \texttt{False}).

\State $A \gets \emptyset$
\State $\vz_i \gets \vzero$ for all $i \in [L]$ \Comment{All positions start at $\gamma_i = 0$}
\If{\texttt{causal}}
    \State $\pi \gets (1, 2, \dots, L)$
\Else
    \State $\pi \sim \mathrm{Unif}(\mathfrak{S}_L)$ \Comment{Single permutation shared across the batch}
\EndIf
\For{$k = 1, \dots, L$}
    \State Compute token posteriors $p_\theta(\cdot \mid \vz)$ via converter and backbone
    \State $s_{\pi_k} \sim \text{nucleus}(p_\theta(\cdot \mid \vz)_{\pi_k};\;\text{top-}p)$
    \State $\vz_{\pi_k} \gets \gamma_{\max}\,\enc(s_{\pi_k})$
    \State $A \gets A \cup \{\pi_k\}$
\EndFor
\State \Return $\vs = (s_1, \dots, s_L)$
\end{algorithmic}
}
\end{algorithm}

\subsection{Hybrid Sampler Details}
\label{app:hybrid_sampler}

The hybrid sampler combines continuous-state EDM-style denoising with discrete masked refinement, both executed on the same DSL-finetuned checkpoint. The high-level idea is that continuous denoising provides a low-cost global initialization across all $L$ token positions in parallel, after which a short discrete refinement stage commits the final tokens. Algorithm~\ref{alg:hybrid_sampler} gives the procedure; we describe the three stages below.

\paragraph{Stage 1: Continuous EDM denoising to $\sigma_{\mathrm{switch}}$.}
We follow the standard EDM~\citep{karras2022elucidating} parameterization with a Karras $\sigma$ schedule, but truncate the schedule at an intermediate switching point $\sigma_{\mathrm{switch}}$ rather than running it all the way to $\sigma_{\min}$. Throughout the continuous stage we maintain two views of the state: the physical state $\vy_t \in \R^{L \times d}$ and its rescaled form $\vz_t = \vy_t / \sigma_t^2$, which matches the SNR-parameterized input the DSL denoiser was trained on (\S\ref{subsec:sentence_level}). Initialization is $\vy_0 = \sigma_{\max}\eps$ with $\eps \sim \mathcal{N}(\vzero, \mI)$. At each step we form $\vz_t = \vy_t / \sigma_t^2$, query the converter and backbone to obtain $\xhat(\vz_t)$, and update $\vy_t$ via either an Euler or Heun step,
\begin{equation}
\vy_{t+1} = \vy_t + (\sigma_{t+1} - \sigma_t)\,\frac{\vy_t - \xhat(\vz_t)}{\sigma_t} \quad (\text{Euler}),
\end{equation}
with the Heun second-order correction applied analogously. We optionally inject EDM-style stochastic churn~\citep{karras2022elucidating} within a configurable $\sigma$ band to improve sample diversity.

\paragraph{Stage 2: Projection to discrete tokens at $\sigma_{\mathrm{switch}}$.}
At the switching point, we read out token logits from the backbone applied to $\vz_T = \vy_T / \sigma_T^2$ and sample an initial discrete sequence using nucleus sampling with parameters $\text{top-}p$ and temperature. This produces an initial token sequence $\vs^{(0)} = (s_1^{(0)}, \dots, s_L^{(0)})$ that captures the global structure resolved by the continuous stage but may contain locally inconsistent tokens.

\paragraph{Stage 3: Discrete masked refinement.}
We refine $\vs^{(0)}$ with $T_{\mathrm{mdlm}}$ steps of MDLM-style remasking. Tokens are mapped to high-SNR continuous states $\vz_i = \gamma_{\max}\,\enc(s_i)$ with $\gamma_{\max} = 100$ to provide the denoiser with a near-clean input at unmasked positions. At each refinement step, the most uncertain tokens---ranked by $1 - \max_v p_\theta(v\mid \vz_i)$---are remasked (their $\vz_i$ set to $\vzero$) and resampled conditioned on the remaining tokens. The mask ratio decays from an adaptive initial value $r_0$ to a target $r_{\mathrm{end}}$ via a cosine schedule. The initial ratio $r_0$ is set from the mean per-token uncertainty at the switching point, clipped to $[r_{\min}, r_{\max}]$, so sequences with sharper switching-point posteriors begin refinement with less aggressive remasking.

\paragraph{Hyperparameters.}
The configurations reported in Table~\ref{tab:hybrid} use the Karras $\sigma$ schedule with $\sigma_{\max}=10$, $\sigma_{\min}=0.01$, $\rho=7$, Heun solver, EDM churn $= 1.41$ within the band $[\sigma_{\min}, \sigma_{\max}]$, and their corresponding $\sigma_{\mathrm{switch}}$. The discrete stage uses $\text{top-}p=0.9$, temperature $0.8$, $\gamma_{\max}=100$, mask ratio bounds $[r_{\min}, r_{\max}] = [0.2, 0.8]$, and $r_{\mathrm{end}} = 0$. 

\begin{algorithm}[H]
\caption{DSL hybrid continuous-then-discrete sampler}
\label{alg:hybrid_sampler}
{\small
\begin{algorithmic}[1]
\Require Sequence length $L$; continuous step count $T_{\mathrm{cont}}$; discrete step count $T_{\mathrm{mdlm}}$; sigma schedule $\{\sigma_t\}_{t=0}^{T_{\mathrm{cont}}}$ with $\sigma_0 = \sigma_{\max}$ and $\sigma_{T_{\mathrm{cont}}} = \sigma_{\mathrm{switch}}$; nucleus parameters $(\text{top-}p, \tau)$; refinement parameters $(\gamma_{\max}, r_{\min}, r_{\max}, r_{\mathrm{end}})$.

\State $\vy_0 \gets \sigma_{\max}\,\eps$, $\eps \sim \mathcal{N}(\vzero, \mI_{L\times d})$ \Comment{Stage 1: continuous EDM}
\For{$t = 0, \dots, T_{\mathrm{cont}} - 1$}
    \State $\vz_t \gets \vy_t / \sigma_t^2$
    \State $\xhat_t \gets$ DSL denoiser$(\vz_t)$ \Comment{converter $\to$ backbone}
    \State Update $\vy_{t+1}$ via Heun step on $(\vy_t - \xhat_t)/\sigma_t$ (with optional churn)
\EndFor

\State $\vz_{T_{\mathrm{cont}}} \gets \vy_{T_{\mathrm{cont}}} / \sigma_{T_{\mathrm{cont}}}^2$ \Comment{Stage 2: project to tokens}
\State Compute logits $\ell_i$ at each position; set $u_i \gets 1 - \max_v \softmax(\ell_i)_v$
\State Sample $s_i^{(0)} \sim \text{nucleus}(\ell_i; \text{top-}p, \tau)$ for all $i$
\State $r_0 \gets \mathrm{clip}(\bar u, r_{\min}, r_{\max})$ where $\bar u$ is the mean of $\{u_i\}$

\For{$k = 0, \dots, T_{\mathrm{mdlm}} - 1$} \Comment{Stage 3: discrete refinement}
    \State $r_k \gets r_{\mathrm{end}} + \tfrac{1}{2}(r_0 - r_{\mathrm{end}})(1 + \cos(\pi k / T_{\mathrm{mdlm}}))$
    \State Select $\lceil r_k L \rceil$ positions with highest current uncertainty; mask them
    \State Set $\vz_i \gets \vzero$ at masked positions, $\vz_i \gets \gamma_{\max}\,\enc(s_i^{(k)})$ otherwise
    \State Resample $s_i^{(k+1)}$ at masked positions from DSL denoiser logits
\EndFor

\State \Return $\vs^{(T_{\mathrm{mdlm}})}$
\end{algorithmic}
}
\end{algorithm}

\section{Mechanistic Analysis}
\label{sec:mech_appendix}

The main text only briefly summarizes why DSL works in practice.
This appendix provides supporting mechanistic evidence.
The central message is that DSL improves refinement when training exposes the denoiser to the posterior states encountered under self-generated drafts, and that useful uncertainty is more important than mere local smoothness.

\subsection{Cyclic Toy Setup and Correction Example}
\label{sec:cyclic_dataset}

We use a simple synthetic discrete dataset to illustrate how DSL handles both masked tokens and visible-but-wrong tokens within a single framework.
Fix an integer $K$ and consider the dataset consisting of all cyclic shifts of a base sequence:
\[
x^{(0)} = [0,1,\ldots,K-1], \qquad
x^{(i)} = \mathrm{roll}(x^{(0)}, i), \;\; i=0,\ldots,K-1.
\]
For visualization, each token is embedded in $\mathbb{R}^2$ on the unit circle with equal angular spacing, yielding a symmetric geometry in which posterior trajectories can be plotted directly.

To illustrate correction, suppose the true sequence is \texttt{ABCDEFG}, but the input contains both masked and garbled positions:
\[
\begin{array}{cccccccc}
\text{Original:} & A & B & C & D & E & F & G \\
\text{Input:} &
\underbrace{\_}_{\text{masked}} &
\underbrace{\_}_{\text{masked}} &
\underbrace{C}_{\text{correct}} &
\underbrace{D}_{\text{correct}} &
\underbrace{B}_{\text{garbled}} &
\underbrace{F}_{\text{correct}} &
\underbrace{F}_{\text{garbled}}
\end{array}
\]
In DSL, masked positions can be assigned $SNR_i=0$, while visible but uncertain positions can be assigned small positive SNR.
This is Proposition~\ref{prop:reverse} with $0<a<\gamma_i$ (Appendix~\ref{app:reverse_sl}).
This allows the model to use partial signal without forcing visible tokens to remain fixed.
In our cyclic toy, this lets the same dynamics both fill masked positions and revise garbled visible ones.

\begin{figure}[tb]
\centering
\begin{subfigure}[t]{0.48\columnwidth}
    \centering
    \includegraphics[width=\linewidth]{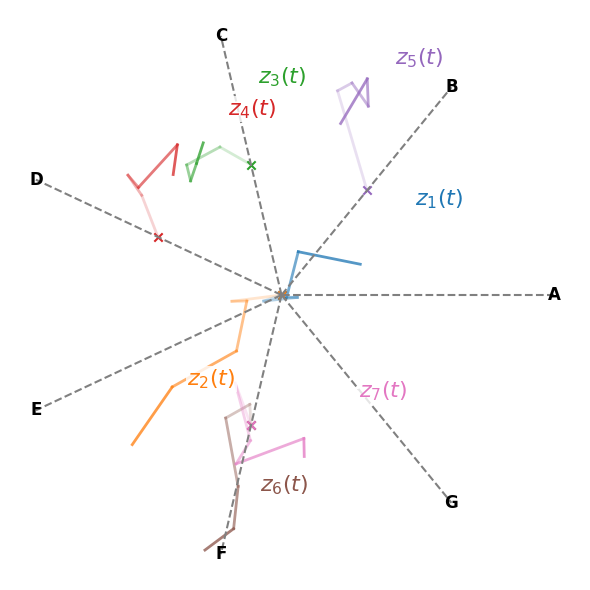}
    \caption{Masked + garbled input}
\end{subfigure}
\hfill
\begin{subfigure}[t]{0.48\columnwidth}
    \centering
    \includegraphics[width=\linewidth]{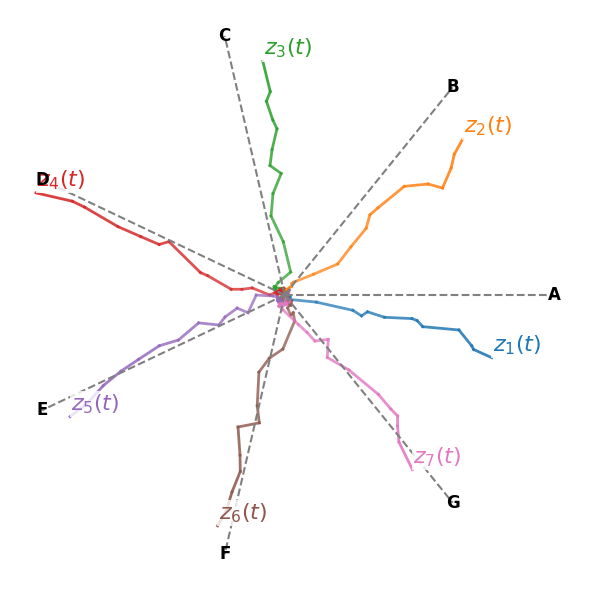}
    \caption{Reconstruction}
\end{subfigure}
\caption{\textbf{DSL correction under masked and garbled inputs.}
The input contains both masked positions and visible-but-wrong tokens.
DSL can assign zero SNR to masked tokens and small positive SNR to uncertain visible tokens, allowing the same refinement dynamics to both fill missing values and correct garbled ones.}
\label{fig:garble}
\end{figure}

\subsection{Why Targeted Remasking Matters}
\label{app:targeted_remasking}

The toy reveals a key asymmetry:
masked positions are already mutable, whereas visible-but-wrong positions must first be remasked before they can be corrected.
This is exactly why refinement-time uncertainty must remain informative.
If the model cannot identify which visible tokens should become mutable again, then self-correction stalls.
Conversely, once posteriors have already sharpened, late remasking tends to induce churn rather than useful revision.

\subsection{Robustness to Self-Generated Intermediate Drafts}
\label{app:self_generated_drafts}

Two complementary ingredients explain DSL's robustness to imperfect self-generated drafts.

\paragraph{(1) Exposure to rollout-like partially corrupted contexts.}
Refinement-time drafts mix masked gaps with model-made mistakes that remain visible.
DSL trains the denoiser under a mixed-support corruption distribution that covers both endpoint-like masked contexts and intermediate partially corrupted drafts.
This broadens the support of training contexts and reduces brittleness when the model conditions on imperfect self-generated states.

\paragraph{(2) Smoothness can help, but remasking requires informative confidence.}
A plausible supporting factor is geometric smoothness in continuous embedding space: a denoiser that behaves smoothly across nearby contexts may generalize better across partially-correct drafts.
However, in a ReMDM-style discrete sampler, correction is bottlenecked by \emph{which tokens get remasked}.
Thus the decisive property is not just smooth interpolation, but whether confidence remains informative enough to identify low-confidence visible tokens that should be revised.

\subsection{Sampling Diagnostics and Over-Refinement}
\label{app:diagnostics}

We track uncertainty, posterior sharpness, and repair workload in order to diagnose when iterative refinement remains productive and when it begins to over-refine.

\paragraph{Token uncertainty and entropy.}
Let $p_{\theta,t}(x_i)$ be the categorical distribution at token position $i$ and sampling step $t$.
Define average uncertainty and entropy by
\begin{align}
u_t &:= \frac{1}{L}\sum_{i=1}^L \Bigl(1 - \max_v p_{\theta,t}(x_i=v)\Bigr), \\
H_t &:= \frac{1}{L}\sum_{i=1}^L \Bigl(-\sum_v p_{\theta,t}(x_i=v)\log p_{\theta,t}(x_i=v)\Bigr).
\end{align}
A rapid drop in $(u_t,H_t)$ indicates posterior sharpening and the onset of a local-repair regime.

\paragraph{Nucleus size as a sharpness proxy.}
For a fixed top-$p$ threshold (e.g.\ $p=0.9$), define
\begin{equation}
k_t := \frac{1}{L}\sum_i |\mathrm{TopP}(p_{\theta,t}(x_i), p)|.
\end{equation}
Large $k_t$ corresponds to broad uncertainty; small $k_t$ indicates sharp, near-deterministic posteriors.

\paragraph{Repair workload and diminishing returns.}
Let $\mathcal M_t$ be the set of positions remasked or rewritten at step $t$.
We track the remask ratio and realized token-change rate:
\begin{equation}
r_t := |\mathcal M_t|/L,
\qquad
\Delta_t := \frac{1}{L}\sum_{i=1}^L \mathbb I[x_i^{(t)} \neq x_i^{(t-1)}].
\end{equation}
Over-refinement is characterized by non-trivial $r_t$ but tiny $\Delta_t$, indicating diminishing returns.

\begin{figure*}[t]
  \centering
  \begin{subfigure}[t]{0.32\textwidth}
    \centering
    \includegraphics[width=\linewidth]{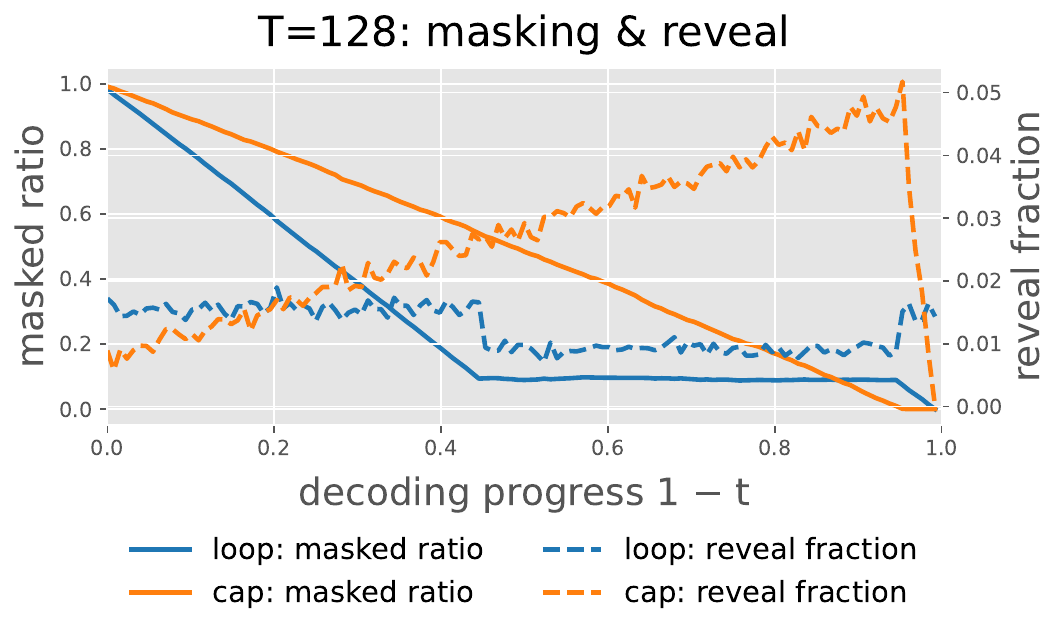}
    \caption{Masking \& reveal}
    \label{fig:mask_reveal}
  \end{subfigure}\hfill
  \begin{subfigure}[t]{0.32\textwidth}
    \centering
    \includegraphics[width=\linewidth]{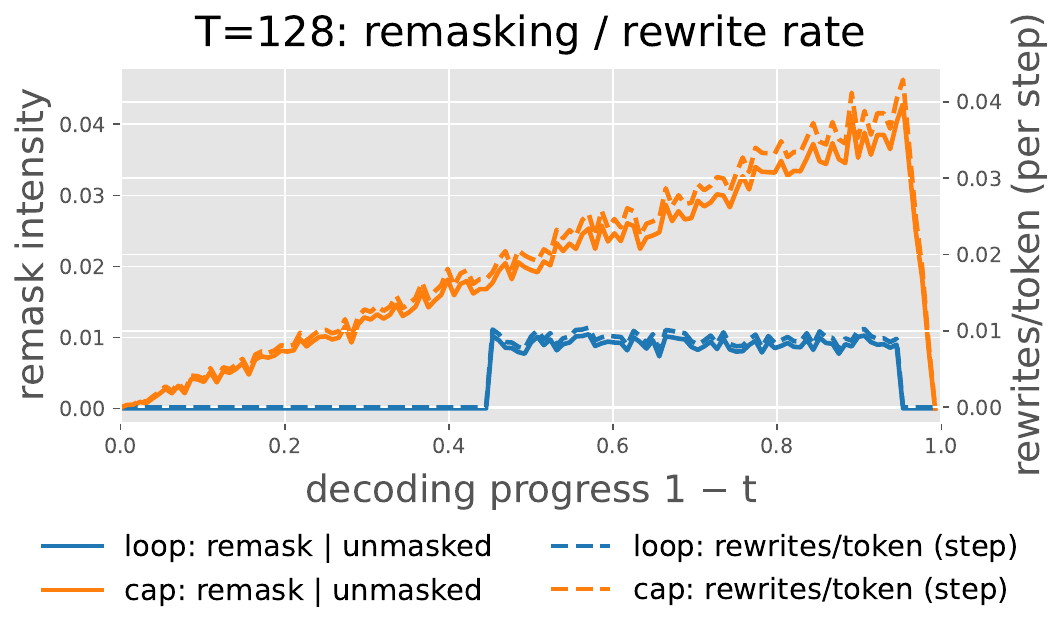}
    \caption{Remasking / rewrite rate}
    \label{fig:remask_rate}
  \end{subfigure}\hfill
  \begin{subfigure}[t]{0.32\textwidth}
    \centering
    \includegraphics[width=\linewidth]{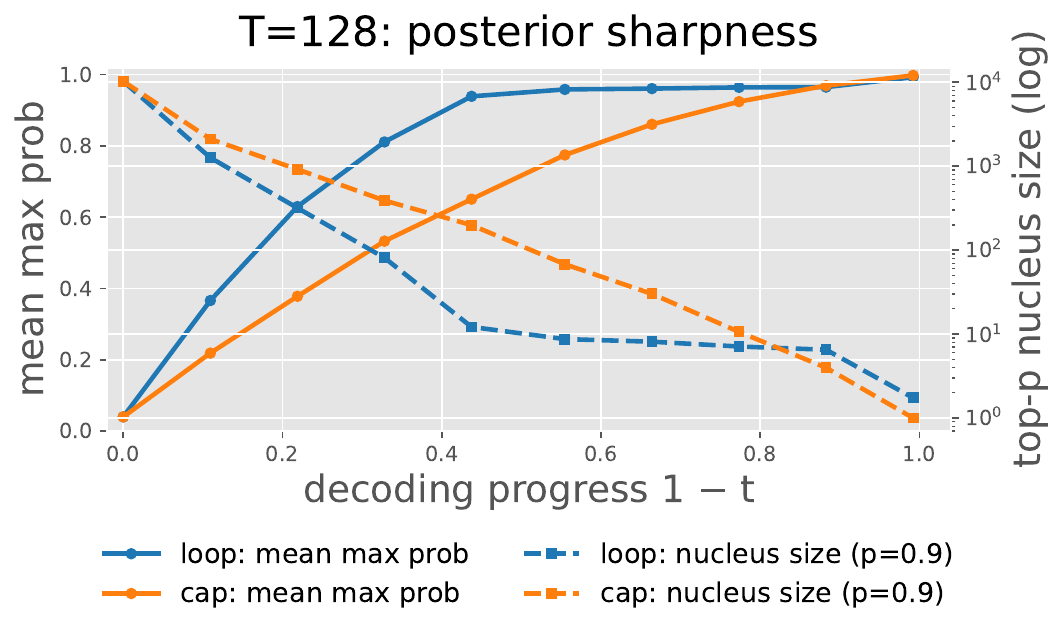}
    \caption{Posterior sharpness}
    \label{fig:posterior_sharpness}
  \end{subfigure}
  \caption{\textbf{Sampling diagnostics under a fixed step budget.}
  (a) Masking and reveal schedule.
  (b) Remasking intensity and realized rewrites per token.
  (c) Posterior sharpening measured by mean max-probability and top-$p$ nucleus size.}
  \label{fig:diag_three_wide}
\end{figure*}

\paragraph{Over-refinement hurts coverage.}
After posteriors sharpen, the sampler may still apply non-trivial remasking while realizing little actual change.
This corresponds to diminishing returns: local likelihood proxies may keep improving while distributional coverage degrades.
This explains the sweet spot observed in the main OWT experiments.

\section{Calibration and Endpoint-Smoothing Ablation}
\label{sec:calibration_appendix}

This appendix supports one of the paper's main practical conclusions:
endpoint-like support is necessary, but endpoint-only training is insufficient.
To obtain usable refinement-time uncertainty, the ROAR-style endpoint branch should be smoothed rather than concentrated entirely on atomic extremes.

\subsection{Atomic vs.\ Smoothed ROAR Endpoints}
\label{app:smooth_roar_why}

DSL uses a ROAR-style endpoint branch to expose the denoiser to masking/reveal-like states.
A natural baseline is \emph{atomic} endpoint sampling, where ROAR tokens are assigned only the two extreme values
\[
\gamma \in \{0,\gamma_{\max}\}.
\]
In contrast, our \emph{smoothed} endpoint variant samples from two narrow endpoint ranges,
\[
\gamma \sim \mathrm{Unif}(0,\gamma_{\min})
\qquad\text{or}\qquad
\gamma \sim \mathrm{Unif}(c\gamma_{\max},\gamma_{\max}),
\]
with equal probability.

The motivation is practical.
Atomic endpoints preserve the masking/reveal semantics, but make the endpoint branch too degenerate:
near the fully masked and near-clean limits, corruption becomes almost deterministic, reducing local variation in training contexts.
Smoothed endpoints preserve the same semantics while broadening the support near both extremes.

\subsection{Near-Clean Calibration}
\label{app:ece_reliability}

We evaluate calibration under teacher forcing on held-out corrupted inputs.
Compared with atomic ROAR endpoints, smoothed endpoints improve calibration precisely in the near-clean regime where refinement decisions are made.

\begin{figure}[t]
\centering
\begin{subfigure}[t]{0.5\columnwidth}
    \centering
    \includegraphics[width=\linewidth]{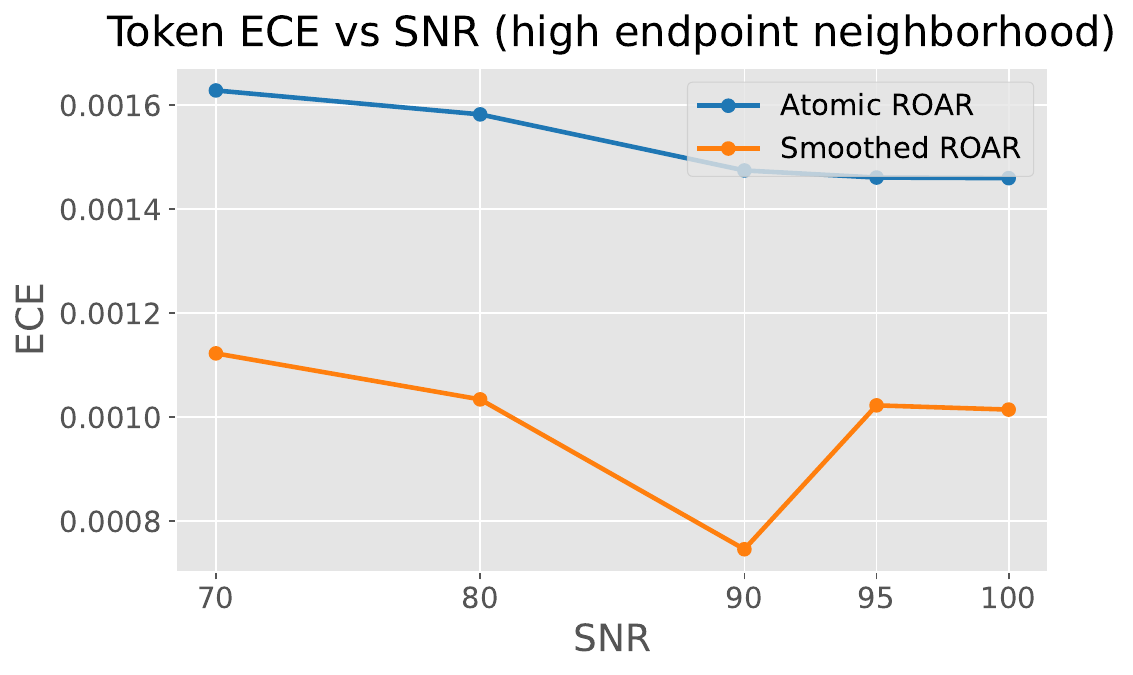}
    \caption{ECE at large SNRs}
\end{subfigure}
\hfill
\begin{subfigure}[t]{0.48\columnwidth}
    \centering
    \includegraphics[width=\linewidth]{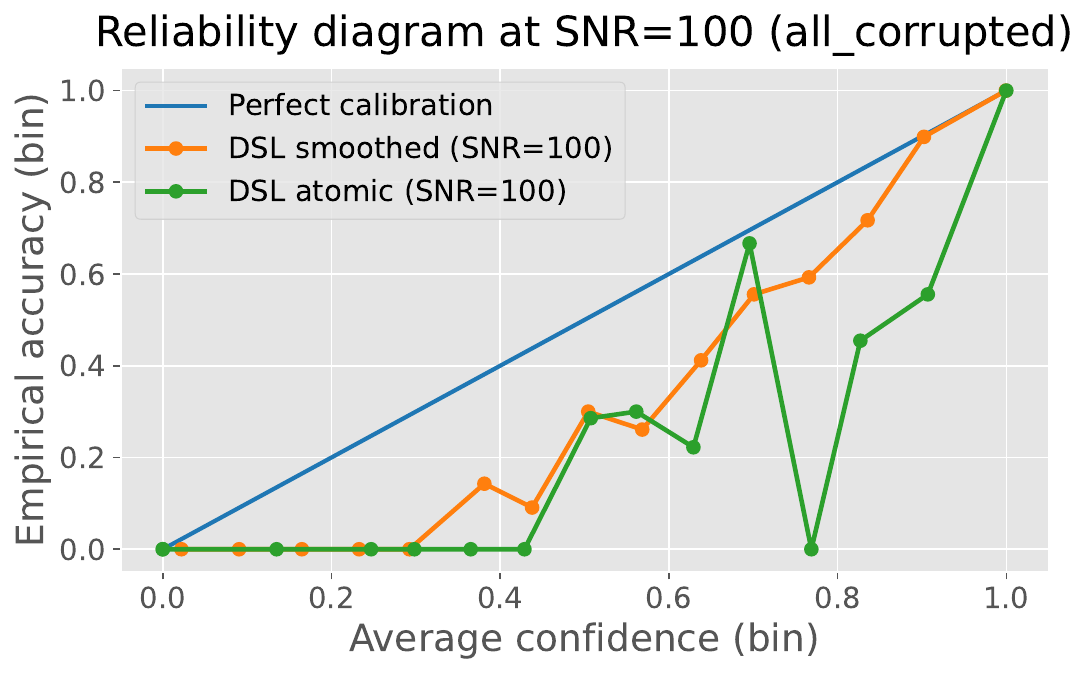}
    \caption{Reliability at $\mathrm{SNR}=100$}
\end{subfigure}
\caption{\textbf{Endpoint smoothing improves near-clean calibration.}
We compare atomic ROAR endpoints ($\gamma\in\{0,\gamma_{\max}\}$) to smoothed endpoint ranges.
Smoothing reduces ECE at large SNR and yields reliability curves closer to the diagonal near the clean limit.}
\label{fig:calibration}
\end{figure}

These results indicate that endpoint-only support is insufficient.
A denoiser trained only on atomic extremes can become poorly calibrated near the clean limit, even though those are exactly the states where refinement-time confidence must be trusted.

\subsection{Downstream Step--Quality Tradeoff}
\label{app:speed_quality_smoothing}

Improved near-clean calibration translates into better downstream refinement under the same decoder family.

\begin{figure}[t]
  \centering
  \begin{subfigure}[t]{0.48\linewidth}
    \centering
    \includegraphics[width=\linewidth]{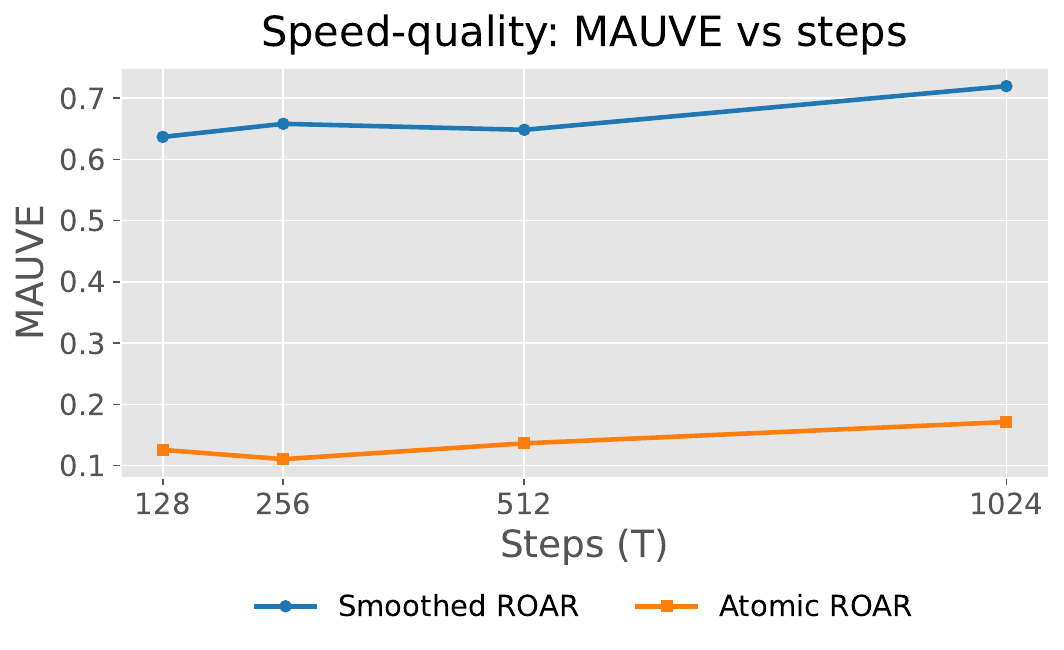}
    \caption{MAUVE vs.\ sampling steps}
    \label{fig:speed_quality_mauve}
  \end{subfigure}\hfill
  \begin{subfigure}[t]{0.48\linewidth}
    \centering
    \includegraphics[width=\linewidth]{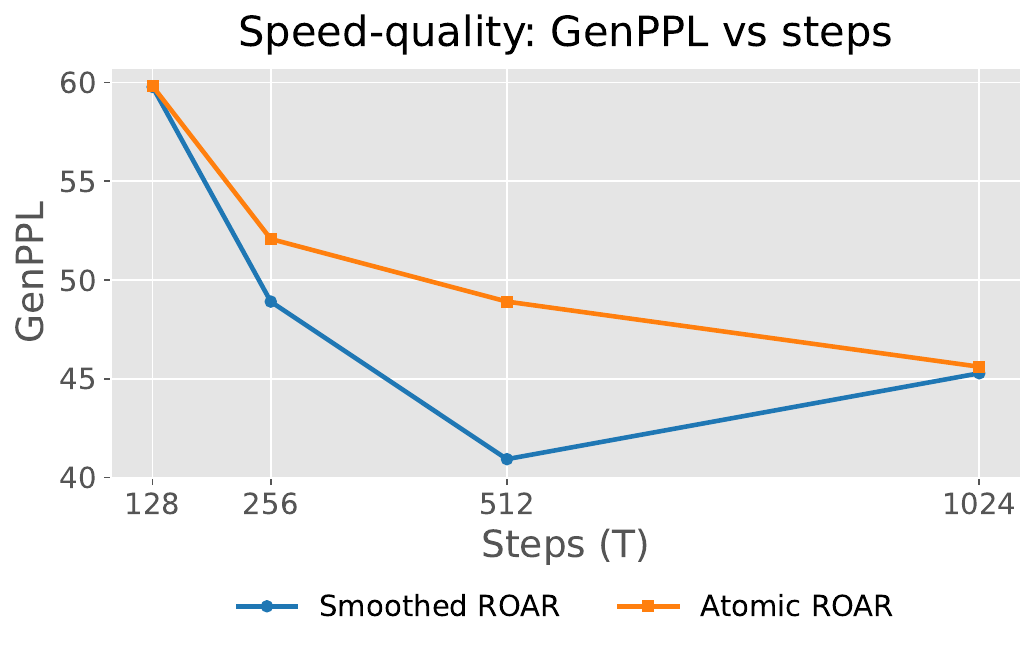}
    \caption{GenPPL vs.\ sampling steps}
    \label{fig:speed_quality_genppl}
  \end{subfigure}
\caption{\textbf{Endpoint smoothing improves the step--quality tradeoff under fixed decoding.}
Using the same ReMDM-style sampler with identical schedules, smoothed-endpoint checkpoints achieve stronger MAUVE across step budgets while maintaining comparable or better GenPPL.}
\label{fig:speed_quality}
\end{figure}

Together, these results support the practical point made in the main text:
endpoint coverage is necessary because masking/reveal states are part of the likelihood family used by refinement, but collapsing training entirely onto atomic endpoints hurts calibration on intermediate and near-clean draft states.
Smoothing the endpoint branch improves the quality of refinement-time uncertainty and thereby improves downstream step-efficient decoding.

\section{Additional Experiments}
\label{sec:additional_exp}

This section provides additional experimental details and diagnostics omitted from the main text for space. 
We report full OpenWebText generation results for masked-refinement decoders and hybrid continuous-then-discrete sampling, including GenPPL and sentence-entropy diagnostics in addition to MAUVE. 
We also provide supplementary Text8 likelihood results and implementation details for reproducibility.

\subsection{OWT}
\label{app:additional_owt_exp}
Here we report the full OpenWebText generation tables. 
Table~\ref{tab:owt_full_table} compares DSL-finetuned checkpoints with masked-refinement baselines across sampling budgets. 
Table~\ref{tab:hybrid_full} reports diagnostics for the selected DSL hybrid continuous-then-MDM handoff configurations used in the main text.

\subsection{Matched-compute and component controls}
\label{app:controls}
To separate DSL from extra OWT training, we continue the MDLM checkpoint under its own objective for the same number of steps (MDLM+CPT). CPT does not improve MDLM under either official sampler (Table~\ref{tab:controls}), while DSL-FT reaches $0.40$--$0.51$ under the same vanilla sampler. We then fix the ReMDM-loop schedule and vary only training (Table~\ref{tab:controls}b). Removing the continuous branch ($\lambda{=}0$) drops MAUVE to $0.004$--$0.005$ under all four samplers. GenPPL rises to $1{,}802$--$2{,}985$ while entropy stays at the MDLM level, so the output is incoherent rather than narrowed. Atomic endpoints reach only $0.22$--$0.31$. Both the continuous branch and endpoint smoothing are needed, and the gain comes from training, not the decoder.

\begin{table}[t]
\centering\scriptsize
\setlength{\tabcolsep}{4pt}
\caption{Training-side controls on OWT (MAUVE$\uparrow$, 5k protocol). (a) Matched-compute continued pretraining (CPT) of MDLM under the official samplers. (b) Training variants under one fixed ReMDM-loop schedule.}
\label{tab:controls}
% \yunshu{add numbers. }
\begin{tabular}{llcccc}
\toprule
& & $T{=}128$ & $256$ & $512$ & $1024$ \\
\midrule
\multicolumn{6}{l}{\textit{(a) Official samplers}} \\
MDLM (published) & MDLM sampler   & 0.015 & 0.023 & 0.031 & 0.042 \\
MDLM + CPT       & MDLM sampler   & 0.016 & 0.019 & 0.030 & 0.043 \\
DSL-FT           & MDLM sampler   & 0.402 & 0.481 & 0.506 & 0.495 \\
MDLM (published) & ReMDM sampler  & 0.057 & 0.216 & 0.350 & 0.403 \\
MDLM + CPT       & ReMDM sampler  & 0.031 & 0.108 & 0.223 & 0.354 \\
\midrule
\multicolumn{6}{l}{\textit{(b) Fixed ReMDM-loop schedule}} \\
MDLM + CPT                               & & 0.011 & 0.030 & 0.069 & 0.111 \\
DSL w/o continuous branch ($\lambda{=}0$) & & 0.005 & 0.004 & 0.005 & 0.004 \\
DSL, atomic endpoints                    & & 0.234 & 0.222 & 0.277 & 0.307 \\
DSL, full                                & & \textbf{0.639} & \textbf{0.661} & \textbf{0.651} & \textbf{0.722} \\
\bottomrule
\end{tabular}
\end{table}

\begin{table*}[t]
\centering
\scriptsize
\setlength{\tabcolsep}{3.0pt}
\renewcommand{\arraystretch}{1.08}
\caption{OWT ($L=1024$) unconditional generation under masked-refinement decoders. Each entry is MAUVE$\uparrow$ / GenPPL$\downarrow$ / SentEnt$\uparrow$; bold marks per-column best MAUVE among non-reference rows. Rows marked $\ddagger$ are taken from \cite{wang2025remdm}; rows marked $\S$ from \cite{rout2025anchored}. All other rows use our protocol. The two DSL rows use the same DSL-finetuned checkpoint with different ReMDM-family remasking schedules.}
\label{tab:owt_full_table}
\resizebox{\textwidth}{!}{%
\begin{tabular}{l|c|c|c|c}
\hline
Method & $T=128$ & $T=256$ & $T=512$ & $T=1024$ \\
\hline
Data (reference) & 1.000 / 14.8 / 5.44 & 1.000 / 14.8 / 5.44 & 1.000 / 14.8 / 5.44 & 1.000 / 14.8 / 5.44 \\
AR (reference)   & -- & -- & -- & 0.760 / 12.1 / 5.22 \\
\hline
SEDD (absorb)        & 0.007 / 119.2 / 5.65 & 0.008 / 112.9 / 5.63 & 0.009 / 107.8 / 5.62 & 0.008 / 104.7 / 5.62 \\
MDLM$^{\ddagger}$     & 0.015 / 61.5 / 5.52  & 0.023 / 55.8 / 5.49  & 0.031 / 53.0 / 5.48  & 0.042 / 51.3 / 5.46 \\
MDLM+FB$^{\ddagger}$  & 0.064 / 42.8 / 5.44  & 0.086 / 39.0 / 5.41  & 0.103 / 37.0 / 5.40  & 0.133 / 33.8 / 5.35 \\
MDLM+DFM$^{\ddagger}$ & 0.041 / 37.9 / 5.31  & 0.098 / 31.2 / 5.30  & 0.168 / 24.2 / 5.26  & 0.254 / 21.7 / 5.20 \\
ReMDM$^{\ddagger}$    & 0.057 / 42.5 / 5.43  & 0.216 / 30.5 / 5.34  & 0.350 / 21.1 / 5.21  & 0.403 / 28.6 / 5.38 \\
PRISM$^{\ddagger}$    & 0.175 / 17.9 / 5.10 & 0.268 / 16.2 / 5.07 & 0.281 / 15.4 / 5.04 & 0.260 / 15.0 / 5.02 \\
PRISM-loop$^{\ddagger}$ & 0.118 / 21.5 / 5.18 & 0.294 / 18.0 / 5.15 & 0.423 / 16.4 / 5.12 & 0.527 / 15.3 / 5.10 \\
ADLM$^{\S}$           & 0.140 / 52.5 / 5.52 & 0.349 / 39.9 / 5.46 & 0.573 / 31.6 / 5.40 & 0.699 / 25.4 / 5.35 \\
\hline
DSL-FT + ROAR (naive) & -- & -- & -- & 0.551 / 47.4 / 5.19 \\
DSL-FT + MDLM &
0.402 / 54.2 / 5.25 &
0.481 / 50.5 / 5.22 &
0.506 / 49.4 / 5.21 &
0.495 / 48.3 / 5.20 \\
\textbf{DSL-FT + ReMDM (confidence-based)} &
0.615 / 68.0 / 5.36 &
0.622 / 61.6 / 5.32 &
\textbf{0.707} / 57.2 / 5.28 &
0.610 / 55.3 / 5.25 \\
\textbf{DSL-FT + ReMDM-loop (principled $\eta_{\mathrm{cap}}$)} &
\textbf{0.639} / 59.7 / 5.34 &
\textbf{0.661} / 49.0 / 5.27 &
0.651 / 41.0 / 5.18 &
\textbf{0.722} / 45.3 / 5.20 \\
\hline
\end{tabular}%
}
\end{table*}

\begin{table}[t]
\centering
\footnotesize
\setlength{\tabcolsep}{5pt}
\renewcommand{\arraystretch}{1.08}
\caption{
Full diagnostics for the selected DSL hybrid continuous-then-MDM sampler
configurations on OWT. All rows use the same DSL-finetuned checkpoint.
The main text reports MAUVE only; here we include GenPPL and sentence
entropy diagnostics for the selected handoff configurations.
}
\label{tab:hybrid_full}
\begin{tabular}{ccccccc}
\hline
$T_{\rm cont}$ & $T_{\rm MDM}$ & NFE & $\sigma_{\rm switch}$
& GenPPL$\downarrow$ & MAUVE$\uparrow$ & SentEnt$\uparrow$ \\
\hline
16 & 16 & 32 & 0.49 & 60.3545 & 0.5008 & 5.0521 \\
16 & 32 & 48 & 0.49 & 52.4820 & 0.6619 & 5.0664 \\
16 & 48 & 64 & 0.49 & 49.0116 & 0.7017 & 5.0604 \\
32 & 16 & 48 & 0.45 & 62.3669 & 0.5872 & 5.1147 \\
32 & 32 & 64 & 0.46 & 53.7585 & 0.5891 & 5.1061 \\
32 & 48 & 80 & 0.46 & 51.4659 & 0.7240 & 5.1033 \\
\hline
\end{tabular}
\end{table}

\subsection{Text8}
\label{app:additional_text8_exp}
\textbf{Dataset} \quad
We test DSL on Text8 dataset \cite{mahoney2006large}, a relatively small-scale, character-level text modeling benchmark extracted from English Wikipedia. 

\textbf{Training Setups} \quad
Following the previous work \citep{gulrajani2023likelihood,austin2021structured,loudiscrete,shi2024simplified}, we evaluated all models on short text chunks of length 256, and also follow the same dataset split and transformer model size to parameterize the denoising models. 
For all the models including our method and baselines, we follow the common practice of using 12-layer transformers similar to GPT2-small scale \citep{shi2024simplified}. Our transformer has the same number of heads (12) and hidden dimension (784) as in \citep{sahoo2024simple}.
Note that all baseline diffusion language models are trained for a million steps, except for our model is trained for half a million.

\textbf{Baselines} \quad
We compare DSL against state-of-the-art continuous and discrete diffusion models, and autoregressive models \citep{Vaswani2017Attention}. 
Continuous diffusion baselines include Plaid \cite{gulrajani2023likelihood}, CDCD \cite{dieleman2022continuous}. 
Discrete diffusion baselines include Discrete Diffusion Model (D3PM)~\citep{austin2021structured}, Score Entropy Discrete Diffusion (SEDD)~\citep{loudiscrete}, Masked Diffusion Language Model (MDLM)~\citep{sahoo2024simple} and MD4~\citep{shi2024simplified}. 
% \todo{Note that we select baseline methods which focusing on improving diffusion pipeline itself. Therefore, not including EDLM~\citep{xu2024energy} which improve basic diffusion by using a outside Energy model.}
For autoregressive models, we choose Any-order Autoregressive Models ARDM \citep{hoogeboom2021autoregressive} and MAC~\citep{shih2022training}, and flow-based methods IAF/SCF~\citep{ziegler2019latent}, AR Argmax Flow~\citep{hoogeboom2021argmax}, Discrete Flow~\citep{tran2019discrete}, and Multinomial Diffusion.

\section{Limitations and Broader Impacts}
\label{sec:limitations_impacts}
DSL is a foundational method for diffusion-based language generation. Potential positive impacts include improving the efficiency, controllability, and flexibility of non-autoregressive text generation systems. At the same time, improvements in text generation quality may increase risks associated with synthetic text, including spam, misinformation, impersonation, or other harmful content generation. We do not release a deployed model or user-facing system, and we encourage future applications of DSL-style methods to follow standard safeguards for generative models, including misuse monitoring, content filtering, and careful release practices.

\section{Assets and Licenses}
\label{sec:assets}

Table~\ref{tab:assets} summarizes the existing assets used in this work. 
We use these assets only for research purposes and follow their stated licenses and terms of use.

\begin{table}[h]
\centering
\small
\caption{Existing assets used in this work.}
\label{tab:assets}
\begin{tabular}{p{0.36\linewidth}p{0.56\linewidth}}
\hline
Asset & License / terms of use \\
\hline
OpenWebText & Public research dataset used for open-ended text generation evaluation; we cite the original release and follow its terms. \\
Text8 & Public benchmark dataset used for character-level likelihood evaluation. \\
LLaDA / MDLM checkpoints & Pretrained checkpoints used for initialization and/or baseline comparison; we cite the original releases and follow their model-use terms. \\
MAUVE & Evaluation metric for open-ended text generation; we cite the original paper/package and follow the package license. \\
GPT-2 tokenizer / evaluator & Used for tokenization and GenPPL evaluation; we cite the original release and follow the corresponding terms. \\
\hline
\end{tabular}
\end{table}

\section{Additional Qualitative Samples}
\label{sec:gen_text_examples}

% If you want to keep generated Text8 samples,
% move "Generated Samples From Text8" here.
% This should remain optional and should not clutter the earlier appendices.

% Paste:
% - Per-sentence sampler examples
% - Per-token sampler examples
\subsection{Text8}
\label{app:text8_samples}

\textbf{Example 1:}\quad
[BOS]s to some anarchists to view the betrayal of destiay and show that the essential character of dut theme survived felled to before nine zero bc and were replaced most of all knightnote have changed dramatically basing on edmund lynn s name death on body f[EOS]

\textbf{Example 2:}\quad
[BOS]ics at a grant university s trade psychology marketing a statement us by saddam is an example of the abortion of aids most modern reputation says the body agrees to the pope whose arm is considered one of the most intrigued piecses of his body see also c[EOS]

\textbf{Example 3:}\quad
[BOS]h the irish protectors caused margraves and other powerful men and hence to present a confusingly successful tour the quotes more interpreted by mary newton puts the rejection of religion lewis and wats honored having led radcliffe men to question the la[EOS]

\subsection{OWT}
\label{app:owt_samples}

The following examples are generated from DSL finetuned checkpoint, with ReMDM-loop (principled $\eta_{\text{cap}}(T)$).

\textbf{Example with $T=128$.} \\
BIP that only need the hard press. Now I think it works very well against your blindness.

At RT11

Here’s how it works.

I’ve adjusted the fast slider to go above the 60’s. Eventually, myshoot rate is now 4 percent. It works all on the T-Mobile bands. If I’m too low, I can cross it up the corner, but if it’s with a stick hand (Carl Rasmus could cut), I infrotm a little near the edge of the box. If I’m going too low, I’ll have the ability to adjust target angles and will still only know what it looks like.

The perfect, though disappointing feeling.

Links:

Lettered September 29st, 2014

are your no planning?

no planning is a natural lifestyle

@khovgeksriumki

forgo where the night of day breaks -go to Eavriumki is going to start the week.

hahahaha no a bet on the day. Daily is on the off right right now. From that I look at the page of his resume and ask if you can break the week, so if the week is fast, I don’t have real hassle asking if he’s going to break it. So I think he’s going to start to go slow, maybe a couple may be there.

Signed on September 29st, 2014

and he’s not going to go after that week.

(out/of-in his catalog)

On this week’s 2nd, I had a lot of green green warning about my start of the weekend so I knew I’m going to lose. I didn’t want to lose any extra weight or anything else, I just wanted to the rest time. At a point in the day I got nothing, I could get anything, but the end of the 60’s was not what I wanted to go. I got myself on track so let’s do a routine at 20am. Before I could wake 20am, I realized I needed to wake up at 20am. It didn’t seem a good day.

Before I had it the next day, I bought a new SIM card, and I spent this week backwards when I was fixed. It was just not a good thing, but I kept the bell circle off for both days.

So it took about 7 minutes to go to work, then 5 minutes on the field between 8 and 7, and had a weight roll out around 5 minutes later. Then the clock on the card still constantly waiting for a delay in the middle of the weight world.

A couple of times, people say, “You don’t have time for this” or “well.” I’m able to stream people and online, listening to radio stations and live. I’m able to tell a lot of people, “Yes, you were able to track people down in front of you today. And I’m thankful I had told myself to do a few things a year and a half.” For me, I get hundreds of texts and messages every day, so it’s right to get to people.

NOLOM feel good

It kind of feels better having spoken to someone now. I’m hungry, but he’s going to do it already. He, I found him happy but I’m halfway through it right now. I’ll bet upon everything from I could buy a couple PTTs from a Gdigy and a T-Mobile, and I don’t want to buy.

The coolest thing now is the ability to fly, long distances, making swimmers with taking pictures. I could do not shoot a flyw if I would like to let it shoot a kmm for me, still out there, still at least where I am. I’ll have got something to show. I don’t worry about a flyw that I shoot for when it’s nowhere.

I could have shoot a group that’s not certain I can shoot, but I come to the angle I have to shoot and where my camera moves it. When I do shoot, I’m going to search, so I can see where a wave’s coming and see if it’s coming on right. Writing and writing is a way to be free from ideas and scrubbing from thought.

I drive a slow car. I’ll take the drive over it and drag it into a right. If I can take it it, I bend my head, and I’m going to take it. And

\textbf{Example with $T=256$.} \\
NASts from age 19 to Cuba for Cuba and Cuba for Cuba with the exception of the house in Havana, a tourism official says.

Fiesta Fiaza is missing from ages 11 and 12 at the time she was conceived, but she happily adopted her younger daughter and adopted her at age. “The best question we have right now is how much she wants to give away,” says Arushal Chomam, Managing Director of Case \& Privacy with the International Investigations Center. “We will give a guess and as soon as we are certified certified, please check for us for some time.”

See also: Cruise Cruise Cruise of Banks Search for MissingFiesta Fiaza – At \$15. Ft Ft : \$15. Ft Ft - Offer Nov. Nov. (1601) Grand Competition Winner Contest Winner New Original Siftth-Officile by an Italian Pararonous Unit

Generally, Fume does handcrafted, and she takes away from her collection. Even though she has not used her collection’s collection or design, she says she’ll put on any of her first set Grashinaza, something that won’t be used again with her collection for years. Not too much is said and given what she will be buying in the amount of time. Fume’s collection is also in an eye of local cultural heritage. “It’s disappointing and whether there are more or more examples of local heritage, but I just said there is something kind of African heritage there that is indicative of people who love Bayay Surplitia,” she says.

In addition to an appreciation of culturals and historical traditions, Fume has an appreciation of indigenous cultures as a part of her collection. She, too, will showcase their cultures across color, ethnicity, art, design, and indigenous traditions. “When I’ll give out, I hope my collection will be just as unique,” she says. “You know, when I come out and when I give out, my collection is just as strong,” she continues. “I’m surprised we have a green color, a dark color, or a black color, or something. It’s just what I hope is a positive experience.”

Fume currently holds about 100 pieces of Tokara Kirara Suguiaiko just out of \$12,002 in Albany, N.Y. “It’s a mixed-together piece, there are a few strokes of stroke or the nib brush that give an impression, but it’s also so beautiful, it’s so cute and I think is so deeply related. My style is the inner hybrid of a stylistic mixed-together piece, not that or anything. It makes me feel comfortable and relaxed. It really feels like a marriage of something and the nuance of something you’ve been doing and done. It feels like fun and it feels like a touch in the room. It’s not very dark when you’re trying to make this touch with a different color palette as it does and sometimes you’re so distracted when you want to wake up and enjoy it. It’s fun and it inspires me as much as I want to do, and I am proud of my time and my work.”

All artwork which comes from a detailed dark coloring color tones to a subtle dark colored background scheme.

Estimated ~1~12,013

Estimated ~1~12,n1~12,013 (approximately \$100)

Approximately 100 pieces of Tokara Suguiaiko Grashinaza hand-crafted of Tokara Suguiaiko

Details: As follows

Price: All sales, and pricing may vary and is a requirement for normal times. As soon as this site is on lockdown we will do an FAQ tour in stores for up to \$50.

Cost: Available online (approx \$8 + \$9) when you can order order online. For more details what we can think we know by which orders ship on eBay. There’s a web form in which you can just print the details, grab something else you can use to sign or keep one of your items on eBay page. You’ll see some really nice zilzali, Mugabe, Picasso, and a stock, but you can have a strategy of ordering or keeping one of the items shipped directly to eBay.

• CY CYV Custom Design (rounded by

This is where you’ll get sketches, drawings, and your own delivery information. What went wrong? This is where we say I’ll go see my own artist drawings.

This is where CY CYV Custom Design/aka delivery info here for an overview of what you can choose for whom you can use to construct a shop of the

\textbf{Example with $T=512$.}\\
See also: California is not trying to impose taxes on other states. Instead, states actually adopt the California tax code.

The best form of arguments already makes the great news when comparing taxes, local payouts, local profits, and national profits. But those arguments are too far from explaining how it works and are too far from explaining basic science of how well it works.

I have heard a lot stories about the Corporate Sales Tax. It is that big smaller corporations were forced to tax because they had no enough to pay taxes, other big corporations did not pay taxes because they deduct the taxes for other things like sales.

It is called no taxation: tax, corporate tax, the income is more than a dollar, more than a ‘ with a few multimillionaires and no place in society. It is that you don’t have to pay your taxes all of the time.

Hawking Taxes in the US:

 Corporations (usually big corporations) have learned how no taxation is the best way to live life. Big companies don’t do things that kind of, people, have to pay the taxes, just by converting them to the US. They do other things that kind of: big corporations like being bailed out of the US. Instead of having to make a profit, they get paid.

There is a way in the world, a baker, not baker at all, where they are in trouble. It is because people have to think big corporations can get paid, because those companies hate the system and are trying to hide it.

If you become an investor, you get paid. Think of 0.8% more than you add it to $1 million a year and you get about 1% of your income. Using a small, $3000 a year you get at the same time for up to 2% how in an annual income.

Raising is the most important thing in the economy. It is it it that you aren’t buying for \$1.00 but if you pay the taxes only for the second one, it becomes 25\% of your income. If you pay the taxes only, but according to the law, the rest will not stay on your income until it becomes 25\% of your income. If you deduct a couple of deductions for that number, it won’t stay on your taxable income, or the total, until the income tax is paid.

You don’t see any gains from the tax until \$1 million. By that time you get 25\% of your income, the tax that now has reinvested into your. You can spend the dividends, save your taxes, and use it to retirement.

And of course, the money used to pay the tax doesn’t change it needs to go up or higher. But the more the tax is paid, the more money it needs to spend.

Then the car gets paid off for the insurance, what the \$2.50 gets and the \$2.50 gets are not additional taxes.

But the amount that gets paid by less than the person gets charged. Tax taxes are taxed on some of the earns a person.

The big difference is that this is how the real income gets taxed. Corporate income is one of the highest levels of taxation. The real income income can be easily derived from the excise tax and the chained tax.

Compensation Ductes: Say you spend money on on a business and you earn the capital to ship it. You pay taxes on ‘the corporate owner (or local company)’ and you earn the same amount of capital costs to ship it.

You say your industry is money capital which consists of lots of money from other companies. First you own the profits. Then you give it to the government and use it to others. This is not true. You owe billions of dollars in income tax taxes to them. No profit government for ‘gas makers’, no taxes for ‘gas makers.’

The capital dividend is not something that they use. If they trade with the people of capital, they pay them not, but that capital equals it.

Digid Benefit Ductes: America float these profits from the coffers of corporations to pay their taxes. They can and collect at the same time the rate of interest rates that they do. Let’s wait until we don’t realize how much of a nuisance we don’t need to pay our taxes. Do we need to pay our taxes with a 2.5 trillion hole?

A couple of months ago we were actually paying taxes on a company that would provide to certain people id any ill. Another company had fixed higher tax rates for individual hospitals and other hospitals and hospitals, while another company had fixed 40 times the rate of federal tax rates for individual insurers, services and other health care providers, while another company had fixed more than

\textbf{Example with $T=1024$.}\\
ongs of Africans played a role in determining population status, family structure, political etc. Traditionally, Africa was political symbol of each person of their own, ethnic or ethnic status. Instead, Africa was political symbol of goodwill for U.S. citizens of all races and cultures. In this case, African Americans have not looked back to the 1950s.

The split isn’t because very few Africans (both present and later) arrived in the U.S. History does not necessarily fix the split. An analysis of “African Remnantss” discovered that far fewer African Americans arrived in the U.S than they did in the mid-9.01\% of today’s population from today. The split is primarily due to the influx of Chinese people.

The Chinese people from the past are not migrants who traveled to coast and coast. It may indeed have been allowed to enter the U.S. under open borders. This is probably partially, since the Chinese people are nowhere fully integrated into the U.S., and the borders for the past 25 years are still porous. Under open borders, the nation still belong to the Republic of Somalia. I would suggest that it was the nation of its own or its neighbors who made the decision.

Certainly the nation’s productivity has not been a positive. One scholar whose family was Awaisan Ahmed told Al-Sal Al we have a “culture of immigrants originating from the founders of the Somali Republic of Somalia.” Another former head of Oshala College concluded that the nation’s productivity fell by an average 3.3\% in 2013. That’s what I’d guess from a fellow Al-Sahala college scholar and the Al-Qaeda fundamentalist.

That’s not to say that the nation is a good neighbor — that the U.S. is a state that cares for all. As adjusted capita employment levels and as indexed adjusted income, productivity has fallen 15\% since then. In addition, the new state provides health assistance for women, the elderly and all except for the state basic health care. How has the income fallen by nearly 15\% in the rest of the country? So has the working class.

Let me begin with the opportunity of success in the U.S. as a nation nation – just from being considered a great nation and having a heart, but being able to get the heart that can go anywhere else. I’m a proud American born in the U.S., my life has gone through periods of war, occupation, occupation, employment as a citizen, as a city, elected as a citizen and a citizen of government. In World War II, America had only due to Russia and Russia in an effort to preserve, build and expand its bases to bring it up as a nation with a rich tradition of generosity and generosity. In the U.S., I salute the great fathers as an instrumental to bring the world again. One that truly belongs to all Americans.

Today, I must formally recognize my power that U.S.-centric solutions must not simply solve the world’s problems. As many southern American countries where I am live, I must recognize what I am doing – as a citizen – and our God Spirit to bring us closer to America. The U.S. Federal Government (Fed-Fed)

Everything Doesn’t (CNN Talk Anything)

Fox News News

If statements from the 9/11/11 and 13/11 attacks are accurate, the hijackers or any false theories are simply going to blow up.

Shortly after attempting 9/9/11 to obtain documents from the CIA, they began questioning whether the hijackers, or if they wanted the documents.

The Pentagon 9/11acking conspiracy was a rum being featured on a CNN show. Why have America’s world stature shift people’s attention to credible answers?

Some of the stories are simply nuts.

One 9/9/11 theory that has been spun on to be “fabged” by CIA’s secret secret operations (something U.S. authorities were able to track back), in theory it may have been planted by CIA. This story remains kept wraps ever after leaked 9/9/11 video Video Goes Be Allenged - Message to President Obama.

The narrative of the ongoing 9/9/11 is anything but false and remains as the mainstream media reporting has been consistent.

On Friday, a statement from the White House said that U.S. intelligence personnel were being intercepted by a computer used to collect them between January and October in 9/9/11.

The statement said that data collected when retrieving them from US intelligence operations operating in Moscow is classified because no records exist. Why don’t we know that there are so many no records? i.e. the CIA’s no records?

It seems like President Obama himself

%%%%%%%%%%%%%%%%%%%%%%%%%%%%%%%%%%%%%%%%%%%%%%%%%%%%%%%%%%%%%%%%%%%%%%%%%%%%%%%
%%%%%%%%%%%%%%%%%%%%%%%%%%%%%%%%%%%%%%%%%%%%%%%%%%%%%%%%%%%%%%%%%%%%%%%%%%%%%%%

%%%%%%%%%%%%%%%%%%%%%%%%%%%%%%%%%%%%%%%%%%%%%%%%%%%%%%%%%%%%

% \newpage

\newpage
\input{checklist.tex}

\end{document}

%% file: math_commands.tex
\usepackage{amsmath,amsfonts,bm,nicefrac}

\def\eqref#1{Eq.~\ref{#1}}
\def\Eqref#1{Equation~\ref{#1}}
\def\1{\bm{1}}

\def\rvb{{\mathbf{b}}}

\def\rvw{{\mathbf{w}}}

\def\vzero{{\bm{0}}}

\def\ve{{\bm{e}}}

\def\vm{{\bm{m}}}

\def\vs{{\bm{s}}}

\def\vu{{\bm{u}}}

\def\vx{{\bm{x}}}
\def\vy{{\bm{y}}}
\def\mI{{\bm{I}}}

\DeclareMathAlphabet{\mathsfit}{\encodingdefault}{\sfdefault}{m}{sl}
\SetMathAlphabet{\mathsfit}{bold}{\encodingdefault}{\sfdefault}{bx}{n}

\def\gM{{\mathcal{M}}}
\def\gN{{\mathcal{N}}}

\def\gR{{\mathcal{R}}}

\def\gV{{\mathcal{V}}}

\def\sR{{\mathbb{R}}}

\newcommand{\E}{\mathbb{E}}
\newcommand{\Ls}{\mathcal{L}}
\newcommand{\R}{\mathbb{R}}

\newcommand{\softmax}{\mathrm{softmax}}
\newcommand{\sigmoid}{\sigma}

\newcommand{\Cov}{\mathrm{Cov}}
\DeclareMathOperator*{\argmax}{arg\,max}

\newcommand{\snr}{\gamma}
\newcommand{\rootsnr}{\sqrt{\gamma}}

\newcommand{\half}{\nicefrac{1}{2}}

\def\eps{{\bm \epsilon}}

\newcommand{\be}{\begin{eqnarray} \begin{aligned}}
\newcommand{\ee}{\end{aligned} \end{eqnarray} }
\newcommand{\benn}{\begin{eqnarray*} \begin{aligned}}
\newcommand{\eenn}{\end{aligned} \end{eqnarray*} }
\newcommand{\xhat}{\hat{\bm x}}

\newcommand{\vsnr}{{\bm \snr}}

\newcommand{\vz}{\bm z}
\DeclareMathOperator{\enc}{enc}

\DeclareMathOperator{\LSE}{LSE}

\newcommand{\qed}{\hfill $\blacksquare$}

\newcommand{\norm}[1]{{\| #1 \| }}

%% file: checklist.tex
\section*{NeurIPS Paper Checklist}

\begin{enumerate}

\item {\bf Claims}
    \item[] Question: Do the main claims made in the abstract and introduction accurately reflect the paper's contributions and scope?
    \item[] Answer: \answerYes{} % Replace by \answerYes{}, \answerNo{}, or \answerNA{}.
    \item[] Justification: The abstract and introduction clearly state the main contributions of DSL: the continuous-state framework, the practical training/interface recipe, and empirical evidence across the evaluated decoding paths. These claims are supported by the theory and experiments in the paper, and are scoped to the evaluated OpenWebText and Text8 settings rather than overstated as universal guarantees.

\item {\bf Limitations}
    \item[] Question: Does the paper discuss the limitations of the work performed by the authors?
    \item[] Answer: \answerYes{} % Replace by \answerYes{}, \answerNo{}, or \answerNA{}.
    \item[] Justification: The paper discusses the main limitations of DSL, including the empirical scope of the evaluation, the reliance on appropriate training support over token-wise SNR states, the use of practical surrogate objectives, and the open challenge of scaling mixed-support training to more complex SNR configurations.

\item {\bf Theory assumptions and proofs}
    \item[] Question: For each theoretical result, does the paper provide the full set of assumptions and a complete (and correct) proof?
    \item[] Answer: \answerYes{} % Replace by \answerYes{}, \answerNo{}, or \answerNA{}.
    \item[] Justification: The paper states the assumptions and notation needed for each theoretical result in the corresponding main-text sections, and provides full derivations and proofs in the appendix. The main text cross-references these appendix results where appropriate, so informal explanations in the core paper are backed by complete proofs.

    \item {\bf Experimental result reproducibility}
    \item[] Question: Does the paper fully disclose all the information needed to reproduce the main experimental results of the paper to the extent that it affects the main claims and/or conclusions of the paper (regardless of whether the code and data are provided or not)?
    \item[] Answer: \answerYes{} % Replace by \answerYes{}, \answerNo{}, or \answerNA{}.
    \item[] Justification: The paper provides the information needed to reproduce the main experiments, including datasets, model/backbone details, training objectives, sampling algorithms, evaluation protocols, and hyperparameter settings in the experimental sections and appendix. We also provide implementation details and will release anonymized code/configuration files to support reproduction of the reported results.

\item {\bf Open access to data and code}
    \item[] Question: Does the paper provide open access to the data and code, with sufficient instructions to faithfully reproduce the main experimental results, as described in supplemental material?
    \item[] Answer: \answerYes{} % Replace by \answerYes{}, \answerNo{}, or \answerNA{}.
    \item[] Justification: We provide an anonymized code repository with implementation details, configuration files, environment requirements, and scripts for training, sampling, and evaluating the main reported experiments. The paper and supplemental material also describe the datasets and corresponding data preparation/evaluation protocols.


\item {\bf Experimental setting/details}
    \item[] Question: Does the paper specify all the training and test details (e.g., data splits, hyperparameters, how they were chosen, type of optimizer) necessary to understand the results?
    \item[] Answer: \answerYes{} % Replace by \answerYes{}, \answerNo{}, or \answerNA{}.
    \item[] Justification: The paper describes the experimental setup for the main results, including datasets, model/backbone choices, training objectives, sampling procedures, evaluation metrics, and the key hyperparameters needed to interpret the results. Further implementation details and configuration files are provided in the appendix and anonymized code repository.


\item {\bf Experiment statistical significance}
    \item[] Question: Does the paper report error bars suitably and correctly defined or other appropriate information about the statistical significance of the experiments?
    \item[] Answer: \answerNo{} % Replace by \answerYes{}, \answerNo{}, or \answerNA{}.
    \item[] Justification: We do not report formal error bars or confidence intervals for the main large-scale experiments, since repeating the full training and sampling sweeps across multiple seeds is computationally expensive. To improve reliability, we use fixed evaluation protocols, large generation/evaluation sets, and additional ablations to support the empirical trends.


\item {\bf Experiments compute resources}
    \item[] Question: For each experiment, does the paper provide sufficient information on the computer resources (type of compute workers, memory, time of execution) needed to reproduce the experiments?
    \item[] Answer: \answerYes{} % Replace by \answerYes{}, \answerNo{}, or \answerNA{}.
    \item[] Justification: The appendix reports the compute resources used for the main experiments, including GPU type/count and approximate training, sampling, and evaluation time. We also describe the additional compute used for hyperparameter sweeps where relevant.

    
\item {\bf Code of ethics}
    \item[] Question: Does the research conducted in the paper conform, in every respect, with the NeurIPS Code of Ethics \url{https://neurips.cc/public/EthicsGuidelines}?
    \item[] Answer: \answerYes{} % Replace by \answerYes{}, \answerNo{}, or \answerNA{}.
    \item[] Justification: We have reviewed the NeurIPS Code of Ethics and believe that the research conducted in this paper conforms to it. The work uses standard public datasets/evaluation protocols and does not involve human subjects, private data, or deceptive data collection.


\item {\bf Broader impacts}
    \item[] Question: Does the paper discuss both potential positive societal impacts and negative societal impacts of the work performed?
    \item[] Answer: \answerYes{} % Replace by \answerYes{}, \answerNo{}, or \answerNA{}.
    \item[] Justification: The paper discusses potential positive impacts from improving efficient and flexible diffusion-based language generation, as well as potential negative impacts from stronger synthetic text generation, including misuse for spam, misinformation, or other harmful content. We also note that responsible release practices and standard safeguards for generative models remain important when applying such methods.

    
\item {\bf Safeguards}
    \item[] Question: Does the paper describe safeguards that have been put in place for responsible release of data or models that have a high risk for misuse (e.g., pre-trained language models, image generators, or scraped datasets)?
    \item[] Answer: \answerNA{} % Replace by \answerYes{}, \answerNo{}, or \answerNA{}.
    \item[] Justification: If model checkpoints are released, they will be released for research use with accompanying usage notes and limitations, and without any deployed user-facing generation service. The paper also discusses potential misuse risks of stronger synthetic text generation and encourages standard safeguards such as misuse monitoring and content filtering for downstream applications.


\item {\bf Licenses for existing assets}
    \item[] Question: Are the creators or original owners of assets (e.g., code, data, models), used in the paper, properly credited and are the license and terms of use explicitly mentioned and properly respected?
    \item[] Answer: \answerYes{} % Replace by \answerYes{}, \answerNo{}, or \answerNA{}.
    \item[] Justification: The paper credits the original creators of the datasets, pretrained checkpoints, evaluation metrics, and code packages used in the experiments. We provide asset/license information in the appendix and respect the corresponding terms of use.


\item {\bf New assets}
    \item[] Question: Are new assets introduced in the paper well documented and is the documentation provided alongside the assets?
    \item[] Answer: \answerYes{} % Replace by \answerYes{}, \answerNo{}, or \answerNA{}.
    \item[] Justification: We release anonymized research code, scripts, and configuration files as new assets to support reproducibility. Documentation is provided alongside the code, including environment setup and instructions for running the main training, sampling, and evaluation pipelines.


\item {\bf Crowdsourcing and research with human subjects}
    \item[] Question: For crowdsourcing experiments and research with human subjects, does the paper include the full text of instructions given to participants and screenshots, if applicable, as well as details about compensation (if any)? 
    \item[] Answer: \answerNA{} % Replace by \answerYes{}, \answerNo{}, or \answerNA{}.
    \item[] Justification: The paper does not involve crowdsourcing, human-subject studies, or human annotation experiments.


\item {\bf Institutional review board (IRB) approvals or equivalent for research with human subjects}
    \item[] Question: Does the paper describe potential risks incurred by study participants, whether such risks were disclosed to the subjects, and whether Institutional Review Board (IRB) approvals (or an equivalent approval/review based on the requirements of your country or institution) were obtained?
    \item[] Answer: \answerNA{} % Replace by \answerYes{}, \answerNo{}, or \answerNA{}.
    \item[] Justification: The paper does not involve crowdsourcing, human-subject studies, or human annotation experiments, so IRB approval or equivalent review is not applicable.


\item {\bf Declaration of LLM usage}
    \item[] Question: Does the paper describe the usage of LLMs if it is an important, original, or non-standard component of the core methods in this research? Note that if the LLM is used only for writing, editing, or formatting purposes and does \emph{not} impact the core methodology, scientific rigor, or originality of the research, declaration is not required.
    %this research? 
    \item[] Answer: \answerNA{} % Replace by \answerYes{}, \answerNo{}, or \answerNA{}.
    \item[] Justification: The core method development and experimental methodology do not use LLMs as an important, original, or non-standard component. Any LLM use, if applicable, was limited to writing, editing, or formatting assistance and did not affect the scientific content, methodology, or originality of the work.


\end{enumerate}